\documentclass{article}

\PassOptionsToPackage{numbers, compress}{natbib}

 \usepackage[main, final]{neurips_2026}

\usepackage{amsmath}
\usepackage{amssymb}
\usepackage{amsfonts}
\usepackage{amsthm}
\usepackage{mathtools}
\usepackage{bm}
\usepackage{bbm}
\usepackage{dsfont}
\usepackage{cases}
\usepackage{nicefrac}
\usepackage{bbding}
\usepackage{pifont}

\usepackage[utf8]{inputenc} %
\usepackage[T1]{fontenc}    %
\usepackage{textcomp}

\usepackage{etoc}  %

\usepackage{graphicx}
\usepackage{wrapfig}
\usepackage{float}
\usepackage{caption}
\usepackage{subcaption}     %
\usepackage{array}
\usepackage{multirow}
\usepackage{multicol}
\usepackage{booktabs}
\usepackage{makecell}
\usepackage{colortbl}

\usepackage{xr-hyper}
\usepackage{hyperref}       %
\usepackage[capitalise]{cleveref}
\crefname{equation}{}{}
\crefname{figure}{Fig.}{Figs.}
\crefname{section}{Sec.}{Secs.}
\crefname{appendix}{App.}{Apps.}
\crefname{table}{Tab.}{Tabs.}
\usepackage{url}

\usepackage{tikz}
\usetikzlibrary{positioning, calc, matrix} 

\usepackage[makeroom]{cancel}

\hypersetup{
  final,
  colorlinks,
  linkcolor=ourred,
  citecolor=ourblue,
  urlcolor=url,
}

\usepackage{algorithm}
\usepackage{algpseudocode}
\usepackage[page,toc,titletoc,title]{appendix}
\usepackage{microtype}
\usepackage{xargs}
\usepackage{xspace}
\usepackage[normalem]{ulem}
\usepackage{soul}
\usepackage{fancyvrb}

\usepackage{lipsum}
\usepackage{blindtext}

\usepackage{xcolor}
\definecolor{ourblue}{rgb}{0.368,0.507,0.71}
\definecolor{ourorange}{rgb}{0.881,0.611,0.142}
\definecolor{ourgreen}{rgb}{0.56,0.692,0.195}
\definecolor{ourgreen2}{rgb}{0.46,0.792,0.195}
\definecolor{ourred}{rgb}{0.923,0.386,0.209}
\definecolor{ourviolet}{rgb}{0.528,0.471,0.701}
\definecolor{ourbrown}{rgb}{0.772,0.432,0.102}
\definecolor{ourlightblue}{rgb}{0.364,0.619,0.782}
\definecolor{ourbrightblue}{RGB}{93,135,237}
\definecolor{ourdarkgreen}{rgb}{0.572,0.586,0.}

\definecolor{ourcyan2}{rgb}{0.125,0.722,0.804}
\definecolor{ourred2}{rgb}{0.863,0.184,0.047}
\definecolor{ouryellow2}{cmyk}{0,0.16,1.0,0.07}
\definecolor{ourviolet2}{cmyk}{0.55,0.56,0,0.47}
\definecolor{ourorange2}{cmyk}{0,0.46,0.89,0.11}
\definecolor{grayseq}{RGB}{120,120,120}

\definecolor{discretecolor}{RGB}{11,83,150}
\definecolor{gaussiancolor}{RGB}{230,145,56}
\definecolor{argmaxcolor}{RGB}{154,0,0}

\definecolor{customred}{RGB}{169, 45, 59}
\definecolor{url}{HTML}{d95225}

\definecolor{olivegreen}{RGB}{165,185,115}
\definecolor{olivegreendark}{RGB}{145,165,95}

\definecolor{maroon}{RGB}{140, 45, 45}

\usepackage[most]{tcolorbox} %
\usepackage{soul}

\makeatletter
\DeclareRobustCommand\onedot{\futurelet\@let@token\@onedot}
\def\@onedot{\ifx\@let@token.\else.\null\fi\xspace}

\makeatother

\newtcolorbox{corollarybox}{
  breakable,
  enhanced,
  colback=ourbrightblue!12,
  colframe=ourbrightblue!80!black,
  left=1pt,
  right=1pt,
  top=1pt,
  bottom=1pt,
  boxrule=1pt
}

\newtcolorbox[auto counter]{findingbox}[1][]{
  breakable,
  enhanced,
  colback=ourbrightblue!8,
  colframe=ourbrightblue!80!black,
  left=2pt,
  right=2pt,
  top=1pt,
  bottom=1pt,
  boxrule=1pt,
  before skip=0.7em,
  after skip=0.7em,
  before upper={\textbf{Finding~\thetcbcounter: } },
  #1
}

\newtcolorbox{samplebox}[1][]{
  breakable,
  enhanced,
  colback=ourorange!6,
  colframe=ourorange!75!black,
  left=6pt, right=6pt, top=4pt, bottom=4pt,
  boxrule=0.6pt,
  arc=2.5pt,
  fonttitle=\bfseries\small,
  coltitle=black,
  colbacktitle=ourorange!22,
  attach boxed title to top left={yshift=-2.5mm, xshift=5mm},
  boxed title style={size=small, sharp corners=all, boxrule=0.3pt, arc=2pt, colframe=ourorange!75!black},
  fontupper=\small,
  title=Sample,
  #1
}

\definecolor{justingreen}{rgb}{0.0078, 0.4431, 0.2823}

\usepackage{pifont}
\newcommand{\cmark}{\textcolor{green!70!black}{\ding{51}}}
\newcommand{\ccross}{\textcolor{red!70!black}{\ding{55}}}

\definecolor{topone}{RGB}{112, 149, 239}   %
\definecolor{toptwo}{RGB}{174, 195, 246}   %
\definecolor{topthree}{RGB}{223, 231, 251} %

\newcommand{\topone}[1]{\cellcolor{topone}#1}
\newcommand{\toptwo}[1]{\cellcolor{toptwo}#1}
\newcommand{\topthree}[1]{\cellcolor{topthree}#1}

\usepackage{arydshln}

\usepackage{amsthm}
\usepackage{thmtools}
\usepackage{thm-restate}

\crefname{theorem}{Thm.}{Thms.}
\newtheorem{lemma}{Lemma}
\crefname{lemma}{Lem.}{Lems.}

\crefname{assumption}{Assump.}{Assumps.}

\crefname{proposition}{Prop.}{Props.}

\crefname{corollary}{Cor.}{Cors.}
\newtheorem{definition}{Definition}
\crefname{definition}{Def.}{Defs.}
\newtheorem{remark}{Remark}
\crefname{remark}{Rmk.}{Rmks.}
\crefname{algocf}{Alg.}{Algs.}
\crefname{algorithm}{Alg.}{Algs.}

\usepackage{titlesec}

\titlespacing*{\section}{0pt}{0.5\baselineskip}{0.2\baselineskip}
\titlespacing*{\subsection}{0pt}{0.3\baselineskip}{0.1\baselineskip}
\titlespacing*{\subsubsection}{0pt}{0.2\baselineskip}{0.1\baselineskip}
\titlespacing*{\paragraph}{0pt}{0pt}{1em}

\newcommand{\replaidscx}{20}
\newcommand{\replaidnscx}{27}

\newcommand{\replaidscppl}{22.1}
\newcommand{\replaidnscppl}{23.6}

\usepackage{amsmath,amsfonts,bm}

\DeclareMathOperator*{\argmax}{argmax}
\DeclareMathOperator*{\argmin}{argmin}
\newcommand{\CE}{\operatorname{CE}}
\newcommand{\Var}{\operatorname{Var}}
\newcommand{\KL}{\operatorname{KL}}
\newcommand{\MSE}{\operatorname{MSE}}
\newcommand{\MMSE}{\operatorname{MMSE}}

\newcommand{\sigmoid}{\operatorname{sigmoid}}

\newcommand{\tsum}{{\textstyle\sum}}

\newcommand{\normal}{\mathcal{N}}
\newcommand{\unif}{\mathcal{U}}

\newcommand{\iidsim}{\stackrel{\mathrm{i.i.d.}}{\sim}}

\newcommand{\R}{\mathbb{R}}
\newcommand{\E}{\mathbb{E}}

\newcommand{\cL}{\mathcal{L}}

\newcommand{\SNR}{{\mathrm{SNR}}}

\newcommand{\NELBO}{{\mathrm{NELBO}}}
\renewcommand{\sc}{{\mathrm{sc}}} %
\newcommand{\data}{{\mathrm{data}}}

\newcommand{\ee}{{\mathrm{e}}}
\newcommand{\zero}{{\mathbf 0}}

\renewcommand{\v}{{\mathbf v}}
\newcommand{\x}{{\mathbf x}}
\newcommand{\z}{{\mathbf z}}
\newcommand{\e}{{\mathbf e}}
\newcommand{\D}{{\mathbf D}}
\newcommand{\I}{{\mathbf I}}
\newcommand{\Emb}{{\mathbf E}}
\newcommand{\beps}{{\boldsymbol{\epsilon}}}

\newcommand{\bxi}{{\boldsymbol{\xi}}}
\renewcommand{\d}{{\mathrm{d}}}                    %

\usepackage{fancyvrb}
\usepackage{pythonhighlight}                       %

\usepackage[utf8]{inputenc} %
\usepackage[T1]{fontenc}    %
\usepackage{hyperref}       %
\usepackage{url}            %
\usepackage{booktabs}       %
\usepackage{amsfonts}       %
\usepackage{nicefrac}       %
\usepackage{microtype}      %
\usepackage{xcolor}         %
\usepackage{enumitem} %
\setlist[itemize,enumerate]{leftmargin=1.5em, topsep=1pt, itemsep=2pt, parsep=0pt}

\title{Continuous Diffusion Scales Competitively with Discrete Diffusion for Language}

\newcommand*\samethanks[1][\value{footnote}]{\footnotemark[#1]}

\author{%
  Zhihan Yang \\
   NVIDIA \& Cornell \\
  \texttt{zhihany@cs.cornell.edu} \\
  \And
  Wei Guo \\
  NVIDIA \& Georgia Tech \\
  \texttt{wei.guo@gatech.edu} \\
  \And
  Shuibai Zhang \\
  UW-Madison\\
  \texttt{shuibai@cs.wisc.edu} \\
  \And
  Subham Sekhar Sahoo \\
  MBZUAI-IFM\\
  \texttt{subham.sahoo@mbzuai.ac.ae} \\
  \And
  Yongxin Chen \\
  NVIDIA \& Georgia Tech \\
  \texttt{yongchen@gatech.edu} \\
  \And
  Arash Vahdat \\
  NVIDIA \\
  \texttt{avahdat@nvidia.com} \\
  \And
  Morteza Mardani\thanks{Equal advising; the senior authors are ordered alphabetically.} \\
  NVIDIA \\
  \texttt{mmardani@nvidia.com} \\
  \And
  John Thickstun\samethanks \\
  Cornell \\
  \texttt{jthickstun@cornell.edu} \\
}

\begin{document}

\maketitle

\etocdepthtag.toc{main}  %

\vspace{-1em}

\begin{abstract}
While diffusion has drawn considerable recent attention from the language modeling community, continuous diffusion has appeared less scalable than discrete approaches. To challenge this belief we revisit Plaid, a likelihood-based continuous diffusion language model (DLM), and construct \textit{RePlaid} by aligning the architecture of Plaid with modern discrete DLMs. In this unified setting, we establish the first scaling law for continuous DLMs that rivals discrete DLMs: RePlaid exhibits a compute gap of only $\replaidscx\times$ compared to autoregressive models, outperforms Duo while using fewer parameters, and outperforms MDLM in the over-trained regime. We benchmark RePlaid against recent continuous DLMs: on OpenWebText, RePlaid achieves a new state-of-the-art PPL bound of $\replaidscppl$ among continuous DLMs and superior generation quality. These results suggest that continuous diffusion, when trained via likelihood, is a highly competitive and scalable alternative to discrete DLMs. Moreover, we offer theoretical insights to understand the advantage of likelihood-based training. We show that optimizing the noise schedule to minimize the ELBO's variance naturally yields linear cross-entropy (information loss) over time. This evenly distributes denoising difficulty without any case-specific time reparameterization. In addition, we find that optimizing embeddings via likelihood creates structured geometries and drives the most significant likelihood gain.

\looseness=-1

\vspace{0.3ex}
\centerline{Project page: \href{https://zhihanyang.com/replaid}{https://zhihanyang.com/replaid}}

\end{abstract}

\section{Introduction}
\label{sec:intro}

Discrete diffusion language models (DLMs) show significant promise, increasingly rivaling the performance and scalability of autoregressive (AR) models across various benchmarks \citep{khanna2025mercury,song2025seed}. Despite these successes, continuous diffusion remains a compelling alternative due to its inherent potential for controllability and acceleration. The smooth geometry of a continuous latent space enables structured editing \citep{meng2022sdedit,chung2023diffusion}
and allows for the application of efficient ODE-based solvers \citep{zhang2023fast,lu2022dpmsolver}, facilitating distillation of complex sampling trajectories into high-fidelity, single-step generators \citep{salimans2022progressive,song2023consistency}. Unlike discrete diffusion, which often suffers from severe quality degradation under limited sampling budgets, continuous diffusion offers a principled path toward efficient text generation by enabling superior trade-offs between computational cost and sample quality \citep{kingma2021variational, karras2022elucidating}.

While continuous DLMs present unique advantages, previous work finds a significant training scalability gap compared to discrete DLMs. Previous studies show empirical evidence that continuous diffusion underperforms discrete diffusion and AR models when evaluated at equivalent training FLOPs and parameter counts~\cite{austin2021structured, gulrajani2023plaid}; the original Plaid study reported a significant $64\times$ compute overhead compared to autoregressive models to achieve matching likelihood \cite{gulrajani2023plaid}. This figure is frequently cited as evidence of the poor scalability of continuous DLMs. By revisiting this principled likelihood-based paradigm, we aim to determine whether this performance gap is an inherent limitation of continuous diffusion or a consequence of incompatible scaling analyses.

To address this question, we introduce \textit{RePlaid}, a revised version of Plaid featuring a transformer architecture carefully aligned with modern discrete DLM practices. While the original $64\times$ gap suggested continuous diffusion was impractical at scale, those findings were obtained under a different setup from the protocols used to establish current discrete diffusion scaling laws. By adopting the exact protocol from \citet{sahoo2026scaling}, including the same dataset, optimization hyperparameters, and compute budget, we perform the first unified scaling comparison between continuous and discrete DLMs. Crucially, in this controlled environment, the performance gap narrows to $\replaidscx\times$ for RePlaid versus AR, placing continuous diffusion on par with discrete diffusion and challenging the narrative of its inherent unscalability.

Furthermore, we find that RePlaid achieves consistent improvements in generation quality over other recent continuous models like FLM \cite{lee2026flow}, LangFlow \cite{chen2026langflow}, and ELF \cite{hu2026elf}. In contrast to other recent work, RePlaid's training objective (like Plaid) is explicitly derived from a \textit{likelihood bound (ELBO)}, providing a theoretically grounded framework for training that heuristic cross-entropy or flow-matching methods lack. This suggests a conclusion analogous to that of AR pre-training:

{\centering \textit{Likelihood-based compression remains a key evaluation metric for continuous DLM pre-training.}\par}

In particular, we attribute the success of RePlaid to two primary factors: $(i)$ optimizing the noise schedule via ELBO variance naturally recovers a near-linear cross-entropy schedule that evenly distributes denoising difficulty without heuristic reparameterization, and $(ii)$ ELBO-based embedding optimization inherently regularizes the latent geometry to improve likelihood, creating a structured, low-entropy space that prevents potential token dispersion in cross-entropy-based training.

Our main contributions are summarized as follows:

\begin{itemize}
    \item \textbf{RePlaid \& Benchmarking:} We introduce RePlaid, a modernized likelihood-based DLM achieving a state-of-the-art $\replaidscppl$ perplexity among existing continuous DLMs on OpenWebText.

    \item \textbf{Unified Scaling Analysis:} We present the first unified scaling law comparison between continuous and discrete DLMs, showing that continuous models scale on par with discrete counterparts when properly aligned.

    \item \textbf{Theoretical Insights:} We show that the ELBO objective: $(i)$ naturally recovers a linear cross-entropy noise schedule, eliminating heuristic time reparameterizations, and $(ii)$ inherently regularizes embedding geometry for likelihood improvement and prevents token dispersion. %
\end{itemize}
\vspace{-0.5em}

\section{Background: Plaid---A Likelihood-Based Continuous DLM}
\label{sec:background}

\textbf{Plaid}~\citep{gulrajani2023plaid} is a Variational Diffusion Model (VDM)~\citep{kingma2021variational} for text.
Let $[V] := \{1, \ldots, V\}$ be the set of vocabulary with size $V$. A length-$L$ sequence $\x \in [V]^L$ is identified with a matrix in $\{0, 1\}^{L \times V}$ whose $l$-th row $\x^l$ is the one-hot vector of the $l$-th token. Plaid defines the embedded sequence $\e := \x \Emb \in \R^{L \times d_e}$, where $\Emb\in\R^{V \times d_e}$ is a learnable token-embedding matrix with each row constrained to unit Euclidean length; the $l$-th row of $\e$, $\e^l$, is the embedding of the non-zero index of $\x^l$. Unless otherwise stated, we use low-dimensional embeddings ($d_e = 16$), as in Plaid.

\paragraph{Forward and reverse processes.} The forward process $q$ is Gaussian noising on embedding $\e$:
\begin{align}\label{eq:vdm-forward}
q (\z_t \mid \x) = \normal(\alpha_t \e, \sigma_t^2 \I),\quad t\in[0,1],
\end{align}
where $\alpha_t,\sigma_t>0$ are smooth scalar functions of $t$. VDMs assume a variance-preserving (VP) process, $\alpha_t^2+\sigma_t^2=1$, yielding a strictly decreasing signal-to-noise ratio $\SNR(t):=\alpha_t^2/\sigma_t^2$. The true reverse posterior $q(\z_s \mid \z_t, \x)$ for $s<t$ is then
\begin{align}\label{eq:vdm-reverse}
q(\z_s \mid \z_t, \x) = \normal\big((\alpha_{t|s}\sigma^2_s/\sigma_t^2) \z_t + (\alpha_s \sigma^2_{t|s}/\sigma^2_t) \e, (\sigma^2_{t|s} \sigma^2_s/\sigma^2_t) \I\big),
\end{align}
where $\alpha_{t|s}:=\alpha_t/\alpha_s$ and $\sigma^2_{t|s}:=\sigma_t^2-\alpha_{t|s}^2\sigma_s^2$. See \citep[App. A.2]{kingma2021variational} for the proof.

\paragraph{Training.} Let $\x_\theta(\z_t, t): \R^{L\times d_e} \times [0,1] \rightarrow (\Delta^{V})^L\subset\R^{L\times V}$ denote a time-conditioned denoising model outputting a categorical distribution over the vocabulary for every position (i.e., each row is a probability vector on the simplex $\Delta^V$).  Before softmax, Plaid augments the final logit head output of $\x_\theta$ with output prior logits, a closed-form Gaussian log-density that helps focus the denoiser's prediction to plausible tokens given $\z_t$. Plaid parameterizes the reverse process by plugging the denoiser's prediction into~\cref{eq:vdm-reverse}:
\begin{align}\label{eqn:denoising}
    p_\theta(\z_s\mid\z_t) = q(\z_s\mid\z_t, \x = \x_\theta(\z_t, t)).
\end{align}

Note that $\x$ in~\cref{eq:vdm-reverse} is the $L\times V$ one-hot matrix of the discrete sequence, and $\x_\theta(\z_t,t)$ has the same shape; \cref{eqn:denoising} substitutes the latter for the former, i.e., it relaxes the one-hot rows to simplex rows rather than committing to a discrete sequence.

With prior $p(\z_1)=\normal(\zero,\I)$, this setup yields the Negative Evidence Lower Bound (NELBO) containing respectively \textit{prior loss}, \textit{reconstruction loss}, and \textit{diffusion loss} (see \cref{supp:nelbo-derivation} for the derivation):
\begin{align}
-\log p_\theta(\x) \leq \cL_\NELBO(\x):=&\KL(q(\z_1 \mid \x) \,\|\, p(\z_1)) + \E_{\z_0 \sim q(\z_0 \mid \x)}\big[\tsum_{l=1}^L- \log \langle \x_\theta^l(\z_0, 0), \x^l \rangle\big] \nonumber\\
& -\tfrac{1}{2}\ \E_{t\sim\unif[0,1], \z_t\sim q(\z_t \mid \x)}[\SNR'(t) \| \e_\theta(\z_t, t) - \e \|^2], \label{eq:plaid-nelbo}
\end{align}
where $\e_\theta(\z_t, t):=\x_\theta(\z_t, t) \Emb$ is the predicted clean embedding, and $\langle\cdot,\cdot\rangle$ denotes the inner product. During training, Plaid adaptively allocates a fraction of each batch to the reconstruction loss (Monte Carlo over $\z_0$) and the remainder to the diffusion loss (Monte Carlo over $t$ and $\z_t$); the prior loss uses the whole batch. See~\cref{supp:sec:training-algo} for the training algorithm and details on allocation across GPU ranks.

\paragraph{Learnable noise schedule.} Plaid inherits the VDM parameterization~\citep{kingma2021variational} of the noise schedule $\gamma(t)=\gamma_0+(\gamma_1-\gamma_0)\tilde{\gamma}(t)$, with learnable scalar \textit{endpoints} $\gamma_0<\gamma_1$ and \textit{interior shape} $\tilde{\gamma}:[0,1]\rightarrow[0,1]$ modeled by a monotonically increasing neural net. This gives $\alpha^2_t = \sigmoid(-\gamma(t))$, $\sigma^2_t = \sigmoid(\gamma(t))$, and $\SNR(t) = \ee^{-\gamma(t)}$. The endpoints and interior shape are trained to minimize the diffusion loss and its Monte-Carlo variance respectively (\cref{fig:diffusion-loss-code} for pseudo-code). As with VDMs, $\gamma_0$ and $\gamma_1$ are always learned. Hence we use the term \textit{noise schedule} to refer to the interior shape $\tilde\gamma(t)$.

\paragraph{Beyond VDMs.} On top of VDMs, Plaid adds two features besides the categorical reparameterization $\e_\theta = \x_\theta \Emb$. First, it uses a learnable encoding via the embedding matrix $\Emb$. Second, it adds self-conditioning~\citep{chen2023analog}: during training, for $25\%$ of each batch the denoiser performs an initial gradient-free forward pass to estimate the clean data, which is fed back as an additional conditioning input for the actual prediction; at validation and sampling time, self-conditioning is enabled for all examples. At sampling time the fed-back estimate comes from the \emph{previous} sampling time $t > s$, so self-conditioning adds no extra NFE: querying $\x_\theta(\z_s, s; \hat\e_\sc)$ with $\hat\e_\sc := \x_\theta(\z_t,t)\Emb$ costs the same as the non-self-conditioned query $\x_\theta(\z_s, s; \zero)$. Importantly, self-conditioning only changes the \emph{denoiser parameterization} for a given query $(\z_t, t)$; it leaves the forward process and the variational decomposition untouched, so \cref{eq:plaid-nelbo} remains a valid NELBO. See~\cref{supp:sec:training-algo} for more details.

\paragraph{Cheaper embedding operations.} Although largely overlooked in the current literature, Plaid offers substantially cheaper embedding operations compared to DLMs that inject Gaussian noise into high-dimensional (e.g., one-hot~\citep{lee2026flow, roos2026categorical, potaptchik2026discrete}) embeddings. For example, projecting $V$-dimensional one-hot embeddings to the transformer hidden size $h$ requires an \texttt{FP32} matrix multiplication of $[L, V] \times [V, h]$. Plaid reduces this to $[L, d_e] \times [d_e, h]$ plus categorical reparameterization $[L, V] \times [V, d_e]$ at the output, which requires $\approx50\times$ fewer FLOPs when $V\approx50$K (e.g., GPT-2~\citep{radford2019language}), $h=768$, and $d_e=16$.

\section{Revisiting Scaling Laws of Continuous DLMs}
\label{sec:scaling}

In this section, we introduce RePlaid: a revised version of Plaid using a transformer architecture that is carefully aligned with modern discrete DLMs. This alignment allows us to benchmark RePlaid against discrete DLMs using the scaling law protocol from~\citet{sahoo2026scaling}. We find that RePlaid's scaling is on par with discrete diffusion: to achieve the same validation loss as AR, RePlaid requires $\replaidscx\times$ compute with self-conditioning, and $\replaidnscx\times$ compute without it. To the best of our knowledge, this is the first unified scaling comparison between continuous and discrete DLMs.

\subsection{RePlaid: Aligning Plaid with Modern Discrete DLM Architecture}\label{subsec:replaid-align}

In \citet{sahoo2026scaling}, MDLM and Duo employ a Diffusion Transformer (DiT)~\citep{peebles2023scalable} equipped with bidirectional attention, RoPE~\citep{su2021roformer}, and AdaLN-Zero~\citep{peebles2023scalable}. The autoregressive (AR) baseline utilizes an identical architecture, differing only in its use of causal attention. For RePlaid, we retain all foundational algorithmic components of Plaid (\cref{sec:background}) while fully aligning its transformer architecture to match the one used by MDLM and Duo. Specifically, this alignment incorporates LayerNorm, MLP biases, GELU(tanh) activations, and AdaLN-Zero modulation with learnable gating for residual connections. In addition, the original Plaid computes its final logit head in \texttt{FP32} due to its interaction with output prior logits (\cref{sec:background}). We remove this substantial numerical confounder for fair comparison. See \cref{supp:sec:align-table} for a detailed pre- and post-alignment comparison on architecture and numerical precision. 

%

%
%

%

\subsection{A Unified Scaling Benchmark for Continuous and Discrete DLMs} 
\label{subsec:scaling-law-setup}

\paragraph{Data.} We utilize the large-scale, deduplicated SlimPajama dataset~\citep{cerebras2023slimpajama}, which provides a sufficiently rich distribution to avoid data repetition during training. We use the Llama-2 tokenizer~\citep{touvron2023llama2} ($V=32{,}001$ with a special mask token), a batch size of $256$, and a sequence length of $2048$ tokens. 

\paragraph{Optimization.} AR, MDLM, and Duo use AdamW with a cosine learning rate schedule decaying from $2\ee{-}4$ to $2\ee{-}5$, alongside $\beta_1{=}0.9$ and $\beta_2{=}0.95$ and $1\%$ learning rate warmup (minimum $100$ steps). Parameters are regularized with a weight decay of $0.1$ without dropout. RePlaid applies this identical configuration, with minor adjustments  to accommodate its learnable noise schedule and its length-normalized, low-dimensional ($d_e=16$) embeddings (\cref{tab:plaid-dit-optimization-scaling}), following Plaid. For reference, optimizing the $231$M MDLM for $1\ee20$ FLOPs requires about $32$ hours on $8$ NVIDIA GB200s.

\paragraph{Counting FLOPs.} We measure non-embedding FLOPs using \texttt{calflops}~\citep{calflops}. Self-conditioning of $25\%$ increases the non-embedding FLOPs by $\approx8.3\%$, which we offset by reducing training steps.

\subsection{IsoFLOP Analysis and Scaling Laws}
\label{subsec:isoflop}

\begin{figure*}[t]
    \centering
    \begin{subfigure}{0.325\linewidth}
        \centering
        \includegraphics[width=\linewidth]{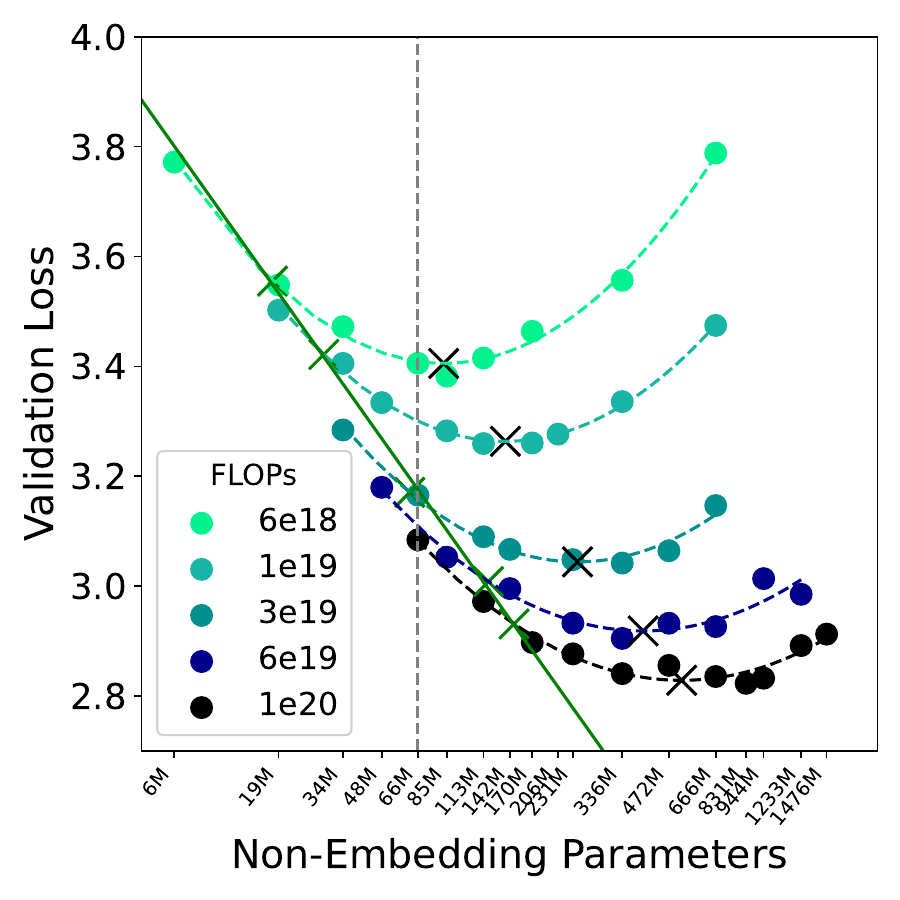}
        \caption{\centering MDLM with\\ low-variance training (low var.)}
        \label{fig:iso_flop_mdlm}
    \end{subfigure}
    \begin{subfigure}{0.325\linewidth}
        \centering
        \includegraphics[width=\linewidth]{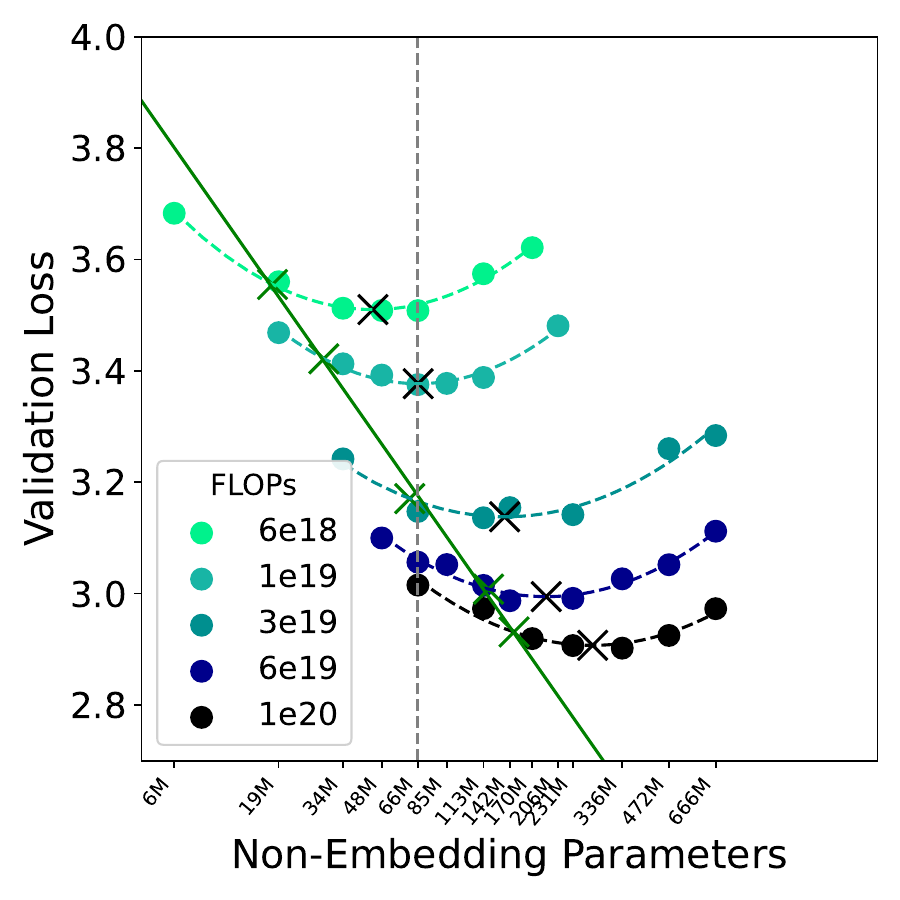}
        \caption{\centering \textsc{RePlaid} with\\ self-conditioning (s.c.)}
        \label{fig:iso_flop_replaid_sc}
    \end{subfigure}
    \begin{subfigure}{0.317\linewidth}
        \centering
        \includegraphics[width=\linewidth]{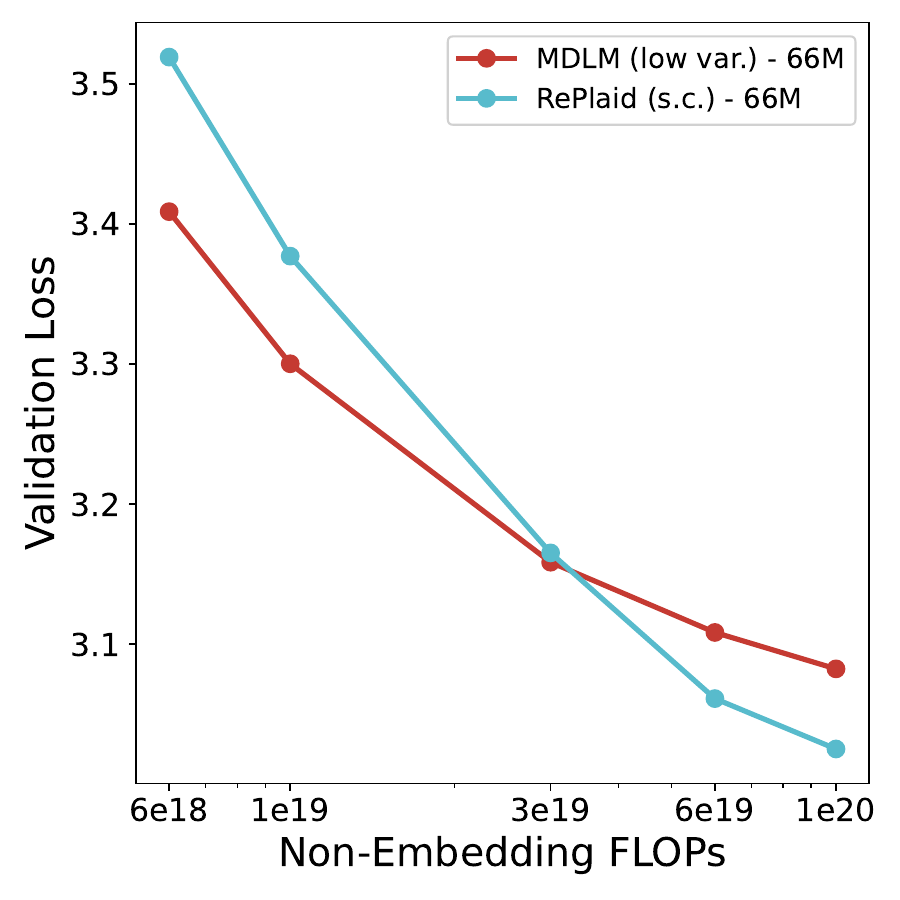}
        \caption{\centering Examples:\\ Over-training at 66M}
        \label{fig:overtrain-66m}
    \end{subfigure}
    \caption{\small \textbf{(a-b)} IsoFLOP curves identify optimal model sizes (black crosses) across fixed compute budgets. The optimal \textsc{RePlaid} (s.c.) loss exhibits power-law scaling, decreasing at a rate comparable to MDLM. To match AR loss, MDLM and \textsc{RePlaid} (s.c.) require $14\times$ and $\replaidscx\times$ the compute, respectively. In the over-trained regime below the {\color[RGB]{80, 150, 63}\textbf{green line}}, \textsc{RePlaid} (s.c.) consistently outperforms MDLM (\cref{subsec:scaling-law-results}). \textbf{(c)} The effect of over-training MDLM versus \textsc{RePlaid} (s.c.) at 66M, corresponding to slicing (a-b) at the vertical dashed lines.}
    \label{fig:isoflop-plots}
    \vspace{-1em}
\end{figure*}

\paragraph{IsoFLOP.} We perform an IsoFLOP analysis~\citep{hoffmann2022training} over compute budgets $C\in\{6\times10^{18}, 1\times10^{19}, 3\times 10^{19}, 6\times 10^{19}, 1\times 10^{20}\}$. For each target budget $C$, we train a set of models with different parameter counts $N$ (\cref{tab:model_sizes}) approximately up to the compute budget, producing (exact or an upper bound on) validation loss $\log \cL(N; C)$. The number of training tokens $D$ is decided by the number of training steps. As with~\citep{sahoo2026scaling}, we use $\mathcal{L} = \text{PPL}$ instead of $\cL=\text{NLL}$/token, so $\log \cL =\text{NLL}$/token.

For each fixed $C$, we fit a quadratic function in log-log ($\log \cL$ and $\log N$) space:
\begin{align*}
    \log \cL(N, C) \approx a_C (\log N)^2 + b_C \log N + c_C,
\end{align*}
which are plotted as dotted curves in~\cref{fig:isoflop-plots} for MDLM and RePlaid (s.c.) and~\cref{fig:scaling_law_appendix} for all methods. The coordinates of the minimum for each $C$ are
\begin{align*}
N^*_C := \argmin\nolimits_N \cL(N;C),\quad \cL^*_C := \cL(N^*_C; C).
\end{align*}
These are the compute-optimal model size ($N^*_C$) and the compute-optimal validation loss ($\cL^*_C$) at compute budget $C$, respectively. Pairs of $(N^*_C, \cL^*_C)$ are denoted by black crosses in~\cref{fig:isoflop-plots,fig:scaling_law_appendix}. This protocol is applied to AR, MDLM, Duo, and RePlaid. For diffusion models, we use variational upper bounds (e.g.,~\cref{eq:plaid-nelbo}) on their negative log-likelihoods $- \log p_\theta (\x)$. RePlaid shows well-behaved IsoFLOP curves~(\cref{fig:iso_flop_replaid_sc}), similar to those of AR, MDLM, and Duo (\cref{fig:iso_flop_mdlm} and \cref{fig:scaling_law_appendix}).

\paragraph{Scaling laws.} To obtain the compute-optimal scaling law for a method, we use its pairs $(C_i, \cL^*_{C_i})$ from the IsoFLOP analysis~(\cref{subsec:isoflop}) and fit a power law in log-log coordinates:
\begin{align*}
    (\alpha^*, \beta^*) = \argmin\nolimits_{\alpha, \beta} \tsum_{i=1}^n (\log \cL^*_{C_i} - \alpha \log C_i-\beta)^2.
\end{align*}\cref{fig:likelihood_scaling_law} plots compute-optimal validation loss ($\cL^*_C$) against compute ($C$). We apply a similar power law fit to the compute-optimal model size, $\log N^*_C \approx \alpha' \log C + \beta'$, which is plotted in~\cref{fig:param_scaling_law}.

\subsection{Results}\label{subsec:scaling-law-results}

\begin{figure*}[t]
    \centering
    \begin{subfigure}{0.45\linewidth}
        \centering
        \includegraphics[width=\linewidth]{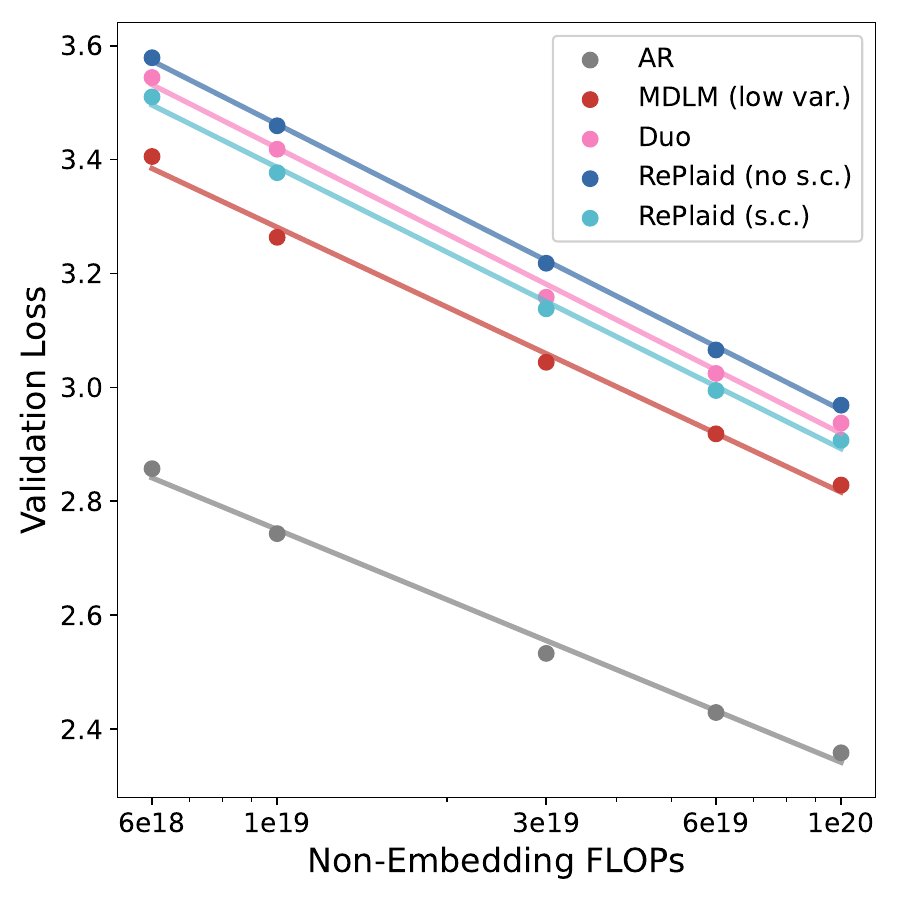}
        \caption{Likelihood scaling law}
        \label{fig:likelihood_scaling_law}
    \end{subfigure}
    \begin{subfigure}{0.45\linewidth}
        \centering
        \includegraphics[width=\linewidth]{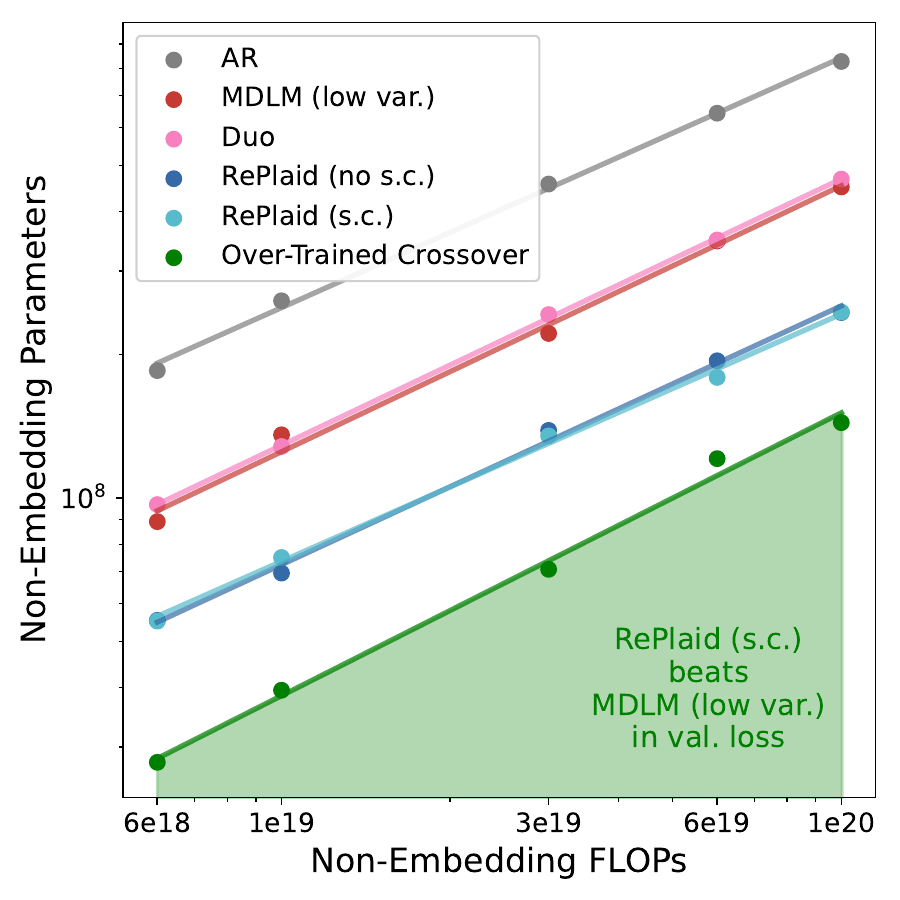}
        \caption{Parameter scaling law}
        \label{fig:param_scaling_law}
    \end{subfigure}
    \caption{\small \textbf{Scaling laws.} \textbf{(a)} The compute-optimal \textsc{RePlaid} loss exhibits power-law scaling, decreasing at a rate comparable to AR, MDLM, and Duo. MDLM requires $14\times$ FLOPs to match AR; Duo needs $22\times$; \textsc{RePlaid} consumes $\replaidscx\times$ with self-conditioning (s.c.) and $\replaidnscx\times$ without it. \textbf{(b)} The compute-optimal \textsc{RePlaid} (s.c.) uses $1.8\times$ fewer parameters than MDLM and Duo -- while outperforming Duo's loss in (a) -- and uses $3.4\times$ fewer than AR. For a given (non-embedding) model size, MDLM and \textsc{RePlaid} (s.c.) match loss at the {\color[RGB]{80, 150, 63}\textbf{green line}}.}
    \label{fig:scaling_law}
     \vspace{-1em}
\end{figure*}

We first validate our pipeline by reproducing that MDLM with low-variance (low var.) training requires $ 14\times$ compute to match AR validation loss~\citep{sahoo2026scaling}; for brevity, we refer to this as MDLM. For Duo, we obtain a $22\times$ compute gap, closely matching $23\times$ reported by~\citet{sahoo2026scaling}.

\paragraph{Competitive loss scaling.} \cref{fig:likelihood_scaling_law} shows that RePlaid scales competitively with discrete DLMs. The compute-optimal loss for RePlaid (s.c.) exhibits power-law scaling, decreasing
at a rate comparable to AR, MDLM, and Duo. The $\approx \replaidscx \times$ compute gap of RePlaid (s.c.) outperforms Duo ($\approx22\times$) and approaches MDLM ($\approx14\times$). RePlaid (no s.c.) exhibits a gap of $\approx \replaidnscx \times$, approaching Duo.

\paragraph{Strong parameter efficiency.} RePlaid also demonstrates significant parameter efficiency. To reach its compute-optimal frontier, RePlaid requires $\approx3.4\times$ fewer parameters than the AR baseline (\cref{fig:param_scaling_law}). Crucially, RePlaid (s.c.) uses $\approx1.8\times$ fewer parameters than MDLM and Duo (\cref{fig:param_scaling_law}), achieving this size reduction while outperforming Duo's compute-optimal loss (\cref{fig:likelihood_scaling_law}). 

\paragraph{Outperforming MDLM at over-training.} Obtaining high-quality small models through over-training is critical for heavy inference~\citep{sardana2024beyond}. RePlaid (s.c.) offers distinct advantages when training past compute-optimal budgets to minimize inference costs. A crossover occurs in this over-trained regime (\cref{fig:isoflop-plots}, below green line), where RePlaid (s.c.) predictably achieves a lower loss than MDLM. More precisely, for a given model size, a RePlaid (s.c.) trained for $\approx 3.1\times$ its optimal compute matches the loss of an MDLM trained for $\approx 6.9\times$ its optimal compute (\cref{fig:param_scaling_law}, green line).

\paragraph{Ablations.} Applying RePlaid's length-normalized, low-dimensional ($d_e=16$) embeddings, linear up-projection, and optimization adjustments (\cref{subsec:scaling-law-setup}) to MDLM ($d_e=768$ by default) degrades its performance (\cref{fig:iso_flop_mdlm_exchange_appendix}), and vice versa (\cref{fig:iso_flop_replaid_sc_exchange_appendix}). This aligns with our observation in~\cref{subsec:geometric_reg}: RePlaid learns low-rank embedding structures, which do not emerge for discrete methods like AR or MDLM.

\begin{table}[t]
  \caption{\small \textbf{Test perplexities} (PPL; $\downarrow$) on LM1B ($L=128$, 1M steps) and OWT ($L=1024$, 1M steps). For diffusion models, we report PPL computed using the NELBO~\cref{eq:plaid-nelbo} as in prior work.  The best value in each method category is \textbf{bolded}; top-3 diffusion values per column are \textbf{shaded}; {\color{black}$^\texttt{d}$MD4 and GenMD4 do not participate in the ranking because they use learning rate decay, which improves convergence.} ``s.c.'' indicates that self-conditioning is used during both training and evaluation. $\P$No sentence packing. $^*$Reported in~\citet{he2022diffusionbert}. $^\ddag$Reported in~\citet{sahoo2025the}. $^\dag$Reported in~\citet{sahoo2025esoteric}. $^\texttt{a}$Approximated by evaluating the checkpoint (trained with s.c.) with s.c. off. $^\texttt{r}$Models retrained from scratch. For LangFlow, we evaluate its official checkpoint with our NELBO~\cref{eq:plaid-nelbo} because we found its ODE-based NELBO to be sensitive to design choices (see \cref{tab:ode-chainrule} in \cref{app:ode-likelihood}). }
  \vspace{0.5em}
  \label{tab:val-ppl}

  \begin{minipage}[b]{1.0\linewidth}
    \centering
    \footnotesize
    \begin{tabular}{llcccc}
    \toprule
     & & & {LM1B} & \multicolumn{2}{c}{OWT} \\
     \cmidrule(lr){4-4} \cmidrule(lr){5-6}
     & Method & s.c. & {1M} & {250K} & {1M} \\
    \midrule
    \multicolumn{3}{l}{\textit{Autoregressive}}\\
    & Transformer~\citep{dai2019transformer} &  & $22.8^{\dagger}$ & $17.8^\dag$ & $17.5^\ddag$ \\
    \midrule
    \multicolumn{3}{l}{\textit{Discrete Diffusion (Mask)}}\\
    & BERT-Mouth~\citep{wang2019bert} &  & $142.9^*$ & - & - \\
    & D3PM Absorb~\citep{austin2021structured} &  & $76.9$ & - & - \\
    & DiffusionBert~\citep{he2022diffusionbert} &  & $63.8$ & - & - \\
    & Eso-LMs ($\alpha_0=1$)~\citep{sahoo2025esoteric} & & $36.1$ & $30.1$ & - \\
    & SEDD Absorb~\citep{lou2024discrete} & & $32.7^\P$ & $26.8^\dag$ & - \\
    & MDLM~\citep{sahoo2024simple} & & \topthree{$31.8$}$^\ddag$ & \topthree{$25.2$}$^\dag$ & \topthree{$23.2$}$^\ddag$ \\
    & MDLM (low var.)~\citep{sahoo2025esoteric} & & \topone{$\mathbf{30.8}$}$^\texttt{r}$ & \toptwo{$\mathbf{25.0}$}$^\texttt{r}$ & \toptwo{$\mathbf{23.1}$}$^\texttt{r}$ \\
    & MD4~\citep{shi2024simplified} & & - & - & \color{black}{$22.1^\texttt{d}$} \\
    & GenMD4~\citep{shi2024simplified} & & - & - & \color{black}{$21.8^\texttt{d}$} \\
    \midrule
    \multicolumn{3}{l}{\textit{Discrete Diffusion (Uniform)}}\\
    & D3PM Uniform~\citep{austin2021structured} & & $137.9$ & - & - \\
    & SEDD Uniform~\citep{lou2024discrete} & & $40.3^\P$ & - & - \\
    & UDLM~\citep{schiff2025simple} & & $36.7^\ddag$ & $30.5^\dag$ & $27.4^\ddag$ \\
    & Duo~\citep{sahoo2025the} & & $\mathbf{33.7}^\ddag$ & $\mathbf{27.1}^\dag$ & $\mathbf{25.2}^\ddag$ \\
    \midrule
    \multicolumn{3}{l}{\textit{Continuous Diffusion}}\\
    & Diffusion-LM~\citep{li2022diffusion} & & $118.6^*$ & - & - \\
    & LangFlow~\citep{chen2026langflow} & \ccross & - & - & $38.4^\texttt{a}$  \\
    & & \cmark & - & - & $32.2$ \\
    & Plaid~\citep{gulrajani2023plaid} & \ccross & $39.4$ & $27.5$ & $25.7$ \\
    &  & \cmark & $37.7$ & $25.4$ & $24.4$ \\
    \midrule
    & \textbf{\textsc{RePlaid} (Ours)} & \ccross & $33.3$ & $26.8$ & $\replaidnscppl$ ($24.1^\texttt{a}$) \\
    &  & \cmark & \toptwo{$\mathbf{31.6}$} & \topone{$\mathbf{24.9}$} & \topone{$\mathbf{\replaidscppl}$} \\
    \bottomrule
    \end{tabular}
  \end{minipage}%
  \hfill%
\vspace{-1em}
\end{table}

\section{Benchmarking RePlaid for Diffusion Language Models}\label{sec:small-models}

The $0.1$B regime is the standard testbed for DLMs where the recent wave of continuous DLMs~\citep{lee2026flow, chen2026langflow, roos2026categorical, potaptchik2026discrete} is exclusively benchmarked. Having shown in~\cref{sec:scaling} that RePlaid scales on par with discrete DLMs, we now compare at this canonical scale against the strongest discrete (MDLM~\citep{sahoo2024simple}, Duo~\citep{sahoo2025the}) and continuous (FLM~\citep{lee2026flow}, LangFlow~\citep{chen2026langflow}, ELF~\cite{hu2026elf}) DLMs along \emph{likelihood} (\cref{subsec:likelihood-eval}), \emph{unconditional generation} (\cref{subsec:sampling}), and \emph{conditional generation} on a verifiable reasoning task (\cref{subsec:tinygsm}). All three are necessary: FLM and ELF do not propose a likelihood bound; in addition, recent works~\citep{sahoo2025the, sahoo2026scaling} have shown that better likelihood may not always lead to better sample quality. 

\subsection{Likelihood Evaluation}
\label{subsec:likelihood-eval}

\paragraph{Unified setup.} AR, MDLM, Duo, FLM, LangFlow, and RePlaid use the same transformer architecture from~\cref{sec:scaling} (\cref{tab:model_sizes_0d1b}). They share a comparable optimization setting (see \cref{tab:replaid-dit-optimization-small} for RePlaid) consistent with prior work at 0.1B, slightly different from the one used in \cref{sec:scaling} (\cref{tab:plaid-dit-optimization-scaling}). Training takes 1M steps with batch size $512$ (\textasciitilde$110$ hours on 16 GB200s). For original Plaid, we only set its layer count, hidden size, and LR schedule to be the same as other methods (\cref{tab:plaid-dit-optimization-small}).

\begin{findingbox}
    RePlaid achieves the best PPL bound $\mathbf{\replaidscppl}$ among considered DLMs on OpenWebText.
\end{findingbox}

\paragraph{Results.} On OpenWebText (OWT)~\citep{Gokaslan2019OpenWeb}, RePlaid (s.c.) attains the best PPL bound $\mathbf{\replaidscppl}$ among considered DLMs (\cref{tab:val-ppl}), improving over the strongest discrete baseline MDLM (low var., $23.1$) and outperforming Duo ($25.2$), the original Plaid ($24.4$), and the concurrent LangFlow ($32.2$). Interestingly, RePlaid (s.c.) exhibits a smaller gap to MDLM (low var.) at 250K, which aligns with the over-training behavior observed in~\cref{sec:scaling}. Even without self-conditioning, RePlaid reaches $\replaidnscppl$, which is worse than MDLM (low var.) but is already below Duo's $25.2$ and far ahead of LangFlow's $38.4$. On LM1B~\citep{chelba2014billion}, RePlaid (s.c.) reaches a top-2 PPL of $31.6$, surpassing Duo but not MDLM. For reference, we further include the \textbf{training dynamics} of RePlaid (s.c.) on OWT in~\cref{app:dynamics}.

\begin{wraptable}{r}{0.38\linewidth}
    \vspace{-1em}
    \centering
    \caption{\small OWT ablations for 1M steps.}
    \label{tab:replaid-ablations}
    \vspace{-0.5em}
    \resizebox{\linewidth}{!}{%
    \begin{tabular}{lc}
    \toprule
    Method & PPL ($\downarrow$) \\
    \midrule
    \textbf{\textsc{RePlaid} (s.c.) (Ours)} & $\mathbf{\color[rgb]{0, 0.5, 0}\replaidscppl}$\\
    \quad w/o output prior &  $22.5$ \\
    \quad or w/o self-conditioning & $\replaidnscppl$ \\
    \quad or w/o learnable noise sched. & $24.4$\\
    \quad or w/o learnable embeddings  & $\mathbf{\color[rgb]{0.5, 0, 0}39.4}$ \\
    \bottomrule
    \vspace{-2em}
    \end{tabular}%
    }
\end{wraptable}
\paragraph{Ablations.} \cref{tab:replaid-ablations} ablates RePlaid components. Among these, \textbf{learning the token embeddings} is the most important: freezing them to be random while keeping other components worsens PPL from $\mathbf{\color[rgb]{0, 0.5, 0}\replaidscppl}$ to $\mathbf{\color[rgb]{0.5, 0, 0}39.4}$, a degradation that makes RePlaid the worst DLM on OWT in~\cref{tab:val-ppl}. Learning the noise schedule contributes the next largest gain ($\Delta = 2.3$), followed by self-conditioning ($\Delta = 1.5$) and the output prior ($\Delta = 0.4$). Crucially, even when omitting any one of self-conditioning or the output prior or the noise schedule, RePlaid still outperforms Duo, indicating that the gains stem mainly from an optimized embedding geometry. This offers a complementary perspective on not optimizing a true likelihood bound in LangFlow~\citep{chen2026langflow} and the use of fixed, one-hot embeddings in FLM~\citep{lee2026flow}, DFM~\citep{potaptchik2026discrete}, and CFM~\citep{roos2026categorical}: learning embedding geometry with a true likelihood bound can be particularly beneficial in likelihood-based continuous DLMs.

\begin{figure*}[t]
    \centering
    \begin{subfigure}{0.333\linewidth}\phantomsubcaption\label{fig:genppl_vs_t}\end{subfigure}%
    \begin{subfigure}{0.333\linewidth}\phantomsubcaption\label{fig:mauve_vs_T}\end{subfigure}%
    \begin{subfigure}{0.333\linewidth}\phantomsubcaption\label{fig:genppl_vs_entropy}\end{subfigure}
    \includegraphics[width=1.0\linewidth]{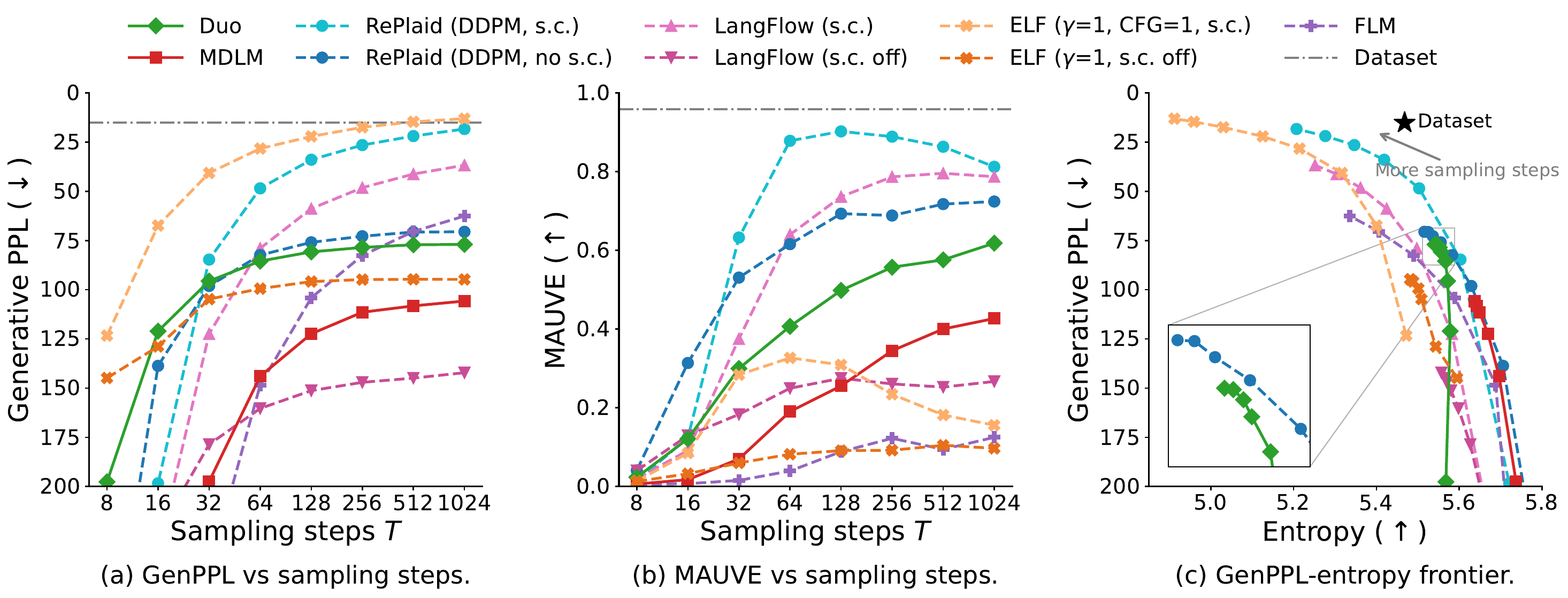}
    \caption{ \small \textbf{RePlaid achieves superior unconditional generation.} We sample from 0.1B models trained for 1M on OWT ($L=1024$, $N=5120$, no temperature). \textbf{(a-b)} GenPPL and MAUVE. GenPPL does not capture diversity while MAUVE does. \textbf{(c)} GenPPL-entropy frontier as $T$ increases. We observe that Duo, LangFlow, ELF, and FLM have worse entropy than RePlaid (no s.c.) at comparable GenPPL.}
    \label{fig:sample-quality}
    \vspace{-1em}
\end{figure*}

\subsection{Unconditional Sampling: OpenWebText}
\label{subsec:sampling}

Unlike discrete DLMs, continuous DLMs admit \textit{acceleration} via efficient ODE solvers and trajectory distillation. While distillation is beyond the scope of our work, we study RePlaid's sample quality across ODE solvers to gauge its distillation potential. A trained RePlaid samples a sequence by simulating the backward trajectory from $\z_1 \sim \normal(\zero, \I)$ to $\z_0$. Going beyond the original Plaid~\citep{gulrajani2023plaid}, which only used a stochastic DDPM-style \emph{ancestral sampler} drawing $\z_s \sim p_\theta(\z_s \mid \z_t)$ from the closed-form Gaussian posterior of VDM~\citep[App.~A.4]{kingma2021variational} (see \cref{eq:vdm-reverse,eqn:denoising}), we also evaluate three deterministic samplers (DDIM~\citep{song2021denoising}, DPM-Solver++(2M)~\citep{lu2025dpm}, Heun~\citep{karras2022elucidating}) that integrate along the \emph{probability-flow ODE (PFODE)}:
\begin{equation}\label{eq:plaid-pfode}
  \dot\z_t = \v_\theta(\z_t, t) := (\dot\alpha_t - \alpha_t \dot\sigma_t / \sigma_t) \x_\theta(\z_t, t)\Emb + (\dot\sigma_t / \sigma_t) \z_t.
\end{equation}
These three deterministic samplers we adopt are mature PFODE integrators developed for image diffusion and transferred here largely unchanged.
All solvers read out the final tokens as $\argmax$ of the decoder logits at $t \approx 0$ as with FLM and LangFlow. See~\cref{app:sampler} for per-step update formulas.

\paragraph{Setup.} We sample unconditionally ($L=1024$, $N=5120$) from RePlaid (s.c.) and (no s.c.) trained on OWT for 1M steps. No temperature is used. During sampling, we vary the number of diffusion discretization steps $T$ (not including the final $\argmax$) to control NFEs. For MDLM, Duo, FLM, LangFlow, and ELF, we use their official, publicly available 1M checkpoints. We measure Generative perplexity (GenPPL)~\citep{dieleman2022continuous} via GPT-2 Large~\citep{radford2019language}, MAUVE\footnote{Reference set size is also 5120. For ELF, we draw from its official data due to slightly different preprocessing.}~\citep{pillutla2021mauve} via ModernBERT-Large~\citep{warner2025smarter} for sample quality, and average unigram entropy for diversity~\citep{zheng2024masked}. GenPPL is the standard metric used in prior work~\citep{sahoo2024simple, sahoo2025the, lee2026flow, potaptchik2026discrete, chen2026langflow}. Inspired by~\citep{dieleman2022continuous, sahoo2025esoteric}, we also measure MAUVE, a metric designed to capture both quality and diversity that has been shown to correlate with human judgements~\citep{pillutla2021mauve}.

\begin{findingbox}
With a standard DDPM solver and a uniform $t$ discretization, \textbf{RePlaid (no s.c.)} generates at (1) higher quality than Duo and (2) significantly higher quality than other continuous DLMs without self-conditioning at all sampling steps. \textbf{RePlaid (s.c.)} struggles at low sampling steps (as with LangFlow) but further improves upon RePlaid (no s.c.) for $T\geq64$, outperforming discrete DLMs and other continuous DLMs with self-conditioning at high sampling steps.
\end{findingbox}

\paragraph{Results.} \cref{fig:sample-quality} reports GenPPL (a) and MAUVE (b) versus sampling steps $T$ without temperature. On GenPPL (\cref{fig:genppl_vs_t}), we confirm the finding of~\citep{sahoo2025the} that Duo outperforms MDLM. Crucially, as $T$ increases, all continuous methods appear to match or outperform Duo, except for LangFlow (s.c. off). Among these, RePlaid (s.c.) becomes noticeably strong starting from $T=64$ and then reaches a GenPPL of $\approx21$ at $T = 1024$, significantly outperforming other continuous methods except ELF. 

On MAUVE (\cref{fig:mauve_vs_T}), the conclusion holds similarly, except that (\textit{i}) RePlaid (no s.c.) outperforms Duo at every $T$, (\textit{ii}) FLM and ELF significantly underperform other methods, and (\textit{iii}) RePlaid (s.c.) shows a decrease in quality at $T\in\{512, 1024\}$. To understand these issues, we plot GenPPL against entropy as $T$ is varied in~\cref{fig:genppl_vs_entropy}. We see that Duo, FLM, and ELF have worse entropy than RePlaid (no s.c.) at comparable GenPPL. Meanwhile, RePlaid (s.c.) yields the lowest entropy among all methods at high $T$. We attribute this to the \textit{self-conditioning loop}: each sampling step feeds the model's previous reconstruction back as a conditioning signal, so additional steps compound the network's confidence and sharpen the sampling distribution below the reference entropy. While MAUVE accounts for both quality and diversity, to affirm our results we sweep temperatures to confirm RePlaid's strong sampling performance in \cref{app:temp}. See~\cref{app:samples} for RePlaid text samples.

%

%

%

%

We further compare RePlaid's DDPM ancestral sampler against the deterministic PFODE solvers (DDIM, DPM-Solver++(2M), Heun) in~\cref{app:ode-samplers}: \textbf{higher-order solvers improve quality at low NFEs} but, at high NFEs, converge to the first-order DDIM and fall behind DDPM.

\subsection{Conditional Sampling: TinyGSM -- Mathematical Reasoning with Python Code}
\label{subsec:tinygsm}

We train 0.1B RePlaid (same architecture and optimization as in~\cref{subsec:likelihood-eval}) for 250K steps on TinyGSM~\citep{liu2023tinygsm} and evaluate on the GSM8K~\citep{cobbe2021gsm8k} test set. This follows the exact experimental setup used in~\citet{deschenaux2026hyperspherical} (\cref{app:tinygsm}). We sweep both logit and score temperature. The score temperature~\citep{dieleman2022continuous}, available only to continuous DLMs, is the effective sharpening knob: among diffusion without self-conditioning, RePlaid achieves an accuracy of $42.5\%$, competitive with the best result from discrete DLMs ($40.3\%$). See \cref{tab:tinygsm} for the full comparison against all baselines.

%
\section{Theoretical Insights}
\label{sec:theory}
Having shown that RePlaid attains the best PPL bound and sample quality among continuous DLMs on OpenWebText, we now ask why, given that its training objective is a plain VDM NELBO with neither cross-entropy (CE) nor hand-engineered time reparameterization~\citep{dieleman2022continuous,lee2026flow, potaptchik2026discrete, roos2026categorical, chen2026langflow}. We argue that the answer is in the \textbf{NELBO} itself: it is a true upper bound on the negative log-likelihood while the alternatives are not, and thus the rest of this section traces two consequences of optimizing this bound. $(i)$ Adding an extra CE term to the NELBO provides a \textit{separative} force on the learned embeddings $\Emb$, which disrupts the low-rank embedding structure and hurts the PPL; $(ii)$ Making the noise schedule a learnable part of the NELBO and training it to minimize Monte-Carlo variance automatically yields a \textit{constant per-timestep diffusion loss} and a \textit{near-linear per-timestep CE}---the time-uniformity property recent CE-based methods engineer with manually-designed time reparameterization.

\subsection{ELBO versus Cross-Entropy Training}
\label{subsec:geometric_reg}

Early embedding-based continuous DLMs such as CDCD~\citep{dieleman2022continuous} and more recent ones such as FLM~\citep{lee2026flow} and LangFlow~\citep{chen2026langflow} adopt the CE loss instead of the MSE-style diffusion loss \cref{eq:plaid-nelbo}. This discrepancy brings a natural question: can RePlaid be trained using only CE loss? We found that this requires non-trivial modifications (e.g., LangFlow~\citep{chen2026langflow} aims to achieve exactly this), as all components of RePlaid are jointly optimized for likelihood. To conduct a meaningful, isolated comparison where no component becomes degenerate under a standalone CE loss,  we compare RePlaid (s.c.) trained with $\cL_\NELBO(\x)$ with RePlaid (s.c.) trained with a combined loss $\cL_\NELBO(\x) + w \cdot \cL_{\CE}(\x)$, where
\begin{align}
    \cL_{\CE}(\x) = \mathbb{E}_{t\sim\unif[0,1],\z_t\sim q(\z_t|\x)}\left[\tsum_{l=1}^L  -\log \langle \x_\theta^l(\z_t,t), \x^l \rangle\right]
\end{align}
and $\nabla\cL_{\CE}$ is constrained to only reach $\Emb$ and $\theta$ so noise schedule components stay non-degenerate.

For RePlaid trained with $\cL_\NELBO(\x)$, \cref{fig:tsne} projects the rows of $\Emb \in \R^{V \times d_e}$ to two dimensions via t-SNE~\citep{van2008visualizing}, with each subword colored by its most frequent participating part-of-speech (POS) assigned by spaCy~\citep{honnibal2020spacy} on a held-out OWT subset (see \cref{supp:sec:spacy}). Even at the small embedding dimension $d_e = 16$, the projection exhibits visible POS clustering, and the PCA scree in \cref{fig:pca} shows the embeddings concentrating on \textasciitilde$7$ principal components ($90\%$ cumulative explained variance)---neither collapsed nor uniformly packed, but \textit{low-rank} and \textit{linguistically meaningful}. Note that we do not claim that these features cause the likelihood gain, but that they characterize NELBO solutions.

Adding a CE loss makes the gradient for $\Emb$ more \textit{separative}: the PCA scree flattens (\cref{fig:pca-aux}) and the PPL bound worsens. Increasing $w$ consistently leads to flatter PCA screes and worse PPL bounds (\cref{app:ce-sweep}). This suggests that optimizing $\cL_\NELBO(\x)$ requires delicate modulation of embedding geometry, which can be disrupted by the discriminative CE loss. We verify this in the realistic scenario of RePlaid vs LangFlow: RePlaid ($d_e=768$) achieves a substantially lower-rank embedding geometry than LangFlow ($d_e=768$ by default) while achieving a better PPL bound (\cref{fig:pca-scree-768}). For reference, MDLM ($d_e=768$) learns a much higher-rank embedding geometry than both (\cref{fig:pca-scree-768-mdlm}).

\begin{figure*}[t]
    \centering
    \begin{subfigure}{0.37\linewidth}
        \centering
        \includegraphics[width=\linewidth]{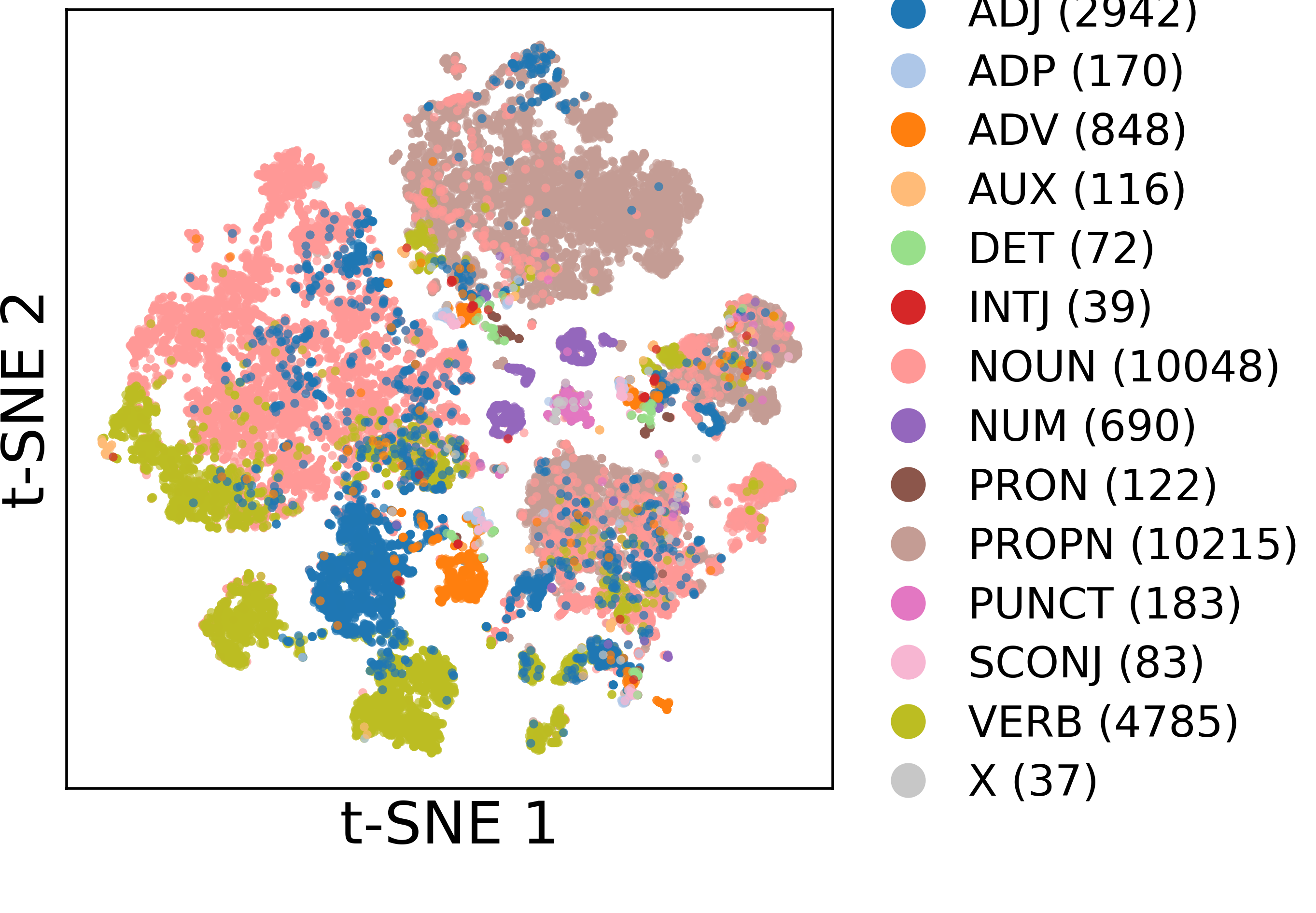}
        \caption{t-SNE ($\cL_{\NELBO}$).}
        \label{fig:tsne}
    \end{subfigure}
    \begin{subfigure}{0.30\linewidth}
        \centering
        \includegraphics[width=\linewidth]{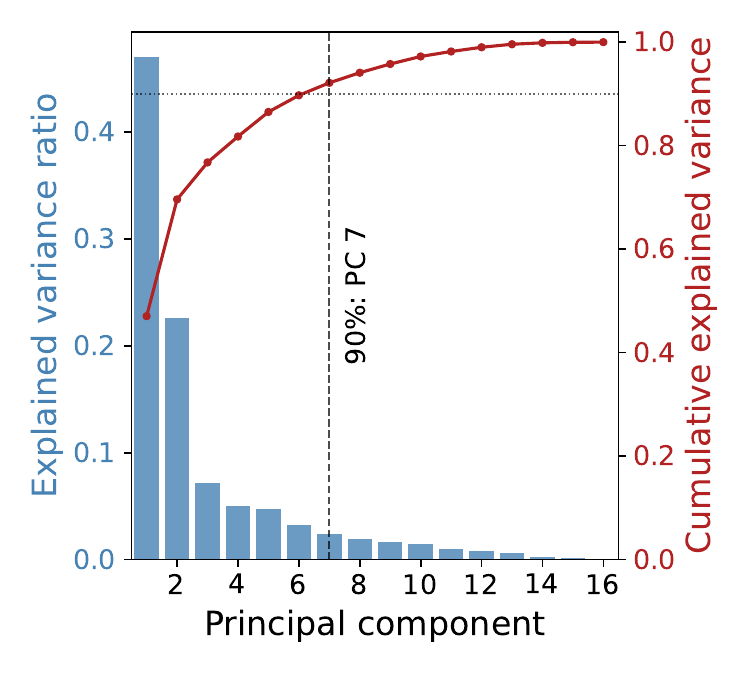}
        \caption{PCA scree ($\cL_{\NELBO}$).}
        \label{fig:pca}
    \end{subfigure}
    \begin{subfigure}{0.30\linewidth}
        \centering
        \includegraphics[width=\linewidth]{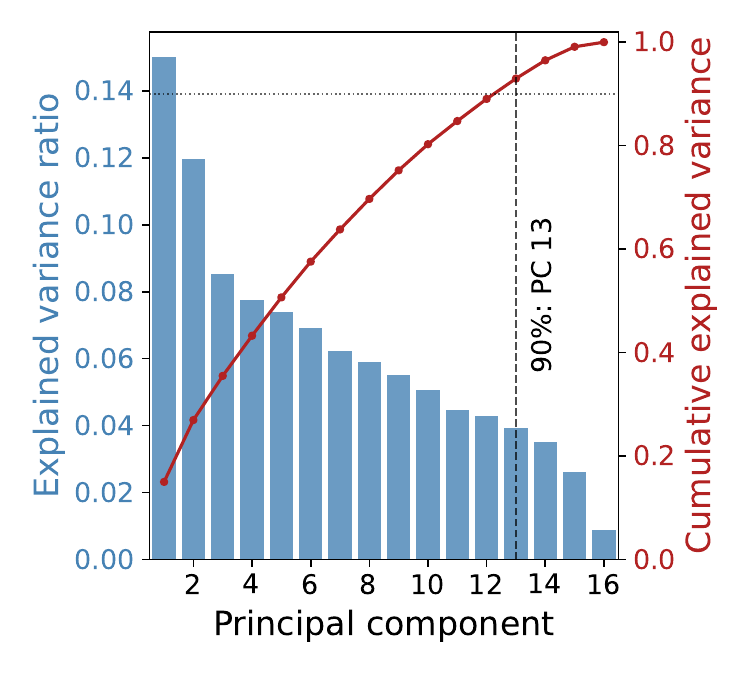}
        \caption{PCA scree ($\cL_{\NELBO} + \cL_{\CE}$).}
        \label{fig:pca-aux}
    \end{subfigure}
    \caption{\small \textbf{Geometry of learned embeddings }of \textsc{RePlaid} (s.c.) (\textbf{OWT Test PPL}: $\replaidscppl$ at 1M). 
    (\textbf{a}) 2D t-SNE plot with each subword colored by its most frequent POS tag.
    (\textbf{b}) PCA scree plot of $\Emb$.
    (\textbf{c}) PCA scree plot of $\Emb$ when an auxiliary CE loss is added (\cref{subsec:geometric_reg}), dispersing the embeddings (\textbf{OWT Test PPL}: $26.1$ at 1M).}
    \label{fig:geometry}
     \vspace{-1em}
\end{figure*}

\subsection{Benefits of Learning the Noise Schedule} 
\label{subsec:learning-noise-schedule}

Prior work on continuous DLMs stresses \textit{distributing denoising difficulty uniformly} across time, via manually-designed schedules that linearize token identifiability or the decoding error rate in time~\citep{pynadath2025candi,lee2026flow,potaptchik2026discrete}, or match the noise density to the information-gain (measured by cross entropy) rate~\citep{chen2026langflow}. We show below that minimizing the variance of the ELBO loss recovers the same property automatically for \textit{any} $\Emb$: the optimal noise schedule makes the per-timestep diffusion loss constant in $t$ (\cref{thm:constant_diffusion_loss}) and, under a Bayes-optimal denoiser, the per-timestep cross-entropy decomposes into a linear-in-$t$ term and a residual term (\cref{thm:linear_ce}).

For some given embeddings $\Emb$, a denoising model $\x_\theta$, and a noise schedule $\gamma$ with corresponding $\SNR(t)=\ee^{-\gamma(t)}$, the model's expected diffusion loss at timestep $t$ is
\begin{align}\label{eq:diffusion-loss-per-t}
    \ell_{\theta, \gamma}(t) := -\tfrac{1}{2} \SNR'(t) \E_{q_\data(\x) q(\z_t \mid \x)}\| \e_\theta(\z_t, t) - \e \|^2 =: -\tfrac{1}{2} \SNR'(t) \MSE_\theta(\SNR(t)).
\end{align}
\citet{kingma2021variational} showed that the diffusion loss  $\E_{t\sim\unif[0,1]}[\ell_{\theta, \gamma}(t)]$ depends on $\theta$ and $\gamma_0,\gamma_1$, but is invariant to the interior shape of $\gamma(t)$; for fixed $\theta$, $\gamma_0$ and $\gamma_1$, $\Var_t [\ell_{\theta, \gamma}(t)]$ depends only on the shape of $\gamma(t)$. We advance this by proving~\cref{thm:constant_diffusion_loss} \&~\cref{lemma:linear_info_decay}, which holds for VDMs and (Re)Plaid.

\vspace{0.5em}
\begin{restatable}[\textbf{\color{ourbrightblue!80!black}Constant Per-Timestep Diffusion Loss}]{proposition}{thmconstantdiffusion}
\label{thm:constant_diffusion_loss}
Let $\kappa:=\E_t[\ell_{\theta, \gamma}(t)]$. Under the endpoint constraint $\gamma(0)=\gamma_0$ and $\gamma(1)=\gamma_1$, there exists a unique noise schedule $\gamma^*$ such that $\ell_{\theta, \gamma^*}(t)\equiv\kappa\geq0$ for all $t$, and consequently $\Var_t [\ell_{\theta, \gamma^*}(t)]=0$.
\end{restatable}

See~\cref{supp:sec:constant-diffusion-proof} for the proof, and see \cref{eq:optimal-gamma,eq:optimal-per-timestep-diffusion} for the closed-form formulas for $\gamma^*$ and $\kappa$, respectively. We show that the learned noise schedule closely matches the theoretical closed-form solution computed for a RePlaid checkpoint obtained mid-training (\cref{fig:learned-gamma}). High $\alpha(t)$ across a large range of $t$ indicates that more denoising effort is spent near clean data. Replacing the learned schedule with the closed-form schedule leads to a perfectly flat per-timestep diffusion loss by construction (\cref{fig:diff-loss-vs-t}). 

But the diffusion loss is a weighted MSE loss and does not reflect denoising difficulty. Thus we follow prior work~\citep{chen2026langflow} and measure the (expected) per-timestep CE loss for practical \textit{one-shot decoding}:
    \begin{align}\label{eq:ce-loss-per-t}
    \CE_{\theta, \gamma} (t) \! := \! \E_{q_\data(\x)\,q(\z_t \mid \x)}[-\log \hat p_\theta (\x\mid\z_t)] \! = \! \E_{q_\data(\x)\,q(\z_t \mid \x)}\big[\tsum_{l=1}^L \! - \! \log \langle \x_\theta^l(\z_t,t), \x^l \rangle \big].
\end{align}
\vspace{-1.5em}

While any model can achieve constant per-timestep diffusion loss, per-timestep CE loss cannot be analyzed for an arbitrary model, as one could imagine an adversary that is only bad at denoising at some $t$. For analyses below, we assume a Bayes-optimal denoiser, which we approach in training.

\begin{figure*}[t]
    \centering
    \begin{subfigure}{0.5\linewidth}\phantomsubcaption\label{fig:diff-loss-vs-t}\end{subfigure}%
    \begin{subfigure}{0.5\linewidth}\phantomsubcaption\label{fig:learned-gamma}\end{subfigure}\\
    \begin{subfigure}{0.5\linewidth}\phantomsubcaption\label{fig:ce-loss-vs-t}\end{subfigure}%
    \begin{subfigure}{0.5\linewidth}\phantomsubcaption\label{fig:dec-error-vs-t}\end{subfigure}
    \includegraphics[width=1.0\linewidth]{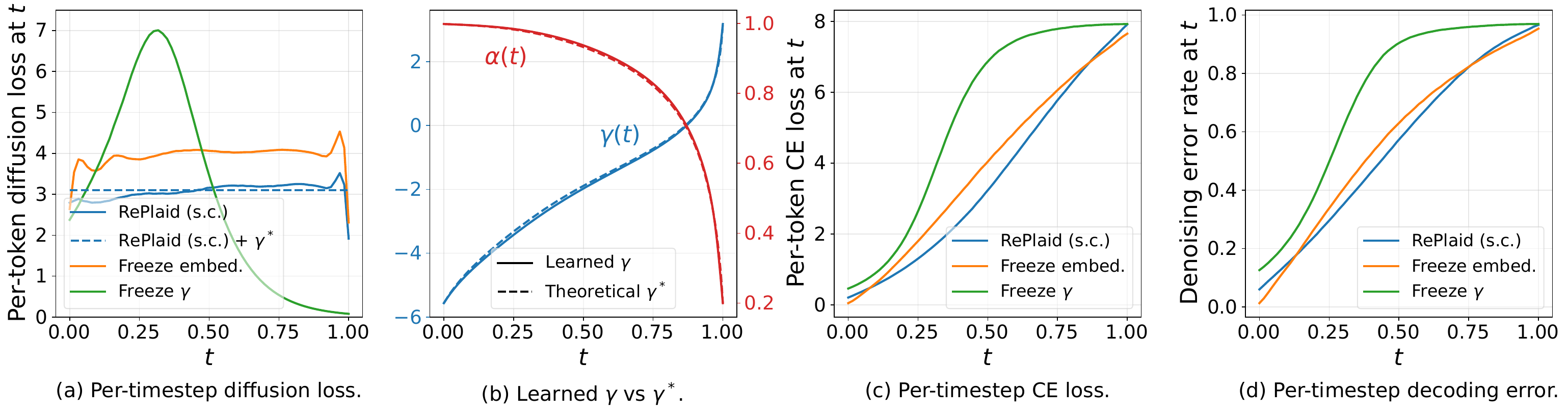}
    \caption{\small Per-timestep diffusion loss, CE loss, and decoding error for \textsc{RePlaid} (s.c.) (\textbf{OWT Test PPL}: $24.9$), an ablation that learns noise schedule shape but freezes embeddings (\textbf{OWT Test PPL}: $45.1$), and an ablation that learns embeddings but freezes the noise schedule (\textbf{OWT Test PPL}: $28.0$). Models are 250K and the two losses are length-normalized. Whenever the noise schedule is learned, the per-timestep diffusion and CE loss are \textit{near constant} and \textit{near linear} respectively. We validate this for an extensive set of embeddings (\cref{sec:linear-ce-diverse}).}
    \label{fig:loss-vs-t}
    \vspace{-1em}
\end{figure*}

\vspace{0.5em}
\begin{definition}[\textbf{\color{ourbrightblue!80!black}Bayes-optimal denoiser}]\label{def:bayes-optimal}
Assume \emph{no embedding collapse}, i.e., any two rows of $\Emb$ are different. A denoiser $\x_\theta:\R^{L\times d_e}\times[0,1]\to(\Delta^V)^L$ is \textbf{Bayes-optimal}, denoted $\theta^*$, if
\begin{align*}
    \x_{\theta^*}^l(\z_t, t) = q(\x^l \mid \z_t),
    \quad\text{or equivalently,}\quad
    \e_{\theta^*}^l(\z_t, t) = \E[\e^l \mid \z_t],
    \quad
    \forall~t\in[0,1],~l\in\{1,\dots,L\}.
\end{align*}
\end{definition}

Combining \cref{def:bayes-optimal}, \cref{thm:constant_diffusion_loss}, and the I-MMSE identity~\citep{guo2005mutual} gives rise to our Linear Information Decay \cref{lemma:linear_info_decay} (proved in \cref{supp:sec:linear-info-decay}), which profoundly shows that the mutual information $I(\e; \z_t)$ decays exactly linearly in $t$ at a rate dictated by the diffusion loss, $\kappa$. Finally, applying \cref{lemma:linear_info_decay} to the per-timestep CE loss \cref{eq:ce-loss-per-t} decomposes it into a linear-in-$t$ trend and a non-negative residual:

\vspace{0.5em}
\begin{restatable}[\textbf{\color{ourbrightblue!80!black}Per-Timestep CE Under Optimality}]{proposition}{thmlinearce}
\label{thm:linear_ce}
    Assuming no embedding collapse, we have
    \begin{equation}
        \CE_{\theta^*, \gamma^*}(t) = \underbrace{H(\x) - I(\e; \z_0) + \kappa t}_{\text{linear in } t \text{ (trend)}} \;+\; \underbrace{C(\x \mid \z_t)}_{\text{increasing in $t$ (conditional total correlation, residual)}},
    \end{equation}
    \vspace{-1pt}where $H$ is entropy, $I$ is mutual information, and $C(\x\mid\z_t) := \sum_l H(\x^l\mid\z_t) - H(\x\mid\z_t) \in [0,C(\x)]$ is the conditional total correlation that increases monotonically in $t$. %
\end{restatable}

See~\cref{supp:sec:linear-ce} for the proof. Despite the residual, $\CE_{\theta, \gamma}(t)$ with learned $\theta,\gamma$ is empirically observed to be near-linear in the middle of training (\cref{fig:ce-loss-vs-t}). To ensure generality, we validate this for an extensive set of frozen or learnable embeddings of different $d_e$ (\cref{sec:linear-ce-diverse}). Therefore, learning the noise schedule almost evenly distributes denoising difficulty, regardless of the embeddings used.

\section{Conclusion}
\label{sec:conclusion}

We present the first unified scaling law comparison between continuous and discrete DLMs, showing that continuous models scale on par with discrete counterparts when properly aligned in architecture and experimental setting. We introduce RePlaid, a modernized likelihood-based continuous DLM achieving a state-of-the-art $\replaidscppl$ perplexity among continuous DLMs on OpenWebText and even rivaling discrete DLMs in terms of both likelihood and sample quality. We show that the ELBO objective: $(i)$ naturally recovers near-linear cross-entropy over time, eliminating heuristic time reparameterizations, and $(ii)$ inherently regularizes embedding geometry for likelihood improvement and prevents potential token dispersion in cross-entropy-based training. See~\cref{app:discussion} for a discussion of limitations and future directions.

\section*{Acknowledgments and Disclosure of Funding}

This work is partially supported by AI-MI (ZY \& JT), NSF Award DMR-2433348 (ZY \& JT), NSF Grants ECCS-1942523, DMS-2206576, 2450378 (WG \& YC), AFOSR Grant FA9550-25-1-0169 (WG \& YC) and Georgia Tech ARC-ACO Fellowship (WG).

\bibliographystyle{plainnat}
\bibliography{main}

\clearpage

\appendix
\crefalias{section}{appendix}
\crefalias{subsection}{appendix}
\crefalias{subsubsection}{appendix}

\etocdepthtag.toc{appendix}  %
\etocsettagdepth{main}{none}        %
\etocsettagdepth{appendix}{section} %

{\small
\tableofcontents}

\clearpage

\section{Discussion}
\label{app:discussion}

[Back to~\cref{sec:conclusion} in the main paper]

While our scaling law study provides a rigorous comparison following prior work on scaling discrete DLMs~\citep{nie2025scaling, sahoo2026scaling}, we acknowledge that a $\replaidscx\times$ compute gap remains between RePlaid and AR models to achieve matching likelihood. Investigating whether these scaling trends persist in the multi-billion-parameter regime where modern LLMs operate is a critical next step. Furthermore, our analysis identifies learnable token embeddings as the primary driver of RePlaid's perplexity gains; extending this framework with richer contextual encoders~\citep{meshchaninov2025cosmos} that condition on the full sequence represents a promising path forward. Finally, the continuous nature of RePlaid opens unique avenues for efficiency research unavailable to discrete diffusion, most notably the distillation of sampling trajectories via the PFODE and the exploration of latent-space inference-time scaling.

\clearpage

\section{Related Works}\label{app:related-works}

This appendix surveys diffusion-based approaches to language generation. We do not attempt to review the dominant autoregressive paradigm~\citep{vaswani2017attention,radford2019language}, which serves as the reference baseline in our scaling-law comparisons (\cref{sec:scaling}). Diffusion-based language models split into two families: \emph{discrete} diffusion, which defines a Markov forward process directly on tokens, and \emph{continuous} diffusion, which lifts tokens into a continuous space and runs a Gaussian (or Riemannian) diffusion there.

\subsection{Discrete diffusion language models}

Discrete diffusion was introduced by D3PM~\citep{austin2021structured}, which defined absorbing-state and uniform-noise transition kernels over token sequences and was lifted to continuous-time Markov chains shortly after~\citep{campbell2022continuous}. Subsequent work refined the parameterization: SEDD~\citep{lou2024discrete} and concrete score matching~\citep{meng2022concrete} replaced the categorical denoising objective with a score-matching surrogate, while MDLM~\citep{sahoo2024simple}, MD4~\citep{shi2024simplified} and RADD~\citep{ou2025your} concurrently and independently identified that the absorbing-state model collapses to a clean per-position cross-entropy objective once the mask token is treated correctly. Block Diffusion~\citep{arriola2025block} and Eso-LMs~\citep{sahoo2025esoteric} interpolate between masked diffusion and autoregressive decoding to obtain unified families. In parallel, discrete flow matching~\citep{gat2024discrete} and generative flows on discrete state spaces~\citep{campbell2024generative} cast the discrete forward process in a flow-matching framework, and LLaDA~\citep{nie2025large} and Dream~\citep{ye2025dream} scale masked diffusion to 7--8B parameters. Most relevant to our methodology, \citet{nie2025scaling} and \citet{sahoo2026scaling} provide the IsoFLOP scaling-law protocol that we adopt to benchmark continuous diffusion against its discrete counterparts on equal footing.

\subsection{Continuous diffusion language models}

Continuous diffusion language modeling lifts a Gaussian forward process onto a continuous representation of tokens, and falls roughly into three camps. The first and most studied embeds tokens into $\R^d$ and diffuses the embeddings: Diffusion-LM~\citep{li2022diffusion} and SSD-LM~\citep{han2022ssd} introduced the recipe with controllable generation in mind; CDCD~\citep{dieleman2022continuous} added a learned categorical-conditional schedule; DiffusionBERT~\citep{he2022diffusionbert} and self-conditioned embedding diffusion~\citep{strudel2022self} explored hybrid masked/continuous training; latent diffusion language models~\citep{lovelace2023latent} push the diffusion into a frozen pretrained latent space; and analog bits~\citep{chen2023analog} encode tokens as binary vectors and diffuse them in $\R$. Plaid~\citep{gulrajani2023plaid}, on which RePlaid is directly based, combined low-dimensional sphere-normalized embeddings, a softmax-categorical denoiser and a learned monotone noise schedule.

Concurrent work continues this thread: FLM~\citep{lee2026flow} recasts the embedding-space diffusion as a flow-map model targeting one-step continuous denoising, and Discrete Flow Maps (DFM) ~\citep{potaptchik2026discrete} and Categorical Flow Maps (CFM) ~\citep{roos2026categorical} are concurrent works that extend the flow-map machinery further. Both FLM and DFM use fixed one-hot embeddings and cross-entropy loss.  LangFlow~\citep{chen2026langflow} swaps the three-term NELBO in Plaid for a pure cross-entropy objective justified via a Bregman flow-matching identity, and uses learnable embeddings as in Plaid. A second camp parameterizes the diffusion on the simplex or its statistical manifold rather than in embedding space: categorical flow matching on statistical manifolds~\citep{cheng2024categorical}, Fisher flow matching~\citep{davis2024fisher} and RDLM~\citep{jo2025continuous} construct flows respecting the Fisher--Rao geometry of categorical distributions. A third dequantization-based camp lifts categorical samples to continuous variables via argmax flows~\citep{hoogeboom2021argmax} or Dirichlet noise~\citep{avdeyev2023dirichlet}; Duo~\citep{sahoo2025the} also belongs here, observing that uniform-state discrete diffusion arises as the argmax projection of an underlying Gaussian diffusion and exploiting this duality to import continuous-diffusion training tricks back into the discrete model. RePlaid revisits the embedding-space camp, specifically Plaid, under the IsoFLOP scaling protocol of \citet{nie2025scaling,sahoo2026scaling}, isolating what fraction of Plaid's reported compute gap stems from the method itself versus from the evaluation setup.

\clearpage

\section{Derivation of the sequence-level NELBO for Plaid}
\label{supp:nelbo-derivation}

[Back to~\cref{sec:background} in the main paper.]

Let $\x\in [V]^L$ be a length-$L$ sequence over a vocabulary of size $V$, identified with a matrix in $\{0, 1\}^{L \times V}$ whose $l$-th row $\x^l$ is the one-hot vector of the $l$-th token, and $\e:=\x\Emb\in\R^{L\times d_e}$ its embedding (\cref{sec:background}). Following VDMs~\citep{kingma2021variational}, the noisy latents $\z_t\in\R^{L\times d_e}$, $t\in[0,1]$, have conditional marginal
\begin{align*}
    q(\z_t \mid \x) = \normal(\alpha_t \e,\ \sigma_t^2 \I),
\end{align*}
with $\alpha_t,\sigma_t>0$ smooth and $\SNR(t):=\alpha_t^2/\sigma_t^2$ strictly decreasing in $t$.

For a discretization with $T$ steps, let $s(i):=(i-1)/T$ and $t(i):=i/T$, so the joint forward distribution can be written as $q(\z_{0:1}\mid\x):=q(\z_1\mid\x)\prod_{i=1}^T q(\z_{s(i)}\mid\z_{t(i)},\x)$, with each transition $q(\z_{s(i)}\mid\z_{t(i)},\x)$ the closed-form Gaussian posterior of VDMs~\cref{eq:vdm-reverse}. Plaid parameterizes the reverse process by plugging the denoiser's clean-data prediction $\x_\theta(\z_{t(i)}, t(i))$ into that posterior in place of $\x$:
\begin{align}\label{eq:plaid-parameterization}
    p_\theta(\z_{s(i)}\mid\z_{t(i)}) := q\bigl(\z_{s(i)}\,\big|\,\z_{t(i)},\ \x = \x_\theta(\z_{t(i)}, t(i))\bigr),
\end{align}
where $\x_\theta:\R^{L\times d_e}\times[0,1]\to(\Delta^V)^L$ takes the entire noisy sequence as input.

Multiplying and dividing the integrand by $q(\z_{0:1}\mid\x)=q(\z_1\mid\x)\,\prod_{i=1}^T q(\z_{s(i)}\mid\z_{t(i)},\x)$ and applying Jensen's inequality to the (concave) logarithm gives an upper bound on the negative log-likelihood:
\begin{align*}
    -\log p_\theta(\x)
    &= -\log \int p(\z_1)\, p_\theta(\x\mid\z_0)\, \prod_{i=1}^T p_\theta(\z_{s(i)}\mid\z_{t(i)})\, \d\z_{0:1} \\
    &\le -\E_{q(\z_{0:1}\mid\x)}\!\left[\log
        \frac{p(\z_1)\, p_\theta(\x\mid\z_0)\, \prod_{i=1}^T p_\theta(\z_{s(i)}\mid\z_{t(i)})}
             {q(\z_1\mid\x)\,\prod_{i=1}^T q(\z_{s(i)}\mid\z_{t(i)},\x)}\right] \\
    &= \underbrace{\E_{q(\z_0\mid\x)}\bigl[-\log p_\theta(\x\mid\z_0)\bigr]}_{\text{Reconstruction loss}}
       + \underbrace{\KL\bigl(q(\z_1\mid\x)\,\|\,p(\z_1)\bigr)}_{\text{Prior loss}} \\
    &\quad + \underbrace{\sum_{i=1}^T \E_{q(\z_{t(i)}\mid\x)}\KL\bigl(q(\z_{s(i)}\mid\z_{t(i)},\x)\,\|\,p_\theta(\z_{s(i)}\mid\z_{t(i)})\bigr)}_{\text{Diffusion loss }\cL_T(\x)}.
\end{align*}

\paragraph{Reconstruction loss.}
Plaid's denoiser emits an independent categorical distribution at every position, $p_\theta(\x\mid\z_0)=\prod_{l=1}^L \langle \x_\theta^l(\z_0, 0),\,\x^l\rangle$, so
\begin{align}\label{eq:vdm-reconst-loss}
    \E_{q(\z_0\mid\x)}\bigl[-\log p_\theta(\x\mid\z_0)\bigr]
    = \sum_{l=1}^L \E_{q(\z_0\mid\x)}\bigl[-\log\langle \x_\theta^l(\z_0,0),\,\x^l\rangle\bigr].
\end{align}

\paragraph{Prior loss.}
Both $q(\z_1\mid\x)=\normal(\alpha_1\e,\sigma_1^2\I)$ and $p(\z_1)=\normal(\bm{0},\I)$ are isotropic Gaussians on $\R^{L\times d_e}$, giving the closed form
\begin{align*}
    \KL\bigl(q(\z_1\mid\x)\,\|\,p(\z_1)\bigr)
    = \frac{1}{2}\Bigl[\, L\,d_e\bigl(\sigma_1^2 - 1 - \log\sigma_1^2\bigr) + \alpha_1^2\,\|\e\|_\mathrm{F}^2 \,\Bigr],
\end{align*}
where $\|\cdot\|_\mathrm{F}$ is the Frobenius norm.

\paragraph{Diffusion loss.}
Substituting Plaid's parameterization~\cref{eq:plaid-parameterization} into each per-step KL collapses it to a scaled squared distance between the predicted and ground-truth embeddings; \citep[Eq. (13)]{kingma2021variational} shows that summing over $i$ yields
\begin{align}\label{eq:supp-LT}
    \cL_T(\x) = \frac{T}{2}\,\E_{i\sim\unif\{1,\dots,T\},\,\z_{t(i)}\sim q(\z_{t(i)}\mid\x)}\!\Bigl[\bigl(\SNR(s(i))-\SNR(t(i))\bigr)\,\bigl\|\,\e_\theta(\z_{t(i)}, t(i)) - \e\,\bigr\|^2\Bigr],
\end{align}
where $\e_\theta(\z_t, t):=\x_\theta(\z_t, t)\,\Emb \in \R^{L\times d_e}$ is the predicted clean-embedding sequence. Taking $T\to\infty$ converts the SNR difference into a derivative,
\begin{align}\label{eq:supp-Linf}
    \cL_\infty(\x) = -\frac{1}{2}\,\E_{t\sim\unif[0,1],\,\z_t\sim q(\z_t\mid\x)}\!\Bigl[\SNR'(t)\,\bigl\|\,\e_\theta(\z_t, t) - \e\,\bigr\|^2\Bigr],
\end{align}
which is non-negative because $\SNR'(t) < 0$.

Combining the three terms recovers the sequence-level NELBO of \cref{eq:plaid-nelbo}:
\begin{align*}
    \cL_\NELBO(\x)
    = \KL\bigl(q(\z_1\mid\x)\,\|\,p(\z_1)\bigr)
    + \E_{q(\z_0\mid\x)}\!\left[\sum_{l=1}^L -\log\langle\x_\theta^l(\z_0, 0),\,\x^l\rangle\right]
    + \cL_\infty(\x).
\end{align*}

\clearpage

\section{Training Algorithm}
\label{supp:sec:training-algo}

[Back to~\cref{sec:background} in the main paper.]

\cref{alg:training-algo} gives the per-step training procedure. Each minibatch contributes to all three terms of the NELBO~\cref{eq:plaid-nelbo}. The batch is split into a reconstruction sub-batch (size $B_r$, time $t = 0$ exactly) and a diffusion sub-batch (size $B - B_r$, low-discrepancy $t \in [0, 1]$); the prior loss uses every example. Following Plaid~\citep{gulrajani2023plaid}, $B_r$ is adapted online to track the running ratio of the two losses' MC standard deviations. Self-conditioning is enabled per-example at a fixed rate $p_\text{sc} = 0.25$: an initial gradient-free forward pass produces a predicted clean embedding $\hat\e_\text{sc}$ that is fed back as an additional input. The schedule endpoints $(\gamma_0, \gamma_1)$ minimize the diffusion loss directly, while the interior shape $\tilde{\gamma}(\cdot)$ is updated to minimize the loss's MC variance through a per-parameter gradient hook (\cref{fig:diffusion-loss-code}); only the diffusion rows feed this hook---reconstruction and self-conditioning rows are detached from $\gamma(\cdot)$. 


\begin{algorithm}[H]
\caption{One Plaid training step on one GPU rank.}
\label{alg:training-algo}
\small
\begin{algorithmic}[1]
\Require per-rank minibatch $\mathbf{X}\in [V]^{B\times L}$; denoiser $\x_\theta$ returning a per-position categorical in $(\Delta^V)^{B\times L}$; embedding $\Emb\in\R^{V\times d_e}$ (rows on the unit sphere); schedule $(\gamma_0, \gamma_1, \gamma(\cdot))$; running EMAs $(m_r, q_r, m_d, q_d)$ of the first two moments of the recon/diff losses, with decay $\rho = 0.997$ and bias correction $\beta$, all initialized to $10^{-8} \approx 0$; self-cond rate $p_\text{sc} = 0.25$; optimizer.
\State $\e \gets \mathbf{X}\Emb$
\State $(m_r, q_r, m_d, q_d) \gets \frac{1}{R}\,\textsc{AllReduceSum}(m_r, q_r, m_d, q_d)$ \Comment{reduce across GPU ranks}
\State $\hat\sigma_\star \gets \bigl[\,q_\star/\beta - (m_\star/\beta)^2\,\bigr]_+^{1/2}$ for $\star \in \{r, d\}$ \Comment{bias-corrected global-batch std}
\State $B_r \gets \mathrm{clip}(\mathrm{round}(B \cdot \hat\sigma_r / (\hat\sigma_r + \hat\sigma_d)), 1, B-1)$ \Comment{adaptive batch split, same on all ranks}
\State sample self-cond mask $\mathbf{m}\in\{0,1\}^B$ with exactly $\lceil p_\text{sc} B\rceil$ ones
\Statex
\State \textbf{Timesteps and noise schedule.}
\State $t_b \gets 0$ for $b \le B_r$;\quad $t_b$ from low-discrepancy $\unif[0,1]$ for $b > B_r$
\State $(\tilde\gamma_t, \dot{\tilde\gamma}_t) \gets (\tilde{\gamma}(t), \tilde{\gamma}'(t))$, detached on the recon rows ($b \le B_r$)
\State on the diffusion rows ($b > B_r$), register a hook scaling the gradient of $(\tilde\gamma_t, \dot{\tilde\gamma}_t)$ by $2 \cL_\text{diff}^{(b)}$ at backward (\cref{fig:diffusion-loss-code})
\State $\gamma_t \gets \gamma_0 + (\gamma_1 - \gamma_0) \tilde\gamma_t$;\quad $\dot\gamma_t \gets (\gamma_1 - \gamma_0) \dot{\tilde\gamma}_t$
\State on self-cond rows ($\mathbf{m}_b = 1$), detach $\gamma_t$, $\dot\gamma_t$, and $\e^{(b)}$ (self-cond loss only affect the denoiser)
\Statex
\State \textbf{Forward pass with self-conditioning.}
\State $\alpha_t \gets \sqrt{\sigmoid(-\gamma_t)}$;\quad $\sigma_t \gets \sqrt{\sigmoid(\gamma_t)}$
\State sample $\beps \sim \normal(\zero, \I)$;\quad $\z_t \gets \alpha_t \e + \sigma_t \beps$
\State $\hat\e_\text{sc} \gets \mathbf{0}$
\If{$\sum_b \mathbf{m}_b > 0$}
    \State for each $b$ with $\mathbf{m}_b = 1$, in \texttt{no\_grad} compute $\hat\x^\text{sc} \gets \x_\theta(\z_t^{(b)}, \gamma_t^{(b)}, \mathbf{0})$ and set $\hat\e_\text{sc}^{(b)} \gets \hat\x^\text{sc} \Emb$
\EndIf
\State $\hat\x \gets \x_\theta(\z_t, \gamma_t, \hat\e_\text{sc})$;\quad $\hat\e \gets \hat\x \Emb$
\Statex
\State \textbf{Three NELBO terms (per-token).}
\State $\cL_\text{recon}^{(b)} \gets -\frac{1}{L} \sum_l \log \langle \hat\x^{(b,l)}, \mathbf{X}^{(b,l)} \rangle$ for $b \le B_r$
\State $\SNR'(t) \gets -\dot\gamma_t \ee^{-\gamma_t}$;\quad $\cL_\text{diff}^{(b)} \gets -\frac{1}{2} \SNR'(t_b) \cdot \frac{1}{L} \sum_l \|\hat\e^{(b,l)} - \e^{(b,l)}\|^2_2$ for $b > B_r$
\State $\alpha_1^2 \gets \sigmoid(-\gamma_1)$;\quad $\sigma_1^2 \gets \sigmoid(\gamma_1)$
\State $\cL_\text{prior} \gets \frac{1}{2BL} \sum_{b,l} \bigl[\alpha_1^2 \|\e^{(b,l)}\|^2_2 + d_e (\sigma_1^2 - 1 - \log \sigma_1^2)\bigr]$
\Statex
\State $\cL \gets \frac{1}{B_r} \sum_{b \le B_r} \cL_\text{recon}^{(b)} + \frac{1}{B - B_r} \sum_{b > B_r} \cL_\text{diff}^{(b)} + \cL_\text{prior}$
\Statex
\State \textbf{Per-rank moment EMAs} \Comment{reduced at the start of the next step}
\State $\beta \gets \rho\,\beta + (1 - \rho)$
\State $m_r \gets \rho\, m_r + \frac{1-\rho}{B_r} \sum_{b \le B_r} \cL_\text{recon}^{(b)}$;\quad $q_r \gets \rho\, q_r + \frac{1-\rho}{B_r} \sum_{b \le B_r} \bigl(\cL_\text{recon}^{(b)}\bigr)^2$
\State $m_d \gets \rho\, m_d + \frac{1-\rho}{B - B_r} \sum_{b > B_r} \cL_\text{diff}^{(b)}$;\quad $q_d \gets \rho\, q_d + \frac{1-\rho}{B - B_r} \sum_{b > B_r} \bigl(\cL_\text{diff}^{(b)}\bigr)^2$
\Statex
\State backpropagate $\cL$; \texttt{optimizer.step()}.
\end{algorithmic}
\end{algorithm}


\subsection{Self-Conditioning: Formal Description}
\label{app:selfcond-formal}

Self-conditioning (s.c.)~\citep{chen2023analog} lets the model refine its current output by conditioning on a previous output of its own. We state the training- and sampling-time procedures explicitly, writing $\x_\theta(\z, t; \e_\sc)$ for the denoiser taking an auxiliary self-conditioning input $\e_\sc \in \R^{L \times d_e}$.

\paragraph{Training.} For a batch $\{\x^{(b)}\}$ with noised versions $\{\z_t^{(b)}\}$, draw an indicator $m^{(b)} \sim \mathrm{Bernoulli}(p_\sc)$ with $p_\sc = 0.25$ per example (in practice, we pre-select exactly 25\% sequences with $m^{(b)}=1$):
\begin{itemize}
    \item If $m^{(b)} = 1$ (s.c. on), run a first forward pass \emph{without gradient} to obtain $\x_\theta(\z_t^{(b)}, t; \zero) \in \R^{L\times V}$ and set $\e_\sc^{(b)} := \x_\theta(\z_t^{(b)}, t; \zero)\,\Emb \in \R^{L\times d_e}$.
    \item If $m^{(b)} = 0$ (s.c. off), set $\e_\sc^{(b)} := \zero$.
\end{itemize}
The second forward pass computes $\x_\theta(\z_t^{(b)}, t; \e_\sc^{(b)})$, which is the output entering the loss. Only this second pass carries gradients; see \cref{alg:training-algo}. We account for the extra FLOPs of the first forward pass in the scaling-law study of~\cref{sec:scaling}.

\paragraph{Sampling.} Suppose we query the model at time $s$, with the previous (noisier) sampling time $t > s$, and let $\hat\x^{(t)}_\theta$ denote the model output already computed at $(\z_t, t)$. We set $\e_\sc := \hat\x^{(t)}_\theta \Emb$ and query $\x_\theta(\z_s, s; \e_\sc)$. Since $\hat\x^{(t)}_\theta$ was computed by the previous sampling step regardless, this introduces \emph{no additional NFE} relative to non-self-conditioned sampling, which queries $\x_\theta(\z_s, s; \zero)$.

\paragraph{Validation.} At validation time we follow the training procedure but enable s.c. for $100\%$ of examples, matching the sampling-time condition except that both forward passes use the same time.

\paragraph{Does the loss remain a likelihood bound?} Yes. Self-conditioning only changes the \textit{parameterization} of the denoiser for a given query $(\z_t, t)$: it supplies an extra, deterministic function of quantities the denoiser already has access to. It alters neither the forward process $q$ nor the variational decomposition that produces \cref{eq:plaid-nelbo}, so $\cL_\NELBO(\x)$ remains a valid upper bound on $-\log p_\theta(\x)$ with s.c. enabled. (This is a statement about the \emph{NELBO}. The separate question of how s.c. interacts with the \emph{PFODE-based} likelihood estimator is treated in \cref{subsec:ode-likelihood-selfcond}.)

\subsection{Wall-Clock Overhead of Self-Conditioning}
\label{app:selfcond-walltime}

Self-conditioning costs extra compute during \textit{training} but, per the argument above, essentially nothing during \textit{sampling}. We measure both.

\paragraph{Training.} Self-conditioning adds a gradient-free forward pass for $25\%$ of examples per gradient step, raising non-embedding FLOPs by $\approx 8.3\%$ (\cref{subsec:scaling-law-setup}) and wall-clock time by $1.16\times$:

\begin{center}
\small
\begin{tabular}{lc}
\toprule
Self-conditioning & Mean time per gradient step (s) \\
\midrule
Off & $0.251$ \\
On & $0.292$ ($1.16\times$ slower) \\
\bottomrule
\end{tabular}
\end{center}
\noindent Setting: 0.1B models trained on OWT, $16\times$ GB200, global batch $512$, sequence length $1024$. For the scaling experiments of~\cref{sec:scaling} we explicitly offset these extra FLOPs by reducing the number of training steps, so the IsoFLOP comparison remains compute-matched. We note that the 0.1B OWT benchmark of~\cref{sec:small-models} controls only for backbone parameter count, and that baselines carry their own overheads under that convention. For instance, Duo trains $\approx1.5\times$ slower during its curriculum phase (the first 500K steps) because it replaces the cheap embedding lookup with a dense mixture over the full vocabulary~\citep[Tab.~4]{sahoo2025the}.

\paragraph{Sampling.} Consistent with the no-extra-NFE argument of \cref{app:selfcond-formal}, we measure no significant wall-clock cost:

\begin{center}
\small
\begin{tabular}{lc}
\toprule
Self-conditioning & Mean duration $\pm$ standard error (s) \\
\midrule
Off & $7.792 \pm 0.015$ \\
On & $7.821 \pm 0.021$ \\
\bottomrule
\end{tabular}
\end{center}
\noindent Setting: 0.1B models trained on OWT, $1\times$ GB200; each run generates a batch of $64$ sequences using $64$ sampling steps; mean and standard error over $50$ runs.


\clearpage

\section{Sampler Update Formulas}
\label{app:sampler}

[Back to~\cref{subsec:sampling} in the main paper.]

This appendix gives the per-step updates for the four samplers used in the main paper: one stochastic ancestral sampler (DDPM-style) and three deterministic PFODE solvers (DDIM, DPM-Solver++(2M), Heun). All operate on the variance-preserving schedule $\sigma_t^2 = \sigmoid(\gamma_t)$, $\alpha_t^2 = \sigmoid(-\gamma_t)$, and step from $t$ to $s < t$ along a discretization of $[1, 0]$ into $T$ steps. At each step the network is evaluated to obtain a clean-data prediction; we abbreviate
\begin{equation*}
  \hat\e_\theta := \x_\theta(\z_t, t)\Emb, \qquad \hat\beps_\theta := (\z_t - \alpha_t \hat\e_\theta)/\sigma_t,
\end{equation*}
for the predicted clean embedding and the implied noise prediction. The PFODE velocity \cref{eq:plaid-pfode} reads $\v_\theta(\z_t, t) = \dot\alpha_t \hat\e_\theta + \dot\sigma_t \hat\beps_\theta$ in this notation. After $T$ updates the final tokens are read out as $\argmax$ of the decoder logits at $t \approx 0$.

\subsection{Stochastic Ancestral Sampler (DDPM)}
\label{app:sampler-ddpm}

This is the default sampler used to report all generation results unless otherwise noted. It draws $\z_s \sim p_\theta(\z_s \mid \z_t)$ for $s < t$ from the closed-form Gaussian posterior of VDM~\citep[App.~A.4]{kingma2021variational}:
\begin{equation}\label{eq:plaid-ancestral}
\z_s = (1 - c) \frac{\alpha_s}{\alpha_t} \z_t + c \alpha_s \hat\e_\theta + \sqrt{c (1 - \alpha_s^2)}\,\beps, \quad \beps \sim \normal(\zero, \I), \quad c = 1 - \ee^{\gamma_s - \gamma_t},
\end{equation}
where $c \in (0, 1)$ is the per-step SNR gap, since $\SNR(t) = \exp(-\gamma_t)$ implies $c = 1 - \SNR(t)/\SNR(s)$. Each step uses one network evaluation, giving $T + 1$ total NFEs including the final $\argmax$.

\subsection{Deterministic PFODE Solvers}
\label{app:sampler-pfode}

The three deterministic samplers below all integrate the PFODE \cref{eq:plaid-pfode}, differing only in the discretization scheme.

\paragraph{DDIM~\citep{song2021denoising}.} Reusing $\hat\beps_\theta$ at both endpoints when stepping to $s < t$ gives the update
\begin{equation}\label{eq:plaid-ddim}
  \z_s = \alpha_s \hat\e_\theta + \sigma_s \hat\beps_\theta.
\end{equation}
This is the exponential-integrator step for \cref{eq:plaid-pfode} \emph{linearized at $\hat\e_\theta$} (i.e., with $\hat\e_\theta$ held constant along the trajectory), and is therefore exact for any step size whenever the linearization is exact. A $T$-step schedule uses $T + 1$ NFEs including the final $\argmax$.

\paragraph{DPM-Solver++(2M)~\citep{lu2025dpm}.} A second-order multistep solver. Let $\lambda_t := -\gamma_t/2 = \log(\alpha_t/\sigma_t)$ denote the negative half log-SNR; stepping from $t_{i-1}$ to $t_i < t_{i-1}$ gives step size $h_i = \lambda_{t_i} - \lambda_{t_{i-1}} > 0$ in $\lambda$-space. With $\hat\x_\theta^{(i)} := \x_\theta(\z_{t_i}, t_i)$, the update reads
\begin{align}
\D_i &= \left(1 + \frac{1}{2 r_i}\right) \hat\x_\theta^{(i-1)} \Emb - \frac{1}{2 r_i} \hat\x_\theta^{(i-2)} \Emb, \quad r_i = \frac{h_{i-1}}{h_i} \quad (i > 1); \quad \D_1 = \hat\x_\theta^{(0)} \Emb, \label{eq:dpmpp-extrap}\\
\z_{t_i} &= \frac{\sigma_{t_i}}{\sigma_{t_{i-1}}} \z_{t_{i-1}} - \alpha_{t_i} (\ee^{-h_i} - 1) \D_i. \label{eq:dpmpp-update}
\end{align}
The multistep term \cref{eq:dpmpp-extrap} linearly extrapolates $\hat\x_\theta \Emb$ along $\lambda$ using the current prediction and the one cached from the previous step; at the first step, where the cached prediction is unavailable, we fall back to the DDIM update. A $T$-step schedule uses $T + 1$ NFEs including the final $\argmax$.

\paragraph{Heun~\citep{karras2022elucidating}.} A second-order predictor--corrector solver. We reparametrize \cref{eq:plaid-pfode} into the variance-exploding (VE) form via $\bar\z_t := \z_t/\alpha_t$ and $\tilde\sigma_t := \sigma_t/\alpha_t = \exp(\gamma_t/2)$, under which $\bar\z_t \mid \e \sim \normal(\e, \tilde\sigma_t^2 \I)$ and the ODE collapses to the canonical
\begin{equation}
  \frac{\d\bar\z}{\d\tilde\sigma} = \frac{\bar\z-\hat\e_\theta}{\tilde\sigma} =: \hat\beps(\bar\z,\tilde\sigma).
\end{equation}
The update from $\tilde\sigma_t$ to $\tilde\sigma_s$ is
\begin{align}
  \bar\z'_s &= \bar\z_t + (\tilde\sigma_s-\tilde\sigma_t)\,\hat\beps(\bar\z_t,\tilde\sigma_t),
  \label{eq:heun-ve-pred}\\
  \bar\z_s &= \bar\z_t + \frac{1}{2}(\tilde\sigma_s-\tilde\sigma_t)\!\left[\hat\beps(\bar\z_t,\tilde\sigma_t)+\hat\beps(\bar\z'_s,\tilde\sigma_s)\right],
  \label{eq:heun-ve-corr}
\end{align}
with the VP latent recovered as $\z_s = \alpha_s\bar\z_s$. Each step uses two network evaluations; following EDM we skip the corrector on the final step, giving $2T$ total NFEs (including the final $\argmax$) for $T$ sampling steps.

\clearpage

\section{ODE-based Likelihood Estimation}
\label{app:ode-likelihood}

This appendix presents an alternative ODE-based estimator of $\log p_\theta(\x)$ that avoids the multi-step variational bound of \cref{eq:plaid-nelbo}. We use it as an internal sanity check on the VDM NELBO numbers reported in \cref{subsec:likelihood-eval}; it is not used to report any headline numbers in the main paper.

\subsection{ODE-based likelihood identity}
\label{subsec:ode-likelihood}
\begin{lemma}[\textbf{\color{ourbrightblue!80!black}ODE-based likelihood identity}]
\label{lem:ode-likelihood}
Let $p_\theta(\x \mid \z_0) := \prod_{l=1}^L \langle \x_\theta^l(\z_0, 0), \x^l \rangle$ be the categorical decoder at $t = 0$, let $q(\z_0 \mid \x) := \normal(\alpha_0 \e, \sigma_0^2 \I)$ be the forward Gaussian proposal, and let $(\z_t)_{t \in [0, 1]}$ denote the PFODE trajectory \cref{eq:plaid-pfode} simulated from $\z_{t=0} = \z_0$. Then
\begin{equation}\label{eq:ode-iw-identity}
  p_\theta(\x) = \E_{q(\z_0 \mid \x)}\!\left[\frac{p_\theta(\z_0)\, p_\theta(\x \mid \z_0)}{q(\z_0 \mid \x)}\right],
\end{equation}
\begin{align}\label{eq:ode-cov-exact}
  \log p_\theta(\z_0) &= \log p(\z_1) + \int_0^1 \nabla \cdot \v_\theta(\z_t, t)\d t \notag\\
  &= -\frac{L d_e}{2} \log(2\pi) - \frac{1}{2}\|\z_1\|^2 + \int_0^1 \nabla \cdot \v_\theta(\z_t, t)\d t,
\end{align}
and the following \textbf{importance-weighted autoencoder (IWAE)} \citep{burda2015importance} estimator
\begin{equation}\label{eq:ode-iwae}
  \widehat{\log p_\theta(\x)}_K := \log \frac{1}{K} \sum_{k=1}^K w^{(k)}, \quad w^{(k)} := \frac{p_\theta(\z_0^{(k)})\, p_\theta(\x \mid \z_0^{(k)})}{q(\z_0^{(k)} \mid \x)},
  \quad\text{where}~\z_0^{(k)} \iidsim q(\z_0 \mid \x),
\end{equation}
satisfies $\widehat{\log p_\theta(\x)}_K \to \log p_\theta(\x)$ almost surely as $K \to \infty$, and the per-example likelihood estimation $\E\bigl[\widehat{\log p_\theta(\x)}_K\bigr]$ is monotonically non-decreasing in $K$ with bias $\Theta(1/K)$~\citep{nowozin2018debiasing}.
\end{lemma}

\begin{remark}
The divergence $\nabla\cdot \v_\theta(\z_t, t) = \operatorname{tr}(\partial \v_\theta/\partial \z)$ is estimated unbiasedly by \textbf{Hutchinson's trace estimator}~\citep{hutchinson1989stochastic}
\begin{align*}
\operatorname{tr}\Bigl(\frac{\partial \v_\theta}{\partial \z}\Bigr) = \E_\bxi\Bigl[\bxi^\top \Bigl(\frac{\partial \v_\theta}{\partial \z}\Bigr) \bxi\Bigr]
\quad \text{with} \quad
\E[\bxi\bxi^\top] = \I,
\end{align*}
reducing $L d_e$ partial derivatives to a single Jacobian-vector product per draw of $\bxi$. We use $\bxi$ following Rademacher (i.i.d.\ $\unif\{\pm 1\}$ per coordinate), which is the minimum-variance choice among distributions satisfying $\E[\bxi\bxi^\top] = \I$~\citep{hutchinson1989stochastic}.
\end{remark}

\begin{remark}
    The concurrent ODE ELBO of \citet[Thm.~3.1]{chen2026langflow} is the $K = 1$ specialization of \cref{eq:ode-iwae}.
\end{remark}

\begin{proof}[Proof of \cref{lem:ode-likelihood}] The proof follows these steps:

\textbf{Importance-sampling identity.} The marginalization $p_\theta(\x) = \int p_\theta(\x \mid \z_0)\,p_\theta(\z_0)\d\z_0$ together with importance sampling against $q(\z_0\mid\x)$ yields \cref{eq:ode-iw-identity}.

\textbf{Liouville closure of $\log p_\theta(\z_0)$.} The marginal $p_\theta(\z_0)$ is set by the PFODE $\dot\z_t = \v_\theta(\z_t, t)$ of \cref{eq:plaid-pfode}, which connects $\z_0$ at $t = 0$ to $\z_1$ at $t = 1$. Liouville's theorem applied along the same trajectory yields the first equality of \cref{eq:ode-cov-exact}; substituting the prior $p(\z_1) = \normal(\zero, \I)$ at $t = 1$ gives the second equality. %

\textbf{Explicit per-sample log-weight.} Substituting \cref{eq:ode-cov-exact} together with the Gaussian density of $q(\z_0 \mid \x)$ into the per-sample weight $w^{(k)}$ of \cref{eq:ode-iwae} gives
\begin{align}\label{eq:ode-logweight}
  \log w^{(k)}
  &= \sum_{l=1}^L \log\langle\x_\theta^l(\z_0^{(k)},0), \x^l\rangle + \frac{Ld_e}{2}\log\sigma_0^2 \notag\\
  &\quad + \frac{1}{2\sigma_0^2}\|\z_0^{(k)}-\alpha_0\e\|^2 - \frac{1}{2}\|\z_1^{(k)}\|^2 + \int_0^1 \nabla\cdot \v_\theta(\z_t^{(k)},t)\d t,
\end{align}
where $\{\z_t^{(k)}\}_{t\in[0,1]}$ is the deterministic PFODE trajectory through $\z_0^{(k)}$ at $t=0$ and $\z_1^{(k)}$ at $t=1$. Plugging \cref{eq:ode-logweight} into \cref{eq:ode-iwae} is the estimator we use to report PPL.

\textbf{Consistency and bias.} Since $\frac{1}{K}\sum_k w^{(k)}$ is an unbiased Monte-Carlo estimate of $\E_q[w] = p_\theta(\x)$, the strong law of large numbers and continuity of $\log$ give $\widehat{\log p_\theta(\x)}_K \to \log p_\theta(\x)$ almost surely as $K\to\infty$. By Jensen on the outer Monte-Carlo average, $\E\widehat{\log p_\theta(\x)}_K \le \log p_\theta(\x)$ at every finite $K$, with bias $\Theta(1/K)$ following from the delta-method expansion of \citet{nowozin2018debiasing}.

\end{proof}

\paragraph{Comparison with the VDM NELBO.}The bounds in \cref{lem:ode-likelihood} and \cref{eq:plaid-nelbo} target \textit{different} model densities, so we cannot directly compare their tightness. Write $p_\theta^{\mathrm{ODE}}(\x) := \int p_\theta^{\mathrm{ODE}}(\z_0)\,p_\theta(\x \mid \z_0)\,\d\z_0$ with $p_\theta^{\mathrm{ODE}}(\z_0)$ the PFODE pushforward of $p(\z_1) = \normal(\zero, \I)$ used in \cref{eq:ode-cov-exact}, and $p_\theta^{\mathrm{Markov}}(\x)$ for the density induced by the discrete reverse Markov chain $p_\theta(\z_{s} \mid \z_{t}) := q(\z_{s} \mid \z_{t}, \x_\theta(\z_{t}, t))$ underlying \cref{eq:plaid-nelbo} (the SDE-pushforward density in the $T \to \infty$ limit). The IWAE estimator of \cref{eq:ode-iwae} is an upper bound on $-\log p_\theta^{\mathrm{ODE}}(\x)$, while $\cL_\NELBO(\x)$ is an upper bound on $-\log p_\theta^{\mathrm{Markov}}(\x)$. The two model densities coincide \textit{only at Bayes-optimal score}; off the optimum, the SDE pushforward and the PFODE pushforward of the same denoiser induce different data marginals~(see~\citet[Sec. 4]{song2021maximum} for detailed discussion), so we make no per-$\x$ tightness assertion between $\E[-\widehat{\log p_\theta(\x)}_{K=1}]$ and $\cL_\NELBO(\x)$. Within the PFODE bound, the IWAE family does monotonically tighten in $K$ \cref{eq:ode-iwae}, and is validated in \cref{fig:ode-kcurve-dpi}. The empirical crossover between the $K$-IWAE and $\cL_\NELBO$ reported in \cref{tab:ode-chainrule,fig:ode-kcurve-dpi} therefore reflects the empirical mismatch between the trained model's SDE-pushforward and PFODE-pushforward marginals rather than an ordering between the two bounds due to data-processing inequality.

\begin{lemma}[$\gamma$-space variant of the log-weight]\label{lem:ode-gamma}
Let the schedule be variance-preserving with $\sigma_t^2 = \sigmoid(\gamma_t)$, $\alpha_t^2 = \sigmoid(-\gamma_t)$. Reparametrizing the PFODE \cref{eq:plaid-pfode} by $\gamma$ via $\tilde\z(\gamma) := \z_{t(\gamma)}$ yields
\begin{equation}\label{eq:ode-pfode-gamma}
  \frac{\d\tilde\z}{\d\gamma} = \frac{1}{2}\bigl(\alpha_\gamma^2\, \tilde\z - \alpha_\gamma\, \x_\theta(\tilde\z,\gamma)\Emb\bigr) =: \tilde\v_\theta(\tilde\z,\gamma),
\end{equation}
and the $\gamma$-space log-weight
\begin{align}\label{eq:ode-logweight-gamma}
  \log w^{(k)} &= \sum_{l=1}^L \log\langle\x_\theta^l(\z_0^{(k)},\gamma_0), \x^l\rangle + \frac{Ld_e}{2}\log\sigma_0^2 \notag\\
  &\quad + \frac{1}{2\sigma_0^2}\|\z_0^{(k)}-\alpha_0\e\|^2 - \frac{1}{2}\|\tilde\z(\gamma_1)^{(k)}\|^2 + \int_{\gamma_0}^{\gamma_1} \nabla\cdot \tilde\v_\theta(\tilde\z,\gamma)\d\gamma
\end{align}
equals the $t$-space log-weight \cref{eq:ode-logweight} as a continuous-time identity along the same trajectory.
\end{lemma}

\begin{proof}
The VP schedule gives $2\sigma_t\dot\sigma_t = \sigmoid'(\gamma_t)\,\gamma'(t) = \sigma_t^2\alpha_t^2\,\gamma'(t)$ via $\sigmoid'(z) = \sigmoid(z)(1-\sigmoid(z))$, and $\alpha_t\dot\alpha_t + \sigma_t\dot\sigma_t = 0$ via $\alpha^2+\sigma^2=1$. Hence
\begin{equation*}
  \frac{\dot\sigma_t}{\sigma_t} = \frac{1}{2}\alpha_t^2\,\gamma'(t), \qquad \dot\alpha_t - \alpha_t\frac{\dot\sigma_t}{\sigma_t} = -\frac{1}{2}\alpha_t(\sigma_t^2 + \alpha_t^2)\,\gamma'(t) = -\frac{1}{2}\alpha_t\,\gamma'(t).
\end{equation*}
Substituting into \cref{eq:plaid-pfode} factors out a common $\gamma'(t)$:
\begin{equation*}
  \v_\theta(\z_t,t) = \frac{\dot\sigma_t}{\sigma_t}\z_t + \!\left(\dot\alpha_t - \alpha_t\frac{\dot\sigma_t}{\sigma_t}\right)\!\x_\theta(\z_t,t)\Emb = \frac{\gamma'(t)}{2}\bigl(\alpha_t^2\z_t - \alpha_t\x_\theta(\z_t,t)\Emb\bigr).
\end{equation*}
The chain rule with $\d\gamma = \gamma'(t)\d t$ gives $\d\tilde\z/\d\gamma = \v_\theta(\z_t,t)/\gamma'(t)$, which cancels the $\gamma'(t)$ in the velocity above and yields \cref{eq:ode-pfode-gamma}; the backbone $\x_\theta(\cdot,\gamma)$ already conditions on $\gamma$ directly, so no schedule inversion is needed.

The two velocity fields satisfy $\tilde\v_\theta(\z,\gamma) = \v_\theta(\z,t(\gamma))/\gamma'(t(\gamma))$. Since $\gamma'(t)$ depends on $\gamma$ alone (not on $\z$), the spatial Jacobian inherits the same scalar factor: $\partial_\z\tilde\v_\theta(\z,\gamma) = \partial_\z\v_\theta(\z,t(\gamma))/\gamma'(t(\gamma))$. Taking traces and substituting $\d t = \d\gamma/\gamma'(t)$ in the $t$-space trace integral gives
\begin{equation}\label{eq:ode-trace-invariant}
  \int_0^1 \nabla\cdot\v_\theta(\z_t,t)\d t = \int_{\gamma_0}^{\gamma_1} \nabla\cdot\tilde\v_\theta(\tilde\z,\gamma)\d\gamma,
\end{equation}
by change of variable plus the pointwise trace identity. Substituting into \cref{eq:ode-cov-exact} and \cref{eq:ode-logweight} yields \cref{eq:ode-logweight-gamma}.
\end{proof}

\paragraph{Numerical advantage.} Although \cref{eq:ode-logweight} and \cref{eq:ode-logweight-gamma} are exactly equal in continuous time, they differ as numerical objects: integrating along $\gamma$ removes the schedule-induced stiffness factor $\gamma'(t)$ near $t \to 0, 1$, where $\gamma'(t)$ blows up under typical learned schedules. We use \cref{eq:ode-logweight-gamma} for all reported PPL numbers.

\subsection{Self-conditioning and the Chain-rule Trace}
\label{subsec:ode-likelihood-selfcond}

With self-conditioning enabled, the denoiser carries an extra input $\x_\sc\in\R^{L\times d_e}$, written $\x_\theta(\tilde\z,\gamma,\x_\sc)$. At inference time, both \citet{chen2026langflow} and our pipeline resolve $\x_\sc$ at every solver step from a single bootstrap forward pass with $\x_\sc \gets \zero$:
\begin{equation*}
  \x_\sc(\tilde\z,\gamma) := \x_\theta(\tilde\z,\gamma,\zero)\Emb,
\end{equation*}
making $\x_\sc$ a deterministic function of the current state. The PFODE \cref{eq:ode-pfode-gamma} actually integrated then has effective drift $\tilde\z\mapsto\tilde\v_\theta(\tilde\z,\gamma,\x_\sc(\tilde\z,\gamma))$, and Liouville's identity in \cref{eq:ode-cov-exact} requires the divergence of this composition. By the chain rule,
\begin{equation}\label{eq:ode-selfcond-divergence}
  \nabla_\z\cdot\Bigl[\tilde\v_\theta(\z,\gamma,\x_\sc(\z,\gamma))\Bigr] = \underbrace{\operatorname{tr}\bigl(\partial_\z\tilde\v_\theta\bigr)}_{\text{partial}} + \underbrace{\operatorname{tr}\bigl(\partial_{\x_\sc}\tilde\v_\theta\cdot\partial_\z\x_\sc\bigr)}_{\text{chain rule}},
\end{equation}
with all partials of $\tilde\v_\theta$ evaluated at the unperturbed $(\z,\gamma,\x_\sc(\z,\gamma))$. The standard Hutchinson trace estimator ($\bxi^\top[\tilde\v_\theta(\z+h\bxi,\gamma,\x_\sc^\star) - \tilde\v_\theta(\z-h\bxi,\gamma,\x_\sc^\star)]/(2h)$ in the finite-difference path, or its JVP analogue) caches $\x_\sc^\star := \x_\sc(\z,\gamma)$ at the unperturbed state and reuses it across the two perturbations, recovering only the partial term; the chain-rule term is silently dropped. The integrated divergence is then that of $\tilde\v_\theta(\cdot,\gamma,\x_\sc^\star)$, a different velocity field from the one in \cref{eq:ode-pfode-gamma}, and the resulting ``$-\widehat{\log p_\theta(\x)}$'' is not a Liouville identity for $p_\theta$, hence not an upper bound on $-\log p_\theta(\x)$.

The correction is to recompute $\x_\sc$ at each perturbed state (FD path) or to leave the bootstrap forward inside the JVP graph (JVP path), at a cost of one extra backbone forward per perturbation. With self-conditioning disabled ($\x_\sc \equiv \zero$) the chain-rule term vanishes ($\partial_\z\x_\sc = \zero$) and the standard estimator is unbiased.

\paragraph{Why dropping the chain-rule term inflates the estimate?} The PFODE drift $\tilde\v_\theta$ contains a denoising contribution that contracts the trajectory toward the network's predicted clean embedding $\hat\x_\theta(\tilde\z, \gamma, \x_\sc)\Emb$; with self-conditioning enabled, the bootstrap $\x_\sc(\tilde\z, \gamma)$ co-moves with $\tilde\z$ in the high-SNR regime where the prediction is most informative, so the closed-loop velocity actually integrated by the solver is \emph{less} sensitive to perturbations of $\tilde\z$ than the open-loop velocity $\tilde\v_\theta(\cdot, \gamma, \x_\sc^\star)$ that the cached Hutchinson trace estimator probes. Equivalently, since $\partial_{\x_\sc}\tilde\v_\theta \propto -\alpha_\gamma\,\partial_{\x_\sc}\hat\x_\theta\,\Emb$, the missing chain-rule trace $\operatorname{tr}(\partial_{\x_\sc}\tilde\v_\theta \cdot \partial_\z\x_\sc)$ is typically negative for a self-conditioning denoiser whose two input pathways agree on the same predicted token. Dropping a negative summand overestimates the integrated divergence $\int_0^1 \nabla\cdot\v_\theta\,\d t$, which via the Liouville closure \cref{eq:ode-cov-exact} inflates $\log p_\theta(\z_0)$ and the log-weight $\log w^{(k)}$, deflating the reported PPL. The bias grows with $\partial_\z\x_\sc$ and is therefore largest in the high-SNR regime where the bootstrap is most accurate.

\subsection{Experimental Results and Ablation Study}
\label{subsec:ode-likelihood-ablation}

We evaluate the estimator on RePlaid and on the publicly released LangFlow checkpoint\footnote{\url{https://huggingface.co/Continuous-Rivals-Discrete/langflow-owt}} as an external sanity check, using the chain-rule-corrected drift of \cref{subsec:ode-likelihood-selfcond}. Unless stated otherwise, all numbers are computed on a fixed $1024$-sequence subset of the OpenWebText-valid split (GPT2 tokenizer, sequence length $1024$). The $\gamma$-space PFODE of \cref{lem:ode-gamma} is integrated with a $128$-step Heun2 solver (\texttt{heun2-128}). The divergence term is estimated by Hutchinson's estimator with $n_\mathrm{hutch}=1$ Rademacher random vector per call. The Jacobian-vector product is approximated by central finite differences with step size $h=0.001$ rather than forward-mode automatic differentiation (\texttt{torch.func.jvp}), since the latter currently forces PyTorch's much slower math SDPA backend. The same sequences and self-conditioning protocol are used for both the $K$-IWAE and the matched VDM NELBO baseline.

\paragraph{Chain-rule trace correction.} \cref{tab:ode-chainrule} reports the $K=1$ ELBO in three configurations: (i) self-conditioning disabled, (ii) self-conditioning bootstrap on but the chain-rule term of \cref{eq:ode-selfcond-divergence} dropped (the LangFlow recipe\footnote{\url{https://github.com/nealchen2003/LangFlow/blob/main/eval_ppl.py}}), and (iii) self-conditioning on with the chain-rule correction (ours). Dropping the chain-rule term deflates the reported PPL by $8.03$ on RePlaid and $8.35$ on LangFlow---the missing summand is consistently negative and large at the model scale studied here. The corrected estimate is the only one of the three that is provably an upper bound on $-\log p_\theta(\x)$. Two consistency checks support the implementation: reproducing LangFlow's estimator verbatim yields $24.94$~PPL on their public checkpoint, closely matching the $24.6$~PPL reported by \citet{chen2026langflow} on the full dataset; conversely, the $15.52$~PPL obtained without the chain-rule correction on RePlaid is below the $17.5$~PPL of an autoregressive baseline on the same corpus (\cref{tab:val-ppl}), which would be implausible for an embedding-space diffusion model.

\begin{table}[h]
\centering
\caption{\small Comparison of different PPL estimation methods for \textsc{RePlaid} and LangFlow checkpoints on OpenWebText.
Upper part: Chain-rule omission in ODE-based ELBO estimator deflates the PPL estimation \cref{eq:ode-selfcond-divergence}.
Lower part: Reproduction of the VDM NELBOs from \cref{tab:val-ppl} for reference.
Under the same PPL estimation protocols, \textsc{RePlaid} always outperforms LangFlow.
}
\begin{tabular}{l c c}
\toprule
                                                                          & \textsc{RePlaid} (s.c.) & LangFlow \\
\midrule
\multicolumn{3}{l}{\textit{ODE-based estimator ($K=1$, $1024$ sequences)}} \\
\quad self-conditioning off                                                    & $24.26$ & $38.12$ \\
\quad {\color{gray}self-conditioning\ on, chain-rule term \emph{dropped} (LangFlow)}        & ${\color{gray}15.52}$ & ${\color{gray}24.94}$ \\
\quad self-conditioning\ on, chain-rule term \emph{included} (\textbf{ours})                & $23.54$ & $33.29$ \\
\midrule
\quad $\Delta$ (row 3 $-$ row 2): chain-rule omission bias                     & $+8.03$ & $+8.35$ \\
\midrule
\multicolumn{3}{l}{\textit{Reference: VDM NELBO from \cref{tab:val-ppl} (full dataset)}} \\
\quad self-conditioning off                                                    & $\replaidnscppl$  & $38.4$  \\
\quad self-conditioning on                                                     & $\replaidscppl$  & $32.2$  \\
\bottomrule
\end{tabular}
\vspace{0.5em}
\label{tab:ode-chainrule}
\end{table}

\paragraph{IWAE $K$-curve and DPI ordering.} \cref{fig:ode-kcurve-dpi} plots the $K$-IWAE PPL of \cref{eq:ode-iwae} for $K \in \{1,2,4,8,16,32\}$ on the same $1024$-sequence subset, derived offline by subsampling the per-sequence log-weights of a single $K=32$ run. The curve drops by $1.4$~PPL on RePlaid and $4.1$~PPL on LangFlow between $K=1$ and $K=32$. Fitting the per-token NLL to $a + b/K$ on $K \ge 4$ matches the data cleanly there and overpredicts at $K \le 2$, consistent with the $\Theta(1/K)$ leading-order IWAE bias of \citet{nowozin2018debiasing} together with non-negligible $O(1/K^2)$ corrections at small $K$. The fitted asymptotes are $\text{PPL}_\infty \approx 22.08$ for RePlaid and $\approx 29.11$ for LangFlow. In both cases the $K$-IWAE approaches the matched VDM NELBO from above as $K$ grows: LangFlow's $K$-IWAE drops below the VDM NELBO at $K = 2$ (continuing well below as $K$ increases), whereas RePlaid's $K$-IWAE meets the VDM NELBO essentially at $K = 32$ ($22.11$ vs.\ $\replaidscppl$). That $K=1$ empirically exceeds the VDM NELBO reflects the SDE-vs-PFODE gap under imperfect scores~\citep{karras2022elucidating}, as discussed in \cref{subsec:ode-likelihood}; the gap closes (RePlaid) or reverses (LangFlow) as $K$ tightens the IWAE bound. For paper-quality likelihood reporting one should therefore use $K \gtrsim 16$ rather than the $K=1$ ELBO, which, due to the extensive computational cost, is infeasible on the full OpenWebText dataset given our computational resources.

\begin{figure}[h]
\centering
\includegraphics[width=0.6\linewidth]{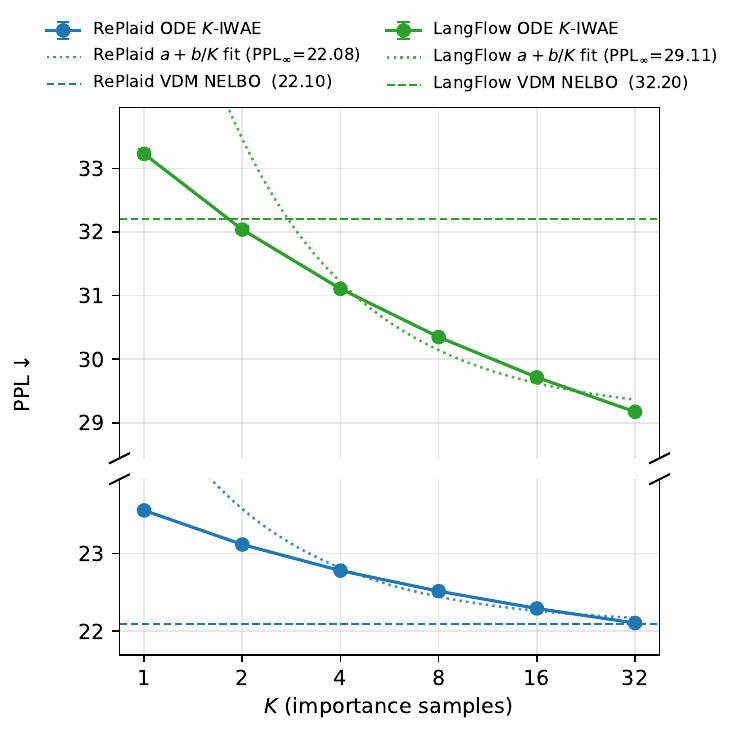}
\caption{\small IWAE $K$-curve of \cref{eq:ode-iwae} on a fixed $1024$-sequence OWT-valid subset (solid lines), with matched VDM NELBOs of \cref{eq:plaid-nelbo} drawn as dashed references. Dotted curves are the leading-order $a + b/K$ fits on $K \ge 4$ with extrapolated asymptotes $\text{PPL}_\infty$ in the legend; the visible deviation at $K \le 2$ reflects $O(1/K^2)$ corrections to the $\Theta(1/K)$ IWAE bias of \citet{nowozin2018debiasing}.}
\label{fig:ode-kcurve-dpi}
\end{figure}

\paragraph{Estimator-side sensitivity.} \cref{tab:ode-sensitivity} sweeps four solver / Hutchinson knobs at fixed $K=1$ on RePlaid, on the same $1024$ sequences. Doubling the discretization steps (\texttt{heun2-256}), switching to an adaptive solver (\texttt{dopri5}), swapping the $\bxi$ distribution from Rademacher to standard Gaussian, and increasing the per-step Hutchinson sample count up to $n_\mathrm{hutch}=16$ all leave the mean PPL within $0.4$ of the baseline $23.54$. This confirms that (i) the deterministic ODE solver is well-converged at \texttt{heun2-128}, (ii) the Hutchinson estimator is unbiased irrespective of the probe distribution, and (iii) raising $n_\mathrm{hutch}$ does not shift the mean---the residual scatter reflects Monte-Carlo variance from the $K=1$ importance sample, not the trace estimator.

\begin{table}[h]
\centering
\caption{\small Estimator-side hyperparameter sensitivity on \textsc{RePlaid} ($K=1$, $1024$ sequences) for ODE-based likelihood estimation. Mean PPL is preserved to within $\pm 0.4$ across all variants, consistent with an unbiased trace estimator and a converged ODE solver.}
\begin{tabular}{l c c c}
\toprule
Variation                                  & PPL$\downarrow$ & $\Delta$ vs.\ Baseline & avg.\ NFE/batch \\
\midrule
Baseline (\texttt{heun2-128}, Rademacher, $n_\mathrm{hutch}=1$) & $23.54$ & $-$        & $1{,}280$  \\
~~solver: \texttt{heun2-256}                   & $23.91$ & $+0.36$           & $2{,}560$  \\
~~solver: \texttt{dopri5} adaptive             & $23.83$ & $+0.28$           & $8{,}684$  \\
~~Hutchinson distribution: Gaussian            & $23.58$ & $+0.03$           & $1{,}280$  \\
~~$n_\mathrm{hutch}=2$                         & $23.58$ & $+0.03$           & $2{,}304$  \\
~~$n_\mathrm{hutch}=4$                         & $23.58$ & $+0.04$           & $4{,}352$  \\
~~$n_\mathrm{hutch}=8$                         & $23.57$ & $+0.02$           & $8{,}448$  \\
~~$n_\mathrm{hutch}=16$                        & $23.56$ & $+0.02$           & $16{,}640$ \\
\bottomrule
\end{tabular}
\vspace{0.5em}
\label{tab:ode-sensitivity}
\end{table}

\clearpage

\section{Constant Per-Timestep Diffusion Loss}\label{supp:sec:constant-diffusion-proof}

[Back to \cref{thm:constant_diffusion_loss} in the main paper]

\vspace{0.5em}
\thmconstantdiffusion*

\begin{proof}
In $\gamma$-coordinates the per-timestep loss is
\begin{align}
    \ell_{\theta,\gamma}(t) = \frac{1}{2}\,\gamma'(t)\,w(\gamma(t)),
    \qquad w(\gamma) := \ee^{-\gamma}\,\MSE_\theta(\ee^{-\gamma}),
\end{align}
since $\SNR(t)=\ee^{-\gamma(t)}$ implies $\SNR'(t) = -\gamma'(t)\,\ee^{-\gamma(t)}$. By assumption $w$ is positive and continuous on $[\gamma_0,\gamma_1]$. Changing variables to $\gamma$,
\begin{align}
    \kappa
    = \E_t[\ell_{\theta,\gamma}(t)]
    = \frac{1}{2}\!\int_{\gamma_0}^{\gamma_1} w(\gamma)\d\gamma,
\end{align}
which depends only on the endpoints, recovering the schedule invariance of~\citet{kingma2021variational}. Define the cumulative weight $G(\gamma) := \int_{\gamma_0}^\gamma w(\gamma)\d\gamma$; then $G(\gamma_0)=0$, $G(\gamma_1)=2\kappa$, and $G$ is strictly increasing on $[\gamma_0,\gamma_1]$. Set
\begin{align}\label{eq:optimal-gamma}
    \gamma^*(t) := G^{-1}(2\kappa\,t),
\end{align}
which gives $\gamma^*(0)=\gamma_0$ and $\gamma^*(1)=\gamma_1$ by construction. Differentiating $G(\gamma^*(t))=2\kappa t$ yields $(\gamma^*)'(t) = 2\kappa / w(\gamma^*(t))$, so
\begin{align}\label{eq:optimal-per-timestep-diffusion}
    \ell_{\theta,\gamma^*}(t)
    = \frac{1}{2}\,(\gamma^*)'(t)\,w(\gamma^*(t))
    = \kappa
    \qquad\text{for all }t\in[0,1],
\end{align}
hence $\Var_t[\ell_{\theta,\gamma^*}(t)]=0$. Uniqueness of $\gamma^*$ follows from the strict monotonicity of $G$.
\end{proof}

\clearpage

\section{Linear Information Decay Under Optimality}\label{supp:sec:linear-info-decay}

[Back to \cref{thm:linear_ce} in the main paper]

\vspace{0.5em}
\begin{restatable}[\textbf{\color{ourbrightblue!80!black}Linear Information Decay Under Optimality}]{lemma}{lemmalinearinfodecay}
\label{lemma:linear_info_decay}
Assuming a Bayes-optimal model $\theta^*$ with its optimal noise schedule $\gamma^*$ (\cref{thm:constant_diffusion_loss}) that achieves constant $\ell_{\theta^*, \gamma^*}(t)=\kappa\geq0$, we have
\begin{equation}
    I(\e; \z_t) = I(\e; \z_0) - \kappa t.
\end{equation}
Intuitively, $\z_t$ contains less information about what $\e$ is as $t$ increases.
\end{restatable}

\begin{proof}
Write $\nu_t := \SNR(t)$ and recall the minimum mean-square error of estimating $\e$ from $\z_t$,
\begin{align}
    \MMSE(\nu_t) := \E_{q_\data(\x)\,q(\z_t \mid \x)}\bigl[\,\|\e - \E[\e\mid\z_t]\|^2\,\bigr].
\end{align}
Rescaling by $\sigma_t$ casts $\z_t$ into canonical Gaussian-channel form,
\begin{align}
    \z_t / \sigma_t = \sqrt{\nu_t}\,\e + \beps,
    \qquad \beps \sim \normal(\zero, \I),
\end{align}
and dividing by the deterministic $\sigma_t$ preserves information about $\e$. The I-MMSE identity~\citep[Thm.~2]{guo2005mutual} therefore yields
\begin{align}
    \frac{\d}{\d \nu_t}\,I(\e;\z_t)
    = \frac{1}{2}\,\MMSE(\nu_t).
\end{align}
By \cref{def:bayes-optimal}, $\MSE_{\theta^*}(\nu_t)=\MMSE(\nu_t)$, so the constant-loss condition from \cref{thm:constant_diffusion_loss} becomes
\begin{align}
    \ell_{\theta^*,\gamma^*}(t)
    = -\frac{1}{2}\,\frac{\d\nu_t}{\d t}\,\MMSE(\nu_t)
    = \kappa.
\end{align}
Combining the two displays via the chain rule,
\begin{align}
    \frac{\d}{\d t}\,I(\e;\z_t)
    = \frac{\d I}{\d \nu_t}\,\frac{\d \nu_t}{\d t}
    = \frac{1}{2}\,\MMSE(\nu_t)\,\frac{\d \nu_t}{\d t}
    = -\kappa,
\end{align}
and integrating from $0$ to $t$ gives $I(\e;\z_t) = I(\e;\z_0) - \kappa\,t$.
\end{proof}

\clearpage

\section{Per-Timestep CE Under Optimality}\label{supp:sec:linear-ce}

[Back to \cref{thm:linear_ce} in the main paper]

\vspace{0.5em}

\thmlinearce*

\begin{proof}
By inserting $q(\x\mid\z_t)$ inside the log, the per-timestep CE loss decomposes into a conditional entropy plus a model gap:
\begin{align}
    \CE_{\theta^*, \gamma^*} (t)
    &= \E_{q_\data(\x)\,q(\z_t \mid \x)}[-\log \hat p_{\theta^*} (\x\mid\z_t)] \\
    &= \E_{q_\data(\x)\,q(\z_t\mid\x)}\!\left[-\log q(\x\mid\z_t) -\log \frac{\hat{p}_{\theta^*}(\x\mid\z_t)}{q(\x\mid\z_t)}\right]\\
    &= \underbrace{H(\x\mid\z_t)}_{\text{conditional entropy}} \;+\; \underbrace{\E_{q(\z_t)}\bigl[\KL\bigl(q(\x\mid\z_t)\,\|\,\hat{p}_{\theta^*}(\x\mid\z_t)\bigr)\bigr]}_{\text{model gap}}.
\end{align}

\paragraph{Computing the model gap.}
By \cref{def:bayes-optimal}, the Bayes-optimal denoiser realizes a per-position factorized model $\hat{p}_{\theta^*}(\x\mid\z_t)=\prod_l q(\x^l\mid\z_t)$. The KL between the true joint posterior and this factorized product is
\begin{align}
    \KL\bigl(q(\x\mid\z_t) \,\big\|\, \textstyle\prod_l q(\x^l\mid\z_t)\bigr)
    = \E_{q(\x\mid\z_t)}\Big[\log q(\x\mid\z_t) - \sum_l \log q(\x^l\mid\z_t)\Big],
\end{align}
and taking expectation over $\z_t \sim q(\z_t)$ gives, after rearrangement,
\begin{align}
    \E_{q(\z_t)}\bigl[\KL(q(\x\mid\z_t)\,\|\,\textstyle\prod_l q(\x^l\mid\z_t))\bigr]
    = \sum_l H(\x^l\mid\z_t) - H(\x\mid\z_t)
    \;=:\; C(\x\mid\z_t),
\end{align}
i.e., the conditional total correlation of $\x$ given $\z_t$. It is non-negative, and zero iff tokens of $\x$ are conditionally independent given $\z_t$. For language data $q(\x)$ has dimensional correlations, and as $q(\z_t\mid\x) = \prod_l q(\z_t^l\mid\x^l)$, the posterior  $q(\x\mid\z_t)$ also exhibits dimensional correlation. Therefore, $C(\x\mid\z_t)>0$ generically.

\paragraph{Computing the conditional entropy.}
By the definition of mutual information,
\begin{align}
    H(\x\mid\z_t) = H(\x) - I(\x;\z_t).
\end{align}
Under no embedding collapse, $\x^l\leftrightarrow\e^l$ is bijective per position, so $I(\x;\z_t)=I(\e;\z_t)$, and \cref{lemma:linear_info_decay} gives
\begin{align}
    I(\x;\z_t)=I(\e;\z_t) = I(\e;\z_0) - \kappa t \implies H(\x\mid\z_t) = H(\x) - I(\e;\z_0) + \kappa t.
\end{align}

\paragraph{Putting it together.}
Combining the conditional entropy and the model gap,
\begin{align}
    \CE_{\theta^*,\gamma^*}(t) = H(\x) - I(\e;\z_0) + \kappa t + C(\x\mid\z_t).
\end{align}

\paragraph{Monotonicity of $C(\x\mid\z_t)$ in $t$.}
Substituting $H(\x\mid\z_t) = H(\x) - I(\x;\z_t)$ and $H(\x^l\mid\z_t) = H(\x^l) - I(\x^l;\z_t)$ into the definition of $C(\x\mid\z_t)$,
\begin{align}
    C(\x\mid\z_t) = C(\x) - \Bigl[\sum_l I(\x^l;\z_t) - I(\x;\z_t)\Bigr],
\end{align}
so it suffices to show that $\sum_l I(\x^l;\z_t) - I(\x;\z_t)$ is non-increasing in $t$, or equivalently non-decreasing in the SNR $\nu_t := \alpha_t^2/\sigma_t^2$. Writing the channel in the standard form $\z_t/\sigma_t = \sqrt{\nu_t}\,\e + \xi$ with $\xi \sim \normal(0, \I)$ and applying the I-MMSE identity~\citep{guo2005mutual},
\begin{align}
    \frac{\d}{\d\nu_t} I(\x;\z_t) = \frac{1}{2} \sum_l M^l(\nu_t),
    \qquad M^l(\nu_t) := \E\!\left[\|\e^l - \E[\e^l\mid\z_t]\|^2\right].
\end{align}
Combining the chain rule $I(\x;\z_t) = I(\x^l;\z_t) + I(\x^{-l};\z_t\mid\x^l)$ with the conditional I-MMSE identity applied to the second term gives
\begin{align}
    \frac{\d}{\d\nu_t} I(\x^l;\z_t) = \frac{\d}{\d\nu_t} I(\x;\z_t) - \frac{1}{2}\sum_{l' \neq l} M^{l'}_{\,\mid\x^l}(\nu_t),
    \quad M^{l'}_{\,\mid\x^l}(\nu_t) := \E\!\left[\|\e^{l'} - \E[\e^{l'}\mid\z_t,\x^l]\|^2\right].
\end{align}
Summing over $l$ and rearranging,
\begin{align}
    \frac{\d}{\d\nu_t}\!\left[\sum_l I(\x^l;\z_t) - I(\x;\z_t)\right]
    = \frac{1}{2}\sum_{l \neq l'}\bigl[M^{l'}(\nu_t) - M^{l'}_{\,\mid\x^l}(\nu_t)\bigr] \geq 0,
\end{align}
where the inequality is due to tower property: conditioning on $\x^l$ in addition to $\z_t$ can only reduce the MMSE of $\e^{l'}$. Therefore $C(\x\mid\z_t)$ is non-decreasing in $t$, with strict increase whenever any pair $(\x^l,\x^{l'})$ with $l \ne l'$ remains correlated under the Bayes posterior $q(\x\mid\z_t)$.
\end{proof}

\clearpage

\clearpage

\clearpage

\section{Learning Noise Schedule in VDMs}
\label{app:noise-schedule-vdm}

[Back to~\cref{sec:background} in the main paper]

\begin{figure}[H]
\centering   
\begin{minipage}{\linewidth}
\footnotesize
\begin{python}
def min_diffusion_loss(e, e_model, gamma_tilde_net, gamma_bounds, optimizer):
    """
    e: Encoded clean data. Shape: [bs, L, d_e]
    """
    bs = e.shape[0]
    t = torch.rand(bs, device=e.device)  # Shape: [bs]
    
    # 1. Get normalized schedule shape
    gamma_tilde_t, gamma_tilde_prime_t = gamma_tilde_net(t)  
    # Shapes: [bs], [bs]
    
    # 2. Gradient Hook for Variance Minimization
    L_t_detached = None  # Will hold shape [bs] before backward()
    
    def var_min_hook(grad):
        # grad Shape: [bs]. L_t_detached Shape: [bs]. 
        return grad * 2.0 * L_t_detached
        
    gamma_tilde_t.register_hook(var_min_hook)
    gamma_tilde_prime_t.register_hook(var_min_hook)
    
    # 3. Apply Learnable Endpoints
    gamma_0, gamma_1 = gamma_bounds()  # Shapes: [1], [1] (Scalars)
    
    gamma_t = gamma_0 + (gamma_1 - gamma_0) * gamma_tilde_t  # Shape: [bs]
    gamma_prime_t = (gamma_1 - gamma_0) * gamma_tilde_prime_t  # Shape: [bs]
    
    # 4. Forward noising process
    alpha_t = torch.sqrt(torch.sigmoid(-gamma_t))  # Shape: [bs]
    sigma_t = torch.sqrt(torch.sigmoid(gamma_t))  # Shape: [bs]
    
    # Reshape for broadcasting across sequence and feature dimensions
    alpha_t_view = alpha_t.view(bs, 1, 1)  # Shape: [bs, 1, 1]
    sigma_t_view = sigma_t.view(bs, 1, 1)  # Shape: [bs, 1, 1]
    
    eps = torch.randn_like(e)  # Shape: [bs, L, d_e]
    z_t = alpha_t_view * e + sigma_t_view * eps  # Shape: [bs, L, d_e]
    
    # 5. Model prediction
    e_pred = e_model(z_t, gamma_t)  # Shape: [bs, L, d_e]
    
    # 6. Continuous diffusion loss calculation
    # average per-sample MSE
    mse = ((e - e_pred)**2).sum(dim=2).mean(dim=1)  # Shape: [bs]
    
    # SNR = exp(-gamma), SNR'(t) = -exp(-gamma(t)) * gamma'(t)
    snr_prime = -torch.exp(-gamma_t) * gamma_prime_t  # Shape: [bs]
    L_t = -0.5 * snr_prime * mse  # Shape: [bs]
    
    # 7. Single Backward Pass
    loss = torch.mean(L_t)  # Shape: scalar
    
    # Pass the detached per-sample loss to the hook closure
    L_t_detached = L_t.detach()  # Shape: [bs]
    
    # Backward pass automatically routes gradients
    optimizer.zero_grad()
    loss.backward() 
    optimizer.step()
\end{python}
\end{minipage}
    \caption{\small The noise schedule endpoints $(\gamma_0, \gamma_1)$ and the interior shape of $\gamma(t)$ are separately parameterized; they are trained to minimize the diffusion loss and its estimator variance respectively.}
    \label{fig:diffusion-loss-code}
\end{figure}

\clearpage

\section{Aligning Plaid with Modern Discrete DLM Architecture}\label{supp:sec:align-table}

[Back to~\cref{subsec:replaid-align} in the main paper]

\begin{table}[H]
\centering
\caption{\small Aligning components of Plaid with the modern discrete diffusion architecture to create RePlaid. \textbf{\color{gray}Gray} indicates what's used in the alternative method.}
\vspace{0.5em}
\label{tab:plaid-dit-alignment-arch}
\scriptsize
\renewcommand{\arraystretch}{1}
\begin{tabular}{lccl}
\toprule
\textbf{RePlaid} & \textbf{= MDLM~\citep{sahoo2024simple,sahoo2026scaling}} & \textbf{= Plaid ~\citep{gulrajani2023plaid}} & \textbf{Details} \\
\midrule
Embedding dimensionality & & \cmark & 16 dimensions \textcolor{lightgray}{(instead of 768 dimensions)} \\
Embedding corruption     & N.A. & \cmark & Gaussian noise \textcolor{lightgray}{(instead of pre-embedding masking)} \\
Post-corruption processing     & N.A. & \cmark & Rescale $+$ linear projection (no bias) to 768 dimensions\\
&&& Also linearly projected (no bias) \& added: self-cond., time-cond. \\
&&& (These are the three input projections.) \\
\midrule
Normalization           & \cmark & & LayerNorm \textcolor{lightgray}{(instead of RMSNorm)} \\
MLP biases              & \cmark & & Enabled \textcolor{lightgray}{(instead of no biases)} \\
MLP activation          & \cmark & & GELU(tanh) \textcolor{lightgray}{(instead of standard GELU)} \\
AdaLN modulation        & \cmark & & AdaLN-Zero \textcolor{lightgray}{(instead of no adaLN)} \\
Residual connection     & \cmark & & Learnable AdaLN gating \textcolor{lightgray}{(instead of fixed $1/\sqrt{N}$ scaling)} \\
\midrule
Noise schedule          & & \cmark & Learned \textcolor{black}{(monotone MLP $+$ noise schedule endpoints $\gamma_0, \gamma_1$)} \\
Discrete NELBO          & & \cmark & Reconst. loss $+$ diffusion loss \textcolor{black}{(MSE)} $+$ prior loss \\
\bottomrule
\end{tabular}
\end{table}

\begin{table}[H]
\centering
\caption{\small Precision details for RePlaid. \textbf{\color{gray}Gray} indicates what's used in the alternative method. (*) As with Plaid, we use \texttt{float64} for the learnable noise schedule to ensure its numerical stability, monotonicity, and differentiability; it propagates into the final loss aggregation only by dtype promotion.}
\vspace{0.5em}
\label{tab:plaid-dit-alignment-precision}
\renewcommand{\arraystretch}{1}
\resizebox{\linewidth}{!}{%
\begin{tabular}{lccc}
\toprule
\textbf{RePlaid} & \textbf{= MDLM~\citep{sahoo2024simple,sahoo2026scaling}} & \textbf{= Plaid ~\citep{gulrajani2023plaid}} & \textbf{Details} \\
\midrule
Precision (storing parameters)  & \cmark & \cmark & \texttt{float32} \\
Precision (forward pass)     & & & \\
\quad Noise schedule & N.A. & \cmark & \texttt{float64} \\
\quad Input linear projections & N.A. & \cmark & \texttt{float32} \\
\quad MLP & \cmark & \cmark & \texttt{bfloat16} \\
\quad Attention & \cmark & \cmark & \texttt{bfloat16} \\
\quad Residual stream & \cmark & \cmark & \texttt{float32} \\
\quad Final logit head & \cmark &  & \texttt{bfloat16} \textcolor{lightgray}{(not \texttt{float32})}\\
\quad Output prior logits & N.A. & \cmark & \texttt{float32} \\
\quad Categorical reparameterization & N.A. & \cmark & \texttt{float32} \\
\quad Per-element loss computation & \cmark & \cmark & \texttt{float32} \\
\quad Final loss aggregation \& SNR$'$ weighting (*) & & \cmark & \texttt{float64} \textcolor{lightgray}{(not \texttt{float32})} \\
\bottomrule
\end{tabular}
}
\end{table}

\vspace*{\fill}
\clearpage

\section{Optimization Details}
\label{supp:sec:opt}

[Back to~\cref{subsec:scaling-law-setup} or~\cref{sec:small-models} in the main paper]

\begin{table}[h]
\centering
\caption{\small Optimization details for \textbf{RePlaid} in scaling law experiments. The initial LR and weight decay for the embedding matrix, the monotone MLP, and $(\gamma_0,\gamma_1)$ follow from Plaid~\citep{gulrajani2023plaid}; for these components, we found that using lower LR hurts RePlaid when training horizon is short. Meanwhile, using $1\ee{-}2$ as the embedding LR for MDLM did not lead to meaningful improvements (at 0.1B scale, MDLM still reaches PPL 25.0 at 250K steps with low variance training).}
\vspace{0.5em}
\label{tab:plaid-dit-optimization-scaling}
\renewcommand{\arraystretch}{1}
\resizebox{\linewidth}{!}{%
\begin{tabular}{ll}
\toprule
\textbf{RePlaid} & \textbf{Details} \\
\midrule
Optimizer               & AdamW \\
LR and regularization (embedding matrix)   & Cosine decay from $1\ee{-}2$ to $2\ee{-}5$; no weight decay\\
LR and regularization (monotone MLP) & Cosine decay from $1\ee{-}2$ to $2\ee{-}5$; no weight decay\\
LR and regularization ($\gamma_0, \gamma_1$) & Cosine decay from $1\ee{-}2$ to $2\ee{-}5$; weight decay $0.1$\\
LR and regularization (3 linear projections) & Cosine decay from $2\ee{-}4$ to $2\ee{-}5$; weight decay $0.1$ \\
LR and regularization (transformer backbone) & Cosine decay from $2\ee{-}4$ to $2\ee{-}5$; weight decay $0.1$ \\
Adam Moments & $\beta_1=0.9$, $\beta_2=0.95$ \\
EMA (for validation only; applied to all parameters) & $0$ \\
\bottomrule
\end{tabular}
}
\end{table}

\begin{table}[h]
\centering
\caption{\small Optimization details for \textbf{Plaid} at the $0.1$B scale. We use the same LRs (but not the same scheduler), weight decays, and Adam moments as Plaid~\citep{gulrajani2023plaid}. The LR and weight decay values shown below are per-component values passed to AdamW after muP~\citep{yang2021tuning} rescaling.}
\vspace{0.5em}
\label{tab:plaid-dit-optimization-small}
\renewcommand{\arraystretch}{1}
\resizebox{\linewidth}{!}{%
\begin{tabular}{ll}
\toprule
\textbf{RePlaid} & \textbf{Details} \\
\midrule
Optimizer               & AdamW (wrapped by MuAdam~\citep{yang2021tuning}) \\
LR and regularization (embedding matrix)   & $1\ee{-}2$; no weight decay \\
LR and regularization (monotone MLP) & $1\ee{-}2$; no weight decay \\
LR and regularization ($\gamma_0, \gamma_1$) & $1\ee{-}2$; weight decay $0.1$ \\
LR and regularization (3 input linear projections) & $1.4\ee{-}3$; weight decay $\sim 0.029$ \\
LR and regularization (transformer backbone) & $4.67\ee{-}4$ for attention QKV/out and MLP weights; \\
& weight decay $\sim 0.086$ \\
& $1.4\ee{-}3$ for RMSNorm, output LayerNorm, \\
& final linear layer's weight, and final linear layer's bias; \\
& weight decay $\sim 0.029$ \\
LR schedule for all components & Constant with linear warmup ($2500$ steps) \\
Adam moments & $\beta_1=0.9$, $\beta_2=0.99$ \\
EMA (for validation only; applied to all parameters) & $0.9999$ \\
\bottomrule
\end{tabular}
}
\end{table}

\begin{table}[h]
\centering
\caption{\small Optimization details for \textbf{RePlaid} at the $0.1$B scale.}
\vspace{0.5em}
\label{tab:replaid-dit-optimization-small}
\renewcommand{\arraystretch}{1}
\begin{tabular}{ll}
\toprule
\textbf{RePlaid} & \textbf{Details} \\
\midrule
Optimizer               & AdamW \\
LR and regularization (embedding matrix)   & $1\ee{-}2$; no weight decay \\
LR and regularization (monotone MLP) & $1\ee{-}2$; no weight decay\\
LR and regularization ($\gamma_0, \gamma_1$) & $1\ee{-}2$; weight decay $0.1$\\
LR and regularization (3 linear projections) & $3\ee{-}4$; no weight decay\\
LR and regularization (transformer backbone) & $3\ee{-}4$; dropout $0.1$ \\
LR schedule & Constant with linear warmup ($2500$ steps) \\
Adam Moments & $\beta_1=0.9$, $\beta_2=0.999$ \\
EMA (for validation/sampling; applied to all parameters) & $0.9999$ \\
\bottomrule
\end{tabular}
\end{table}

\clearpage

\section{Transformer Configurations}
\label{supp:sec:transformer}
%


[Back to~\cref{subsec:isoflop} or~\cref{subsec:likelihood-eval} in the main paper]

Following the convention of~\citep{kaplan2020scaling, hoffmann2022training}, we define non-embedding parameters to include the final logit head, excluding the embedding table and the three input projections.

\begin{table}[H]
    \centering
    \caption{\small Transformer configurations of AR, MDLM, and \textsc{RePlaid} used in our scaling law study. This follows from~\citet{nie2025scaling} and~\citet{sahoo2026scaling}. We include it for completeness.} 
    \vspace{1em}
    \label{tab:model_sizes}
    \begin{tabular}{c|ccc}
    Non-Embedding Parameters (M) & $n_\text{embed}$ &  $n_\text{layers}$ & $n_\text{heads}$ \\
    \toprule
    $14$ & $256$ & $6$ & $4$ \\
    $29$ & $384$ & $8$ & $6$ \\
    $44$ & $512$ & $8$ & $8$ \\
    $58$ & $576$ & $9$ & $9$ \\
    $74$ & $640$ & $10$ & $10$ \\
    $91$ & $640$ & $13$ & $10$ \\
    $107$ & $640$ & $16$ & $8$ \\
    $116$ & $768$ & $12$ & $12$ \\
    $140$ & $768$ & $15$ & $12$ \\
    $163$ & $768$ & $18$ & $12$ \\
    $173$ & $896$ & $14$ & $14$ \\
    $194$ & $896$ & $16$ & $14$ \\
    $214$ & $896$ & $18$ & $14$ \\
    $247$ & $1024$ & $16$ & $16$ \\
    $274$ & $1024$ & $18$ & $16$ \\
    $300$ & $1024$ & $20$ & $16$ \\
    $413$ & $1280$ & $18$ & $10$ \\
    $475$ & $1280$ & $21$ & $10$ \\
    $493$ & $1408$ & $18$ & $11$ \\
    $537$ & $1280$ & $24$ & $10$ \\
    $568$ & $1408$ & $21$ & $11$ \\
    $642$ & $1408$ & $24$ & $11$ \\
    $698$ & $1536$ & $22$ & $12$ \\
    $787$ & $1536$ & $25$ & $12$ \\
    $1016$ & $1792$ & $24$ & $14$ \\
    $1208$ & $2048$ & $22$ & $16$ \\
    $1364$ & $2048$ & $25$ & $16$ \\
    $1708$ & $2176$ & $28$ & $17$ \\
    \bottomrule
    \end{tabular}
\end{table}

\begin{table}[H]
    \centering
    \caption{\small The transformer configuration of AR, MDLM, Duo, FLM, LangFlow, and \textsc{RePlaid} used in our 0.1B experiments.} 
    \vspace{1em}
    \label{tab:model_sizes_0d1b}
    \begin{tabular}{c|ccc}
    Non-Embedding Parameters (M) & $n_\text{embed}$ &  $n_\text{layers}$ & $n_\text{heads}$ \\
    \toprule
    $116$ & $768$ & $12$ & $12$ \\
    \bottomrule
    \end{tabular}
\end{table}

\clearpage

\section{IsoFLOP Curves}
\label{app:isoflop-all}

[Back to~\cref{subsec:isoflop} in the main paper]

\vspace*{\fill}

\begin{figure}[H]
    \centering
    \begin{subfigure}{0.4\linewidth}
        \centering
        \includegraphics[width=\linewidth]{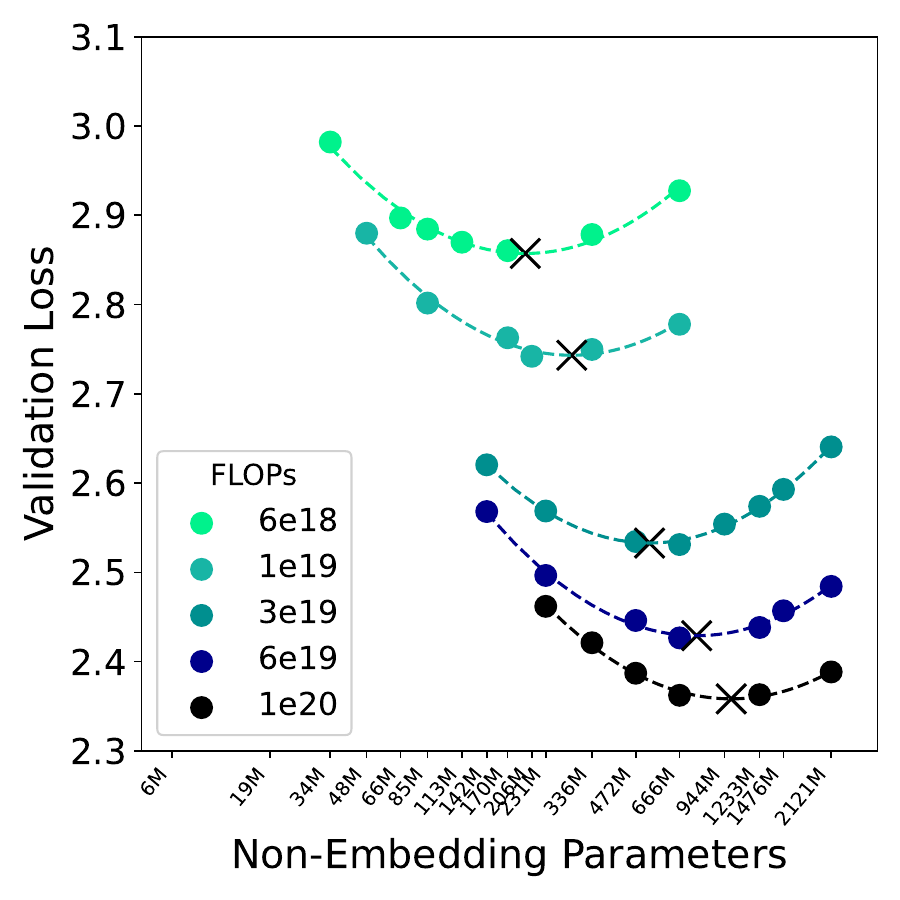}
        \caption{AR}
        \label{fig:iso_flop_AR_appendix}
    \end{subfigure}
    \begin{subfigure}{0.4\linewidth}
        \centering
        \includegraphics[width=\linewidth]{figures_scaling/iso_flop_MDLM-LV.pdf}
        \caption{MDLM with (low var.)}
        \label{fig:iso_flop_mdlm_appendix}
    \end{subfigure}
    
    \begin{subfigure}{0.4\linewidth}
        \centering
        \includegraphics[width=\linewidth]{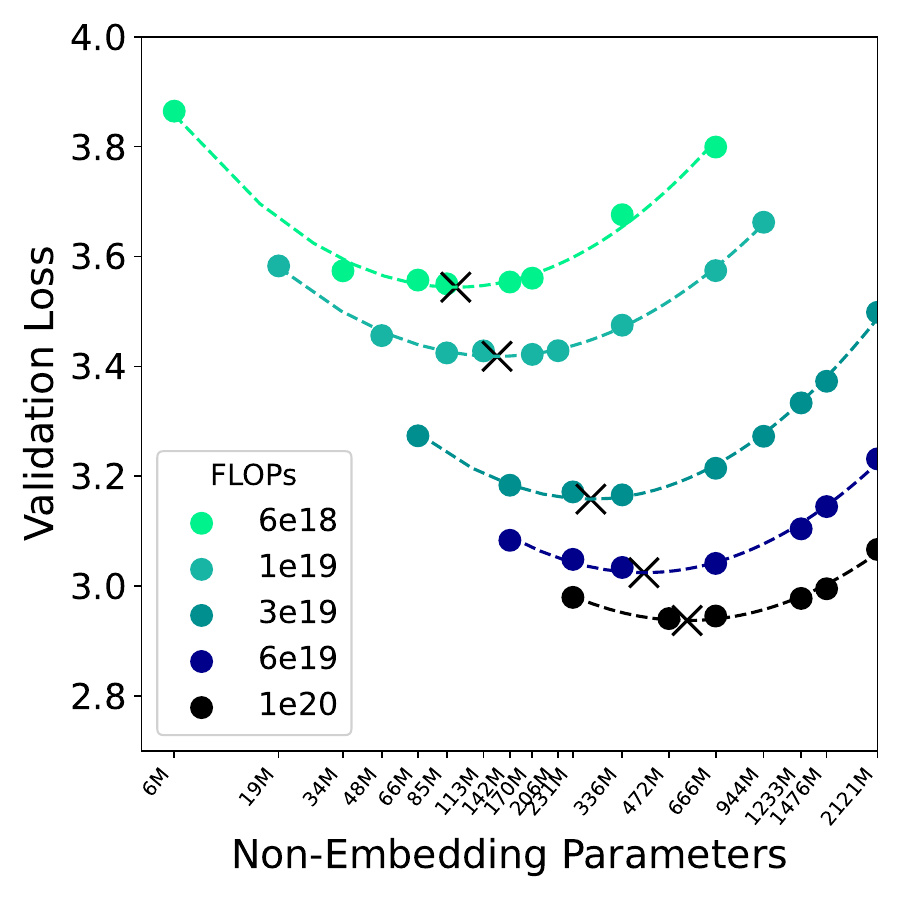}
        \caption{Duo}
        \label{fig:iso_flop_duo_appendix}
    \end{subfigure}
    
    \begin{subfigure}{0.4\linewidth}
        \centering
        \includegraphics[width=\linewidth]{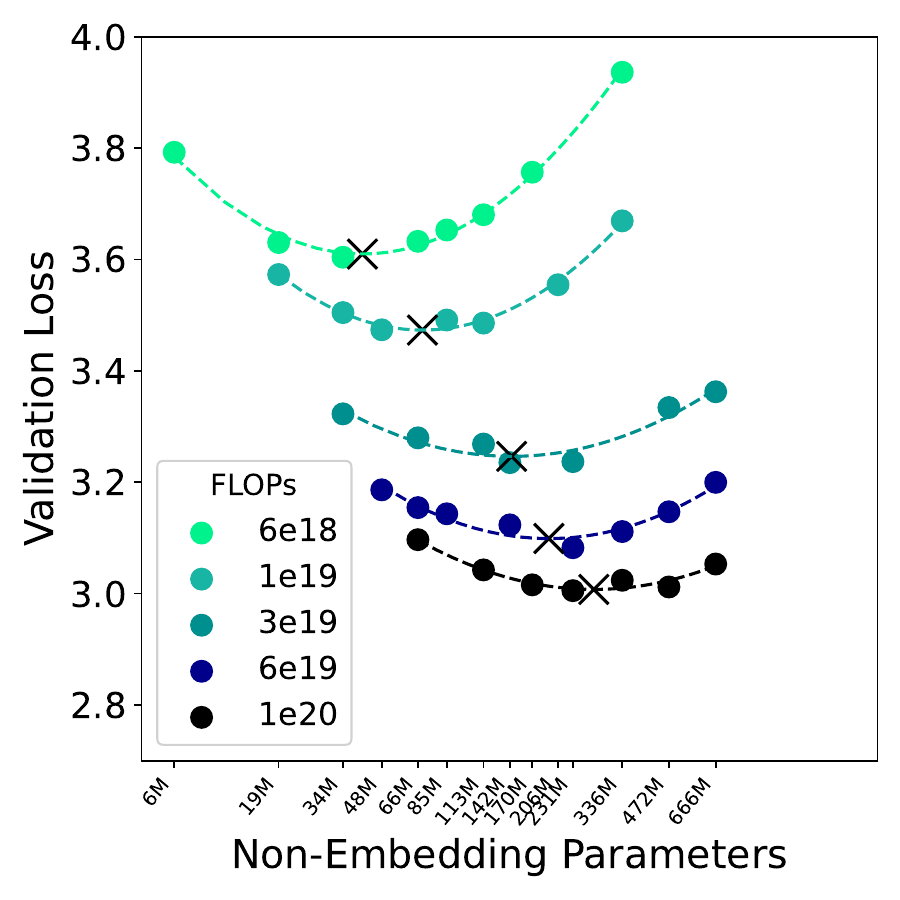}
        \caption{\textsc{RePlaid} (no s.c.).}
        \label{fig:iso_flop_replaid_appendix}
    \end{subfigure}
    \begin{subfigure}{0.4\linewidth}
        \centering
        \includegraphics[width=\linewidth]{figures_scaling/iso_flop_PLAID-V2-SC.pdf}
        \caption{\textsc{RePlaid} (s.c.).}
        \label{fig:iso_flop_replaid_sc_appendix}
    \end{subfigure}
    \caption{\small IsoFLOP curves plot optimal model sizes under fixed compute budgets. The optimal \textsc{RePlaid} loss exhibits power-law scaling, decreasing at a rate comparable to AR and MDLM. MDLM (low var.), Duo, \textsc{RePlaid} (s.c.), and \textsc{RePlaid} (no s.c.) exhibits $14\times$, $22\times$, $\replaidscx\times$, and $\replaidnscx\times$ worse scaling than AR respectively. In the over-trained region below the {\color[RGB]{80, 150, 63}\textbf{green line}}, \textsc{RePlaid} (s.c.) beats MDLM (low var.).}
    \label{fig:scaling_law_appendix}
\end{figure}

\vspace*{\fill}

\clearpage

\section{Exchanging Embedding Configurations of MDLM and RePlaid}

[Back to~\cref{subsec:scaling-law-results} in the main paper]

MDLM -- exchanged: $d_e=16$ with linear up-projection to 768 dimensions, length-normalized embeddings, higher initial learning rate for embeddings (\texttt{1e-2}), no weight decay for embeddings.

RePlaid (s.c.) -- exchanged: $d_e=768$ without linear up-projection, embeddings are not length-normalized, initial learning rate and weight decay same as the transformer backbone (\texttt{2e-4}; 0.1).

\begin{figure}[H]
    \centering
    \begin{subfigure}{0.4\linewidth}
        \centering
        \includegraphics[width=\linewidth]{figures_scaling/iso_flop_MDLM-LV.pdf}
        \caption{MDLM with (low var.)}
        \label{fig:iso_flop_mdlm_appendix_v2}
    \end{subfigure}
    \begin{subfigure}{0.4\linewidth}
        \centering
        \includegraphics[width=\linewidth]{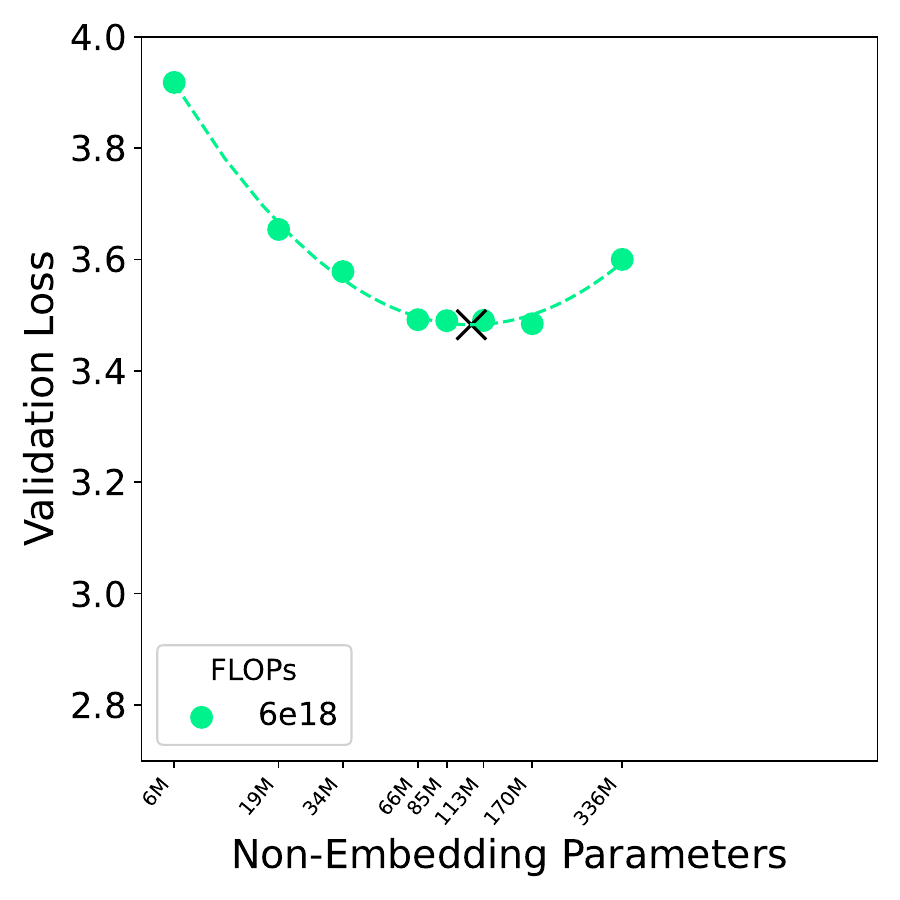}
        \caption{MDLM with (low var.) -- exchanged.}
        \label{fig:iso_flop_mdlm_exchange_appendix}
    \end{subfigure}
    
    \begin{subfigure}{0.4\linewidth}
        \centering
        \includegraphics[width=\linewidth]{figures_scaling/iso_flop_PLAID-V2-SC.pdf}
        \caption{\textsc{RePlaid} (s.c.).}
        \label{fig:iso_flop_replaid_appendix_v2}
    \end{subfigure}
    \begin{subfigure}{0.4\linewidth}
        \centering
        \includegraphics[width=\linewidth]{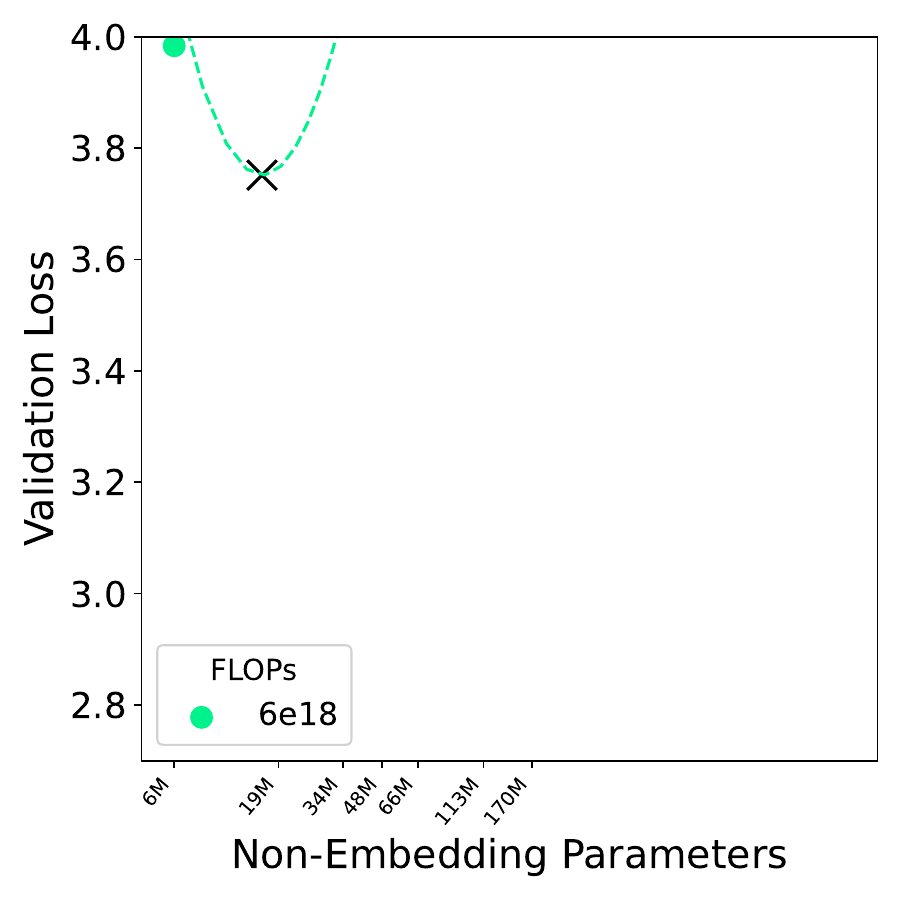}
        \caption{\textsc{RePlaid} (s.c.) -- exchanged.}
        \label{fig:iso_flop_replaid_sc_exchange_appendix}
    \end{subfigure}
    \caption{\small Exchanging the embedding configurations of MDLM and RePlaid (s.c.) degrades both methods. For comparison, we unify the vertical range of all methods and do not show points outside of this range (e.g., \textbf{(d)}).}
    \label{fig:exchange}
\end{figure}

\clearpage

\clearpage

\section{Training Dynamics of RePlaid at the 0.1B Scale on OWT} 
\label{app:dynamics}

[Back to~\cref{subsec:likelihood-eval} in the main paper]


\begin{figure*}[h]
    \centering
    \includegraphics[width=1\linewidth]{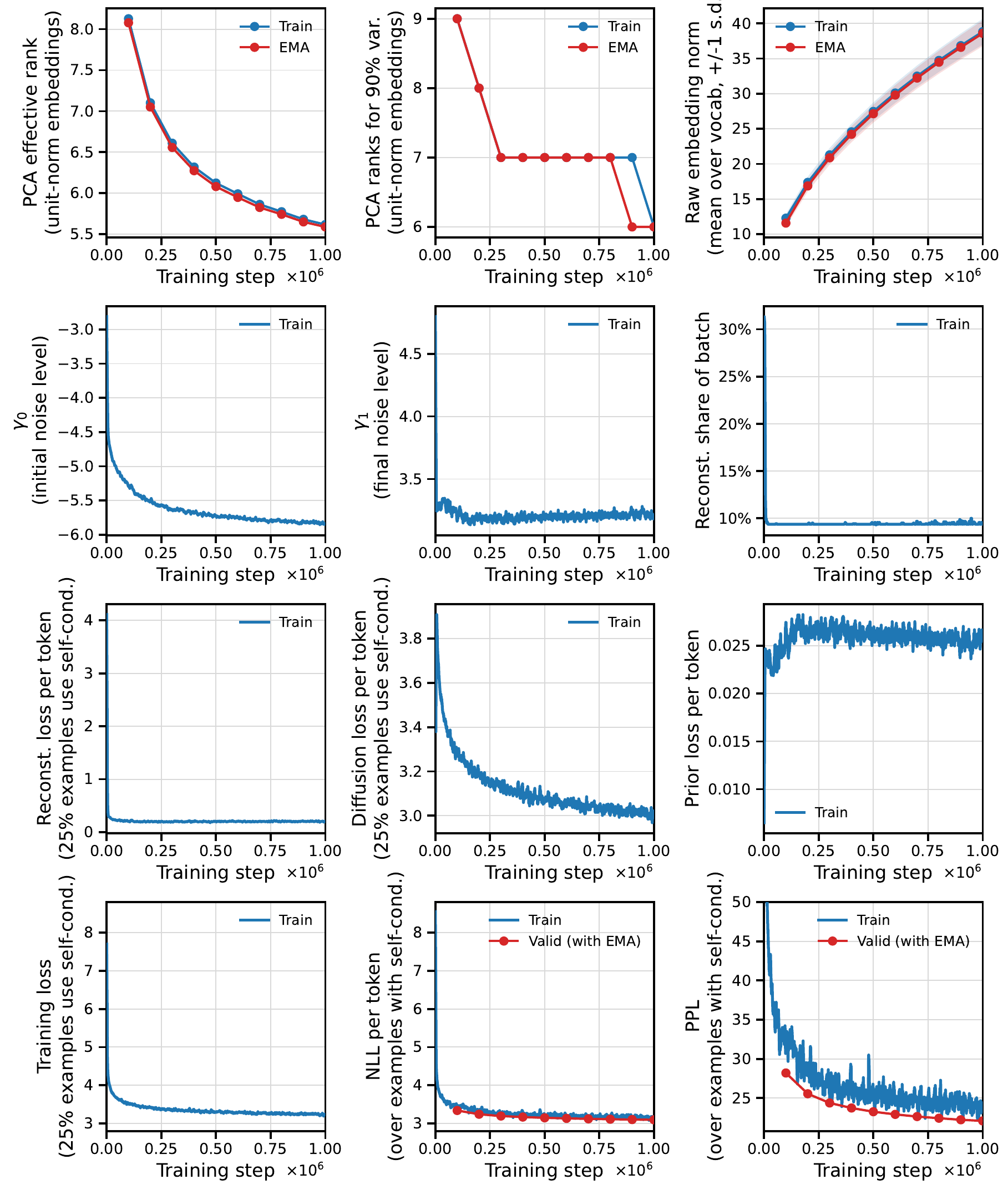}
    \caption{\small \textbf{Training dynamics of RePlaid at the 0.1B scale on OWT.} Reconst. share of batch shows the fraction of global batch size (512) allocated to the reconstruction loss. Since we are using 16 GPUs, the per-rank batch size is 32. Since this fraction is tied across GPU ranks, it can only take values in multiples of 3.125\%.}
    \label{fig:reconst-bs}
    \vspace{-0.5em}
\end{figure*}

\clearpage

\section{ODE Samplers}
\label{app:ode-samplers}

[Back to~\cref{subsec:sampling} in the main paper]

\cref{fig:sample-quality-2nd} compares RePlaid's DDPM ancestral sampler against the deterministic samplers, with self-conditioning on and off. Both panels show the same qualitative pattern: at small $T$ the second-order DPM-Solver++(2M) clearly beats both DDPM and the first-order DDIM, confirming that solver order matters in the few-step regime as in image diffusion. From $T \ge 64$, however, DDPM overtakes both deterministic solvers and plateaus at the highest MAUVE. Heun seems less favorable compared to DPM-Solver++(2M), likely because Heun costs 2 NFEs per step while DPM-Solver++(2M) only costs 1 NFE per step, and $T$ is adjusted to account for Heun's more expensive steps.

\begin{figure*}[h]
    \centering
    \begin{subfigure}{0.42\linewidth}
        \centering
        \includegraphics[width=\linewidth]{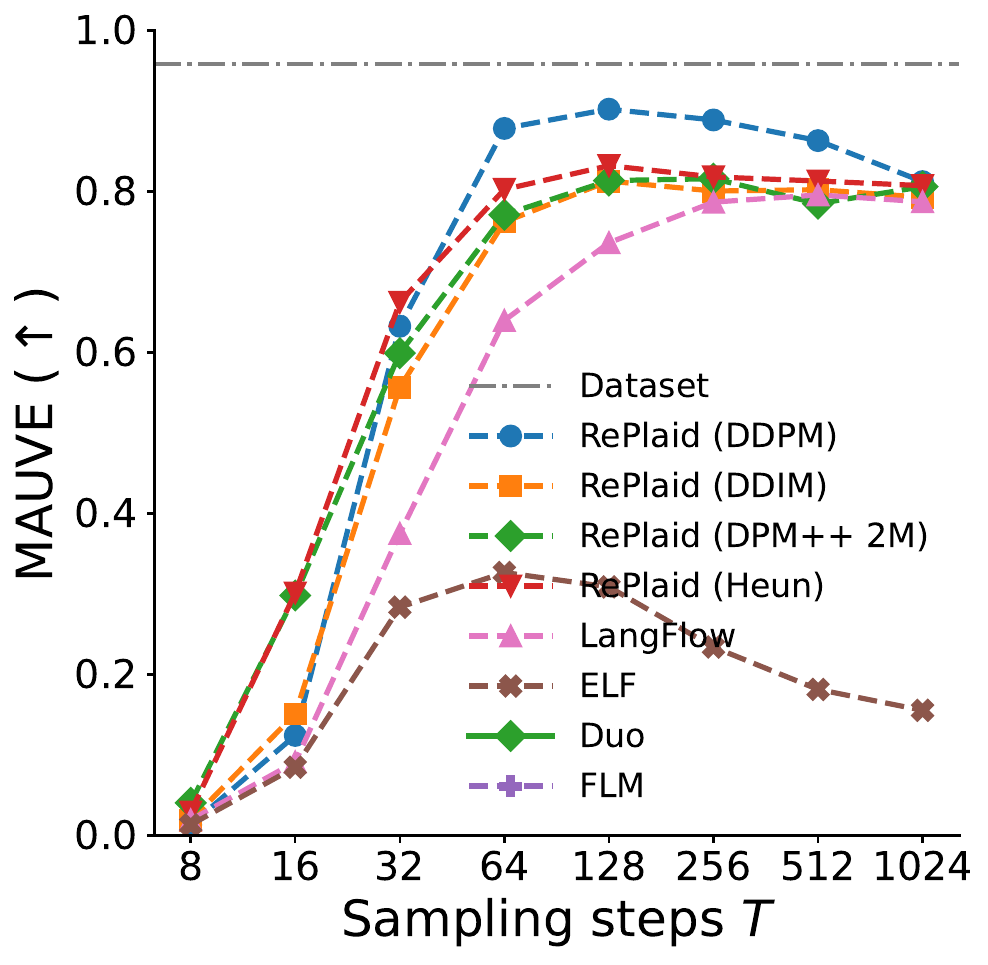}
        \caption{Self-conditioning}
        \label{fig:mauve_vs_T_2nd_sc}
    \end{subfigure}
    \begin{subfigure}{0.42\linewidth}
        \centering
        \includegraphics[width=\linewidth]{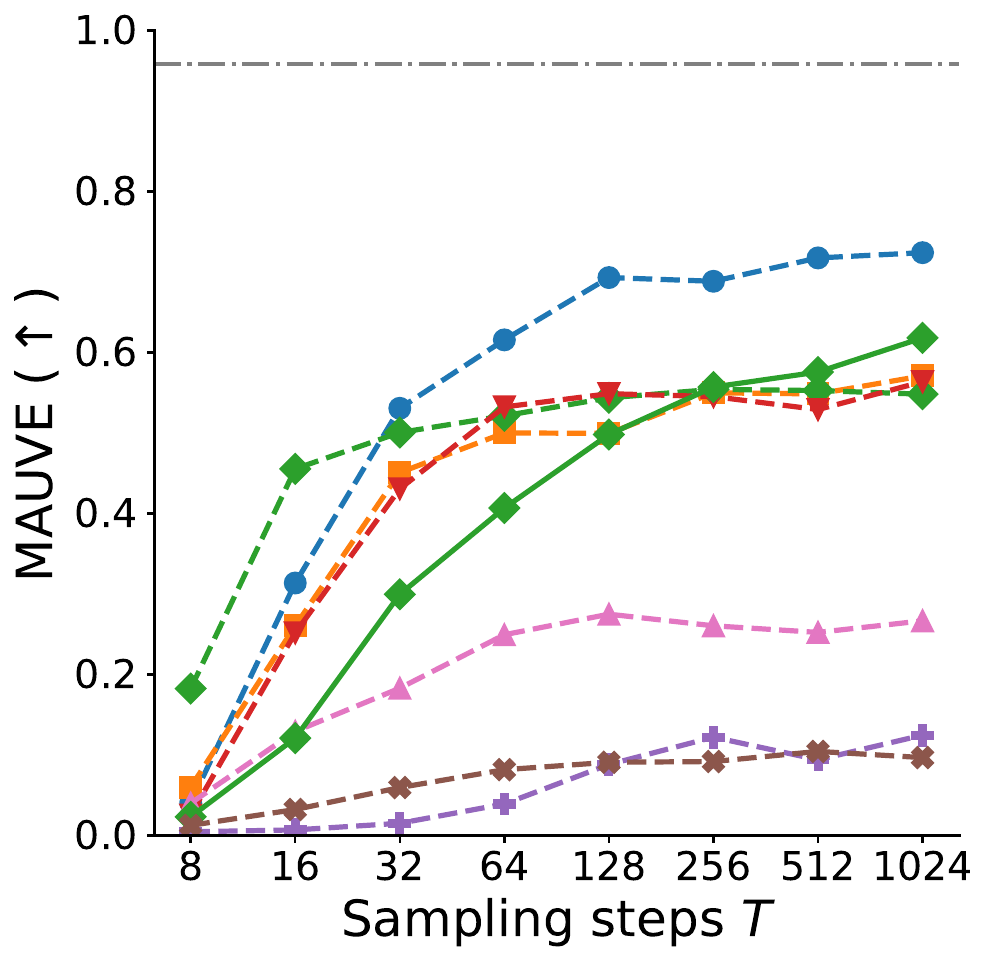}
        \caption{No self-conditioning / off}
        \label{fig:mauve_vs_T_2nd_nsc}
    \end{subfigure}
    \caption{\small MAUVE on OWT samples $(L=1024, N=5120)$ versus sampling steps $T$, comparing the ancestral DDPM sampler, DDIM, DPM-Solver++(2M), Heun on \textsc{RePlaid}, with LangFlow and Duo plotted as baselines.}
    \label{fig:sample-quality-2nd}
\end{figure*}

\clearpage

\section{Sampling with Temperature}
\label{app:temp}

[Back to~\cref{subsec:sampling} in the main paper]

By default, all of our headline sampling comparisons are sampled at temperature $\tau = 1$, i.e., no temperature sampling. \citet{pynadath2025candi} discussed the necessity of performing \textit{temperature sweeps} and report \textit{GenPPL-entropy frontiers} at each sampling budget $T$, as GenPPL can be made artificially low (good) by collapsing entropy. We tackle this problem by reporting MAUVE~\citep{pillutla2021mauve}, a metric designed to capture both quality and diversity. Empirically, MAUVE assigns low scores to methods that achieve low GenPPL but low entropy (e.g., see FLM in~\cref{fig:sample-quality}). 

To complement the $\tau = 1$ comparison, we additionally sweep $\tau$ for each method to trace out its GenPPL--entropy Pareto frontier, following CANDI~\citep{pynadath2025candi}. We control $\tau$ by scaling the logits of $\x_\theta$ (as done in CANDI), rather than the score-temperature mechanism used in the original Plaid~\citep{dieleman2022continuous, gulrajani2023plaid}; we found score temperature to perform substantially worse at low NFE but slightly better at high NFE in the low-entropy regime. The final $\argmax$ readout at $t = 0$ is unaffected by $\tau$.

\begin{figure}[h]
    \centering
    \includegraphics[width=0.9\linewidth]{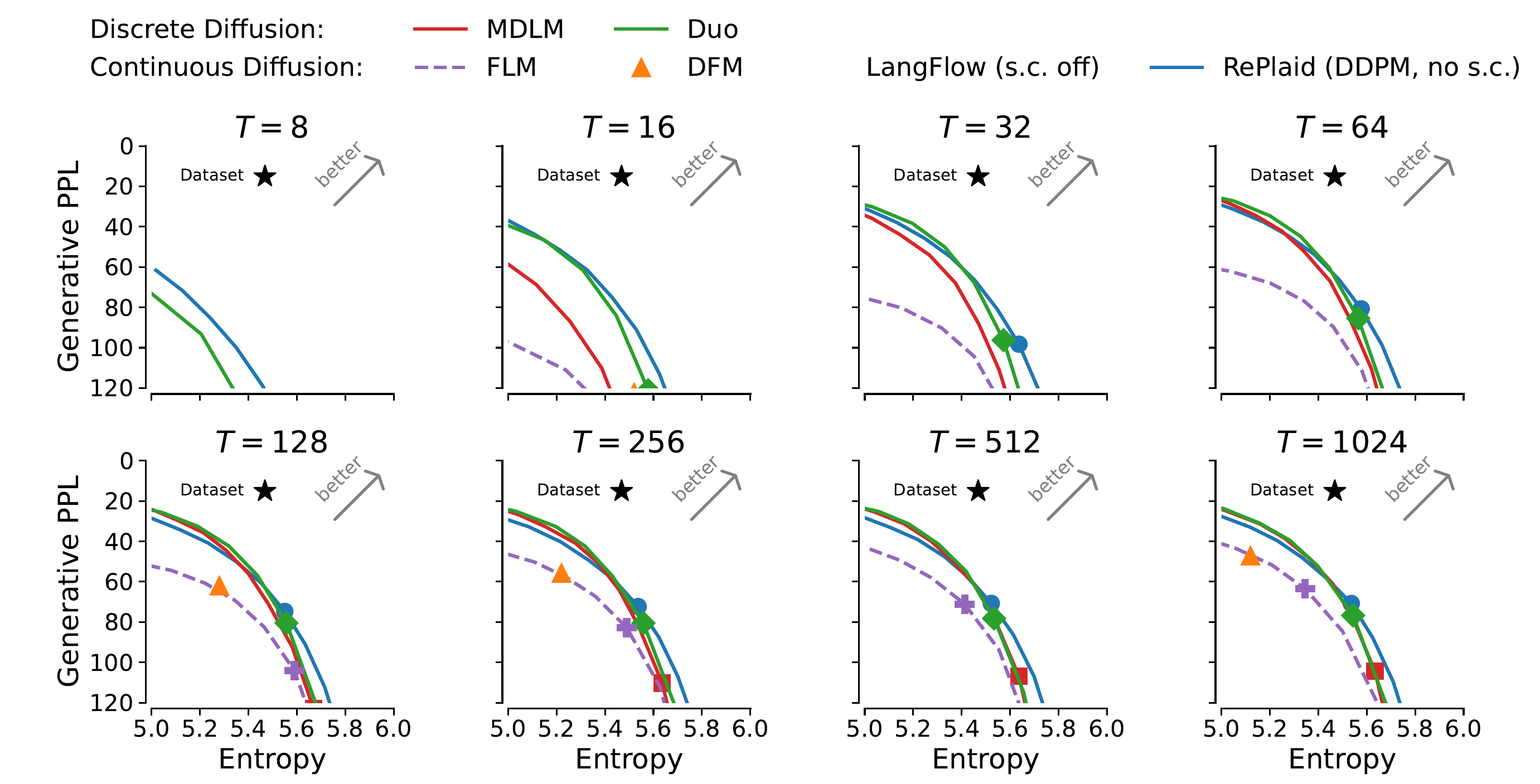}
    \caption{\small Quality-diversity trade-off of \textsc{RePlaid} against discrete and continuous DLM baselines for unconditional generation on OWT. Markers denote $\tau=1$. We use DFM numbers reported in~\citet{potaptchik2026discrete} ($T=512$ not available). Comparing methods at a matched entropy budget, \textsc{RePlaid} is competitive with discrete DLMs and surpasses prior continuous DLMs.}
    \label{fig:tradeoff-baselines}
\end{figure}

\begin{figure}[h]
    \centering
    \includegraphics[width=0.82\linewidth]{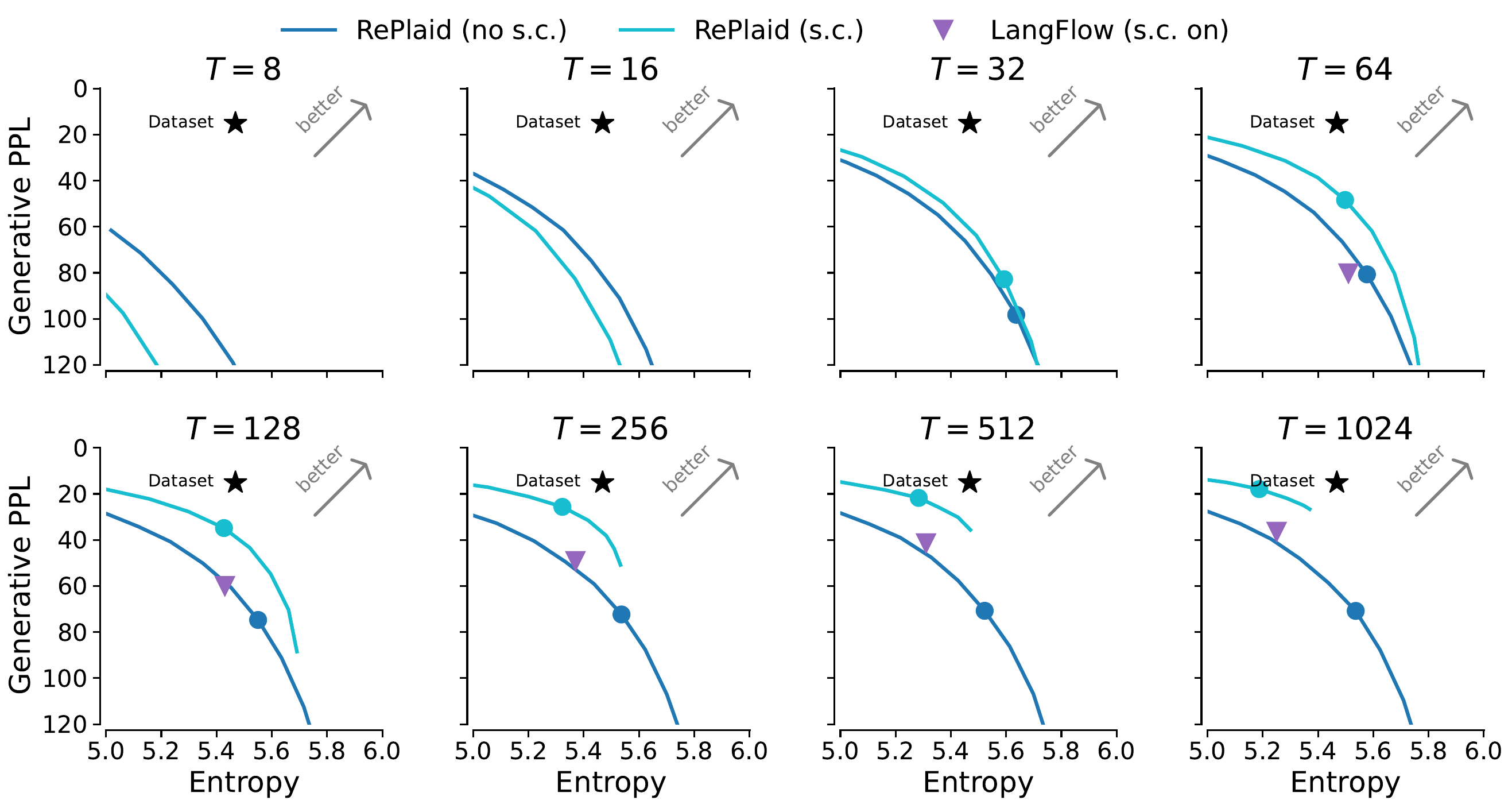}
    \caption{\small Quality-diversity trade-off of \textsc{RePlaid} against LangFlow for unconditional generation on OWT. Markers denote $\tau=1$. High $\tau$'s lead to degenerate samples for both LangFlow and \textsc{RePlaid} with self-conditioning on; parts of the curves corresponding to this degenerate behavior are omitted.}
    \label{fig:tradeoff-score-vs-logit}
\end{figure}

\begin{figure*}[h]
    \centering
    \begin{subfigure}{0.42\linewidth}
        \centering
        \includegraphics[width=\linewidth]{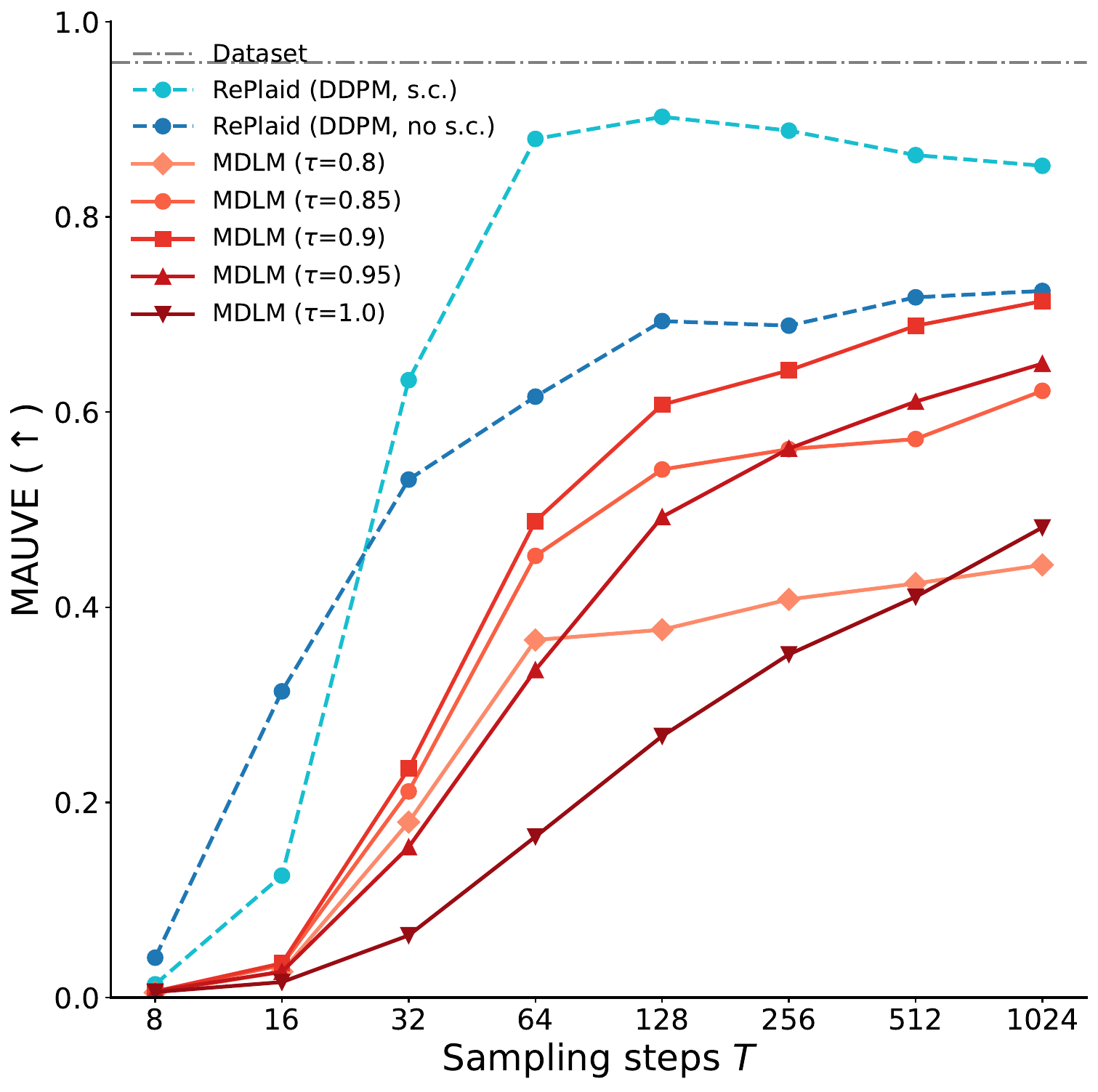}
        \caption{Sweeping temperature for MDLM }
        \label{fig:plaid_mauve_vs_T_mdlm_temp}
    \end{subfigure}
    \begin{subfigure}{0.42\linewidth}
        \centering
        \includegraphics[width=\linewidth]{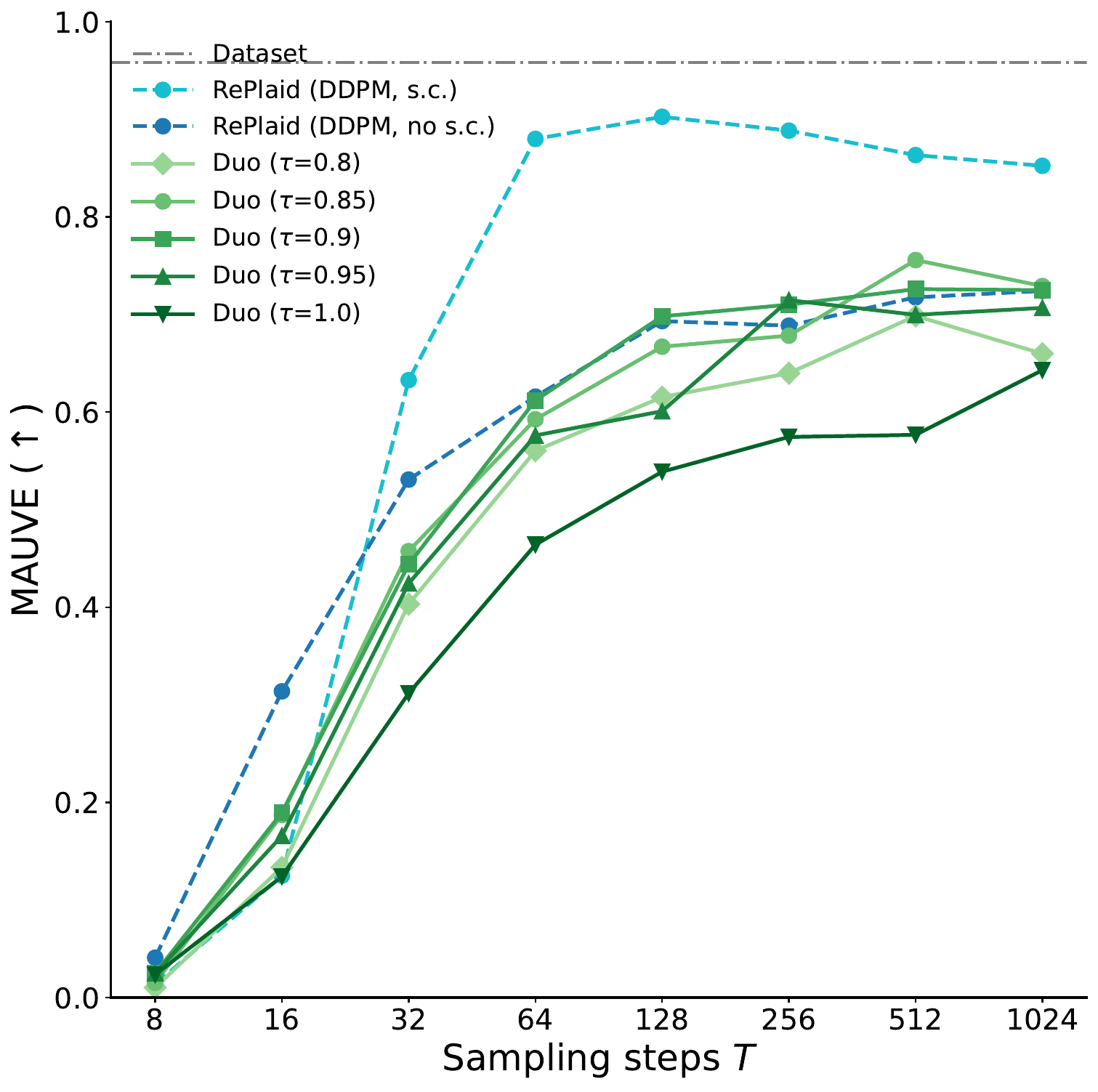}
        \caption{Sweeping temperature for Duo}
        \label{fig:plaid_mauve_vs_T_duo_temp}
    \end{subfigure}
    \caption{\small For unconditional generation on OWT, MDLM and Duo with optimal temperature perform worse than or on par with RePlaid (no s.c.).}
    \label{fig:mauve-over-temp-mdlm-duo}
\end{figure*}

\clearpage

\section{TinyGSM Experimental Details and Results}
\label{app:tinygsm}

[Back to~\cref{subsec:tinygsm} in the main paper]

We follow the setup of~\citet{deschenaux2026hyperspherical} and its official implementation, which we adopt unchanged for RePlaid so that the comparison against their released MDLM, Duo, and FLM checkpoints is controlled. We choose TinyGSM~\citep{liu2023tinygsm} because it was specifically designed to make GSM8K evaluation meaningful at the 0.1B scale.

\paragraph{Data.} TinyGSM~\citep{liu2023tinygsm}, tokenized with \texttt{HuggingFaceTB/SmolLM-135M}. Each example follows the template
\begin{center}
\texttt{BOS} + natural-language math problem + separator + solution (an executable Python function) + \texttt{EOS} + \texttt{PAD},
\end{center}
padded to $512$ tokens in total. Examples exceeding $512$ tokens before padding are removed. Evaluation uses the official GSM8K~\citep{cobbe2021gsm8k} test set; generated programs are executed by a Python interpreter and scored by pass@1 (0-shot, one sampling seed) with max response length ${}= 512 - \text{prompt length}$.

\paragraph{Architecture.} The same model size (0.1B) and architecture as our LM1B and OWT experiments, with bidirectional attention over the full $512$-token context.

\paragraph{Prompt conditioning versus diffusion modeling.} Conditioning is cast as fixed-prefix infilling. The diffusion loss is computed \emph{only} over ``solution + \texttt{EOS} + \texttt{PAD}'' tokens, and is aggregated and averaged over those tokens alone. The ``\texttt{BOS} + prompt + separator'' tokens are never noised and never enter the loss; they serve purely as conditioning.

\paragraph{Optimization.} $250$K gradient steps; global batch size $512$; AdamW with $\beta_1 = 0.9$, $\beta_2 = 0.999$; warmup-constant LR schedule with starting LR $0$, peak LR $3\ee{-}4$, and $2500$ warmup steps; EMA $0.9999$; gradient clipping $1.0$; weight decay $0$; dropout $0.1$. This matches our LM1B and OWT setup except for the sequence length ($128$ and $1024$ there), the presence of prompt conditioning (absent in both), and the number of training steps ($1$M there). As for LM1B, OWT, and SlimPajama (\cref{subsec:scaling-law-setup}), RePlaid applies this configuration with minor adjustments to the learning rate and weight decay of its learnable noise schedule and its unit-length, low-dimensional ($d_e = 16$) embeddings (\cref{tab:replaid-dit-optimization-small}), following Plaid.

\paragraph{Separate prompt embedding table.}
The prompt is never noised, so it has no reason to live in the low-dimensional, length-normalized embedding space that the diffusion process operates in. We therefore introduce a \emph{second} embedding table, $\Emb_\mathrm{p} \in \R^{V \times h}$: an ordinary (unnormalized) token-embedding matrix at the full transformer hidden size $h = 768$, initialized from $\normal(0, 0.02^2)$ and used only to encode prompt tokens. It is entirely separate from the diffusion embedding $\Emb \in \R^{V \times d_e}$ of \cref{sec:background}, which keeps $d_e = 16$ and unit-norm rows and still defines both the diffusion target and the output head. $\Emb_\mathrm{p}$ enters at the transformer input only. Writing $\mathbf{u} \in \R^{L \times h}$ for the sequence fed to the first transformer block, and using the one-hot convention $\x^l \Emb_\mathrm{p} \in \R^{h}$ of \cref{sec:background} to index the table,
\begin{align*}
    \mathbf{u}^l =
    \begin{cases}
        \x^l \Emb_\mathrm{p}, & \text{$l$ a prompt position},\\[2pt]
        \mathrm{Proj}(\z_t^l) + \hat\e_\sc^l W_\sc\ + \mathbf{g}(\gamma_t), & \text{$l$ an answer position},
    \end{cases}
\end{align*}
where $\mathrm{Proj}: \R^{d_e} \to \R^{h}$ is RePlaid's input up-projection, $\mathbf{g}(\gamma_t) \in \R^{h}$ the additive Fourier time embedding, and $W_\sc \in \R^{d_e \times h}$ the zero-initialized self-conditioning projection applied to $\hat\e_\sc$ (which is $\zero$ when self-conditioning is off). That is, prompt positions are \emph{overwritten} rather than noised, and answer positions are untouched by $\Emb_\mathrm{p}$. Overwriting after the additive time embedding means prompt tokens carry no time signal on the input side; since our setting also conditions the DiT blocks on a zero time embedding (MDLM-style AdaLN), no time signal reaches them inside the blocks either, so the prompt representation is exactly the same at every $t$.

For optimization, $\Emb_\mathrm{p}$ is treated as an ordinary backbone parameter rather than as one of the Plaid-specific components: it uses the backbone's peak LR $3\ee{-}4$ with the same constant-with-warmup schedule, no weight decay, the same Adam moments, and it participates in gradient clipping and in the $0.9999$ EMA. In particular it does \emph{not} receive the $1\ee{-}2$ LR applied to $\Emb$, the monotone MLP, and $(\gamma_0, \gamma_1)$. Because prompt positions are excluded from all three NELBO terms, $\Emb_\mathrm{p}$ receives gradient only indirectly, through attention from the answer positions. The table adds $V h$ parameters. 

\paragraph{Baselines.} MDLM, Duo, and FLM use the official checkpoints of~\citet{deschenaux2026hyperspherical}, which were trained on TinyGSM under exactly the setup above. The AR numbers are likewise taken from their released model.

\paragraph{Sampling hyperparameters.} Diffusion uses $T = 1024$ steps. Logit and score temperature are both swept over $\tau \in \{0.0001, 0.001, 0.005, 0.01, 0.05, 0.1, 0.2, \ldots, 0.9, 1.0\}$; the confidence threshold is swept over $\{0.8, 0.9\}$. 

\begin{table}[h]
  \centering
  \small
  \caption{\small \textbf{Mathematical reasoning.} GSM8K pass@1; best $\tau$ in parentheses. Logit temperature scales the logits of $\x_\theta(\z_t, t)$ by $\tau$; score temperature scales the diffusion score by $\tau$~\citep{dieleman2022continuous}. ``N.A.'' means \emph{not applicable}: the score temperature requires a Gaussian diffusion score, and the adaptive-order samplers (confidence threshold~\citep{wu2025fast} and top margin~\citep{kim2025train}) requires masked diffusion.}
  \vspace{1em}
  \label{tab:tinygsm}
  \scriptsize
  \begin{tabular}{llcccccc}
    \toprule
    & & & \multicolumn{3}{c}{Ancestral Sampler ($T=1024$)} & \multicolumn{2}{c}{Adaptive-Order Samplers} \\
    \cmidrule(lr){4-6} \cmidrule(lr){7-8}
    & Method & s.c. & No Temp. & Best $\tau_\text{logit}$ & Best $\tau_\text{score}$ & Conf. Thres. & Top Margin \\
    \midrule
    \multicolumn{3}{l}{\textit{Autoregressive}} \\
    & Transformer & & $52.9$ & $64.4$ ($0.2$) & N.A. & N.A. & N.A. \\
    \midrule
    \multicolumn{3}{l}{\textit{Discrete Diffusion}} \\
    & MDLM~\citep{sahoo2024simple} & & $17.0$ & $35.6$ ($0.1$) & N.A. & $37.0$ ($0.8$) & $40.3$  \\
    & Duo~\citep{sahoo2025the} & & $18.6$ & $37.2$ ($0.1$) & N.A. & N.A. & N.A. \\
    \midrule
    \multicolumn{3}{l}{\textit{Continuous Diffusion}} \\
    & FLM~\citep{lee2026flow} & & $0.3$ & -- & -- & N.A. & N.A. \\
    & \textbf{\textsc{RePlaid} (Ours)}, DDPM & \ccross & $17.1$ & $21.8$ ($0.3$) & \topthree{$42.5$} ($0.01$) & N.A. & N.A. \\
    & \textbf{\textsc{RePlaid} (Ours)}, DDIM & \ccross & $13.0$ & $17.9$ ($0.01$) & $41.6$ ($0.005$) & N.A. & N.A. \\
    & \textbf{\textsc{RePlaid} (Ours)}, DDPM & \cmark & $39.1$ & $39.6$ ($0.0001$) & \topone{$\mathbf{49.1}$} ($0.2$) & N.A. & N.A. \\
    & \textbf{\textsc{RePlaid} (Ours)}, DDIM & \cmark & $30.9$ & $34.1$ ($0.5$) & \toptwo{$48.2$} ($0.01$) & N.A. & N.A. \\
    \bottomrule
  \end{tabular}
  \vspace{-0.5em}
\end{table}


\clearpage

\section{Computing Dominant POS Tags per GPT-2 Subword}
\label{supp:sec:spacy}

[Back to~\cref{subsec:geometric_reg} in the main paper]

To analyze embedding geometry conditioned on syntactic role, we precompute a dominant
Universal POS tag for each GPT-2 subword token over a reference corpus. The
challenge is that GPT-2 BPE splits words into subwords (e.g., \texttt{running}
$\rightarrow$ \texttt{\textvisiblespace run}, \texttt{ning}), whereas POS tags
are defined on whole words. We resolve this by aligning subwords to spaCy~\citep{honnibal2020spacy} word
spans via character offsets.

\subsection{Pipeline}

Given a corpus of token-id rows of length $L=1024$, we proceed as follows:

\begin{enumerate}
    \item \textbf{Decode.} Convert each row of token ids back to text using the
    GPT-2 tokenizer.
    \item \textbf{Tag.} Run spaCy's POS tagger on the decoded text to obtain
    word-level tags and their character offsets.
    \item \textbf{Re-encode.} Re-tokenize the decoded text with the fast GPT-2
    tokenizer to recover \texttt{(token\_id, char\_span)} pairs for each
    subword. We use the re-encoded ids rather than the original ids because
    chunk boundaries in the dataset may truncate multi-byte UTF-8 sequences;
    re-encoding guarantees that subword spans align with the actual text spaCy
    saw.
    \item \textbf{Align.} For each subword, inherit the POS tag of the spaCy
    word with which its character span maximally overlaps.
    \item \textbf{Tally.} Accumulate a count matrix
    $C \in \mathbb{Z}^{V \times P}$ where $V$ is the vocabulary size and
    $P{=}18$ tags (17 UPOS tags plus a \textsc{Space} bucket). The dominant
    tag for token $v$ is $\argmax_p C_{v,p}$.
\end{enumerate}

\subsection{Worked Example}

Consider the sentence \texttt{"The cat is purring loudly."} (26 characters).
spaCy produces the following word-level analysis (\texttt{idx} is the
character offset of the word's first character):

\begin{center}
\begin{tabular}{lcl}
\toprule
word & \texttt{idx} & POS \\
\midrule
\texttt{The}     & 0  & \textsc{Det}   \\
\texttt{cat}     & 4  & \textsc{Noun}  \\
\texttt{is}      & 8  & \textsc{Aux}   \\
\texttt{purring} & 11 & \textsc{Verb}  \\
\texttt{loudly}  & 19 & \textsc{Adv}   \\
\texttt{.}       & 25 & \textsc{Punct} \\
\bottomrule
\end{tabular}
\end{center}

The fast GPT-2 tokenizer returns subwords with character spans:

\begin{center}
\begin{tabular}{rll}
\toprule
id & subword & \texttt{char\_span} \\
\midrule
464   & \texttt{The}                       & $(0,3)$   \\
3797  & \texttt{\textvisiblespace cat}     & $(3,7)$   \\
318   & \texttt{\textvisiblespace is}      & $(7,10)$  \\
1308  & \texttt{\textvisiblespace pur}     & $(10,14)$ \\
1806  & \texttt{ring}                      & $(14,18)$ \\
23112 & \texttt{\textvisiblespace loudly}  & $(18,25)$ \\
13    & \texttt{.}                         & $(25,26)$ \\
\bottomrule
\end{tabular}
\end{center}

Note that the verb \texttt{purring} splits into two subwords
(\texttt{\textvisiblespace pur}, \texttt{ring}); the alignment procedure must
attribute both to the same spaCy word.

\paragraph{Alignment via a character-to-word map.} We allocate an integer
array $\mathtt{char\_to\_spacy}$ of length $|text|$ and fill index $j$ at
every character covered by spaCy word $j$, leaving whitespace gaps as $-1$:

\begin{center}
\resizebox{\linewidth}{!}{%
\begin{tabular}{r|cccccccccccccccccccccccccc}
char idx     & 0 & 1 & 2 & 3 & 4 & 5 & 6 & 7 & 8 & 9 & 10 & 11 & 12 & 13 & 14 & 15 & 16 & 17 & 18 & 19 & 20 & 21 & 22 & 23 & 24 & 25 \\
char         & T & h & e & \textvisiblespace & c & a & t & \textvisiblespace & i & s & \textvisiblespace & p & u & r & r & i & n & g & \textvisiblespace & l & o & u & d & l & y & . \\
\midrule
spacy idx    & 0 & 0 & 0 & $-1$ & 1 & 1 & 1 & $-1$ & 2 & 2 & $-1$ & 3 & 3 & 3 & 3 & 3 & 3 & 3 & $-1$ & 4 & 4 & 4 & 4 & 4 & 4 & 5 \\
\end{tabular}%
}
\end{center}

For each subword we slice $\mathtt{char\_to\_spacy}$ by its span, drop $-1$
entries, and take the majority spaCy index:

\begin{itemize}
    \item \texttt{\textvisiblespace pur} $(10,14)$:
        slice $= [-1, 3, 3, 3]$ $\Rightarrow$ word $3$ (\texttt{purring})
        $\Rightarrow$ \textsc{Verb}.
    \item \texttt{ring} $(14,18)$: slice $= [3, 3, 3, 3]$ $\Rightarrow$
        \textsc{Verb}.
    \item \texttt{\textvisiblespace cat} $(3,7)$: slice $= [-1, 1, 1, 1]$
        $\Rightarrow$ \textsc{Noun}.
\end{itemize}

Both subwords of \texttt{purring} are credited as \textsc{Verb} for this row.
Subwords whose span is entirely whitespace (all $-1$) are bucketed into a
special \textsc{Space} tag rather than discarded.

\subsection{Aggregation}

Across the corpus we accumulate
$C_{v,p} \mathrel{+}= 1$ for every (subword id $v$, inherited tag $p$) pair.
The dominant POS for token $v$ is
\[
\mathrm{dominant\_pos}(v) =
\begin{cases}
\argmax_p C_{v,p} & \text{if } \sum_p C_{v,p} > 0, \\
-1 & \text{otherwise (token never observed).}
\end{cases}
\]
We additionally retain the full count matrix $C$, enabling downstream analyses
such as POS entropy per token (a measure of syntactic ambiguity) in addition to
the hard argmax assignment.

\textbf{Implementation notes.} We use
\texttt{en\_core\_web\_sm} with the parser, NER, and lemmatizer disabled, and
raise \texttt{nlp.max\_length} to $10^6$ to accommodate long decoded chunks
(a 1024-token row decodes to roughly 4--6k characters).

\clearpage

\section{PCA Scree for $d_e=768$}

[Back to \cref{subsec:geometric_reg} in the main paper]

Here, RePlaid (s.c.) ($d_e=768$) uses the same settings as RePlaid (s.c.) ($d_e=16$) except for $d_e$.

\begin{figure*}[h]
    \centering
    \begin{subfigure}{0.45\linewidth}
        \centering
        \includegraphics[width=\linewidth]{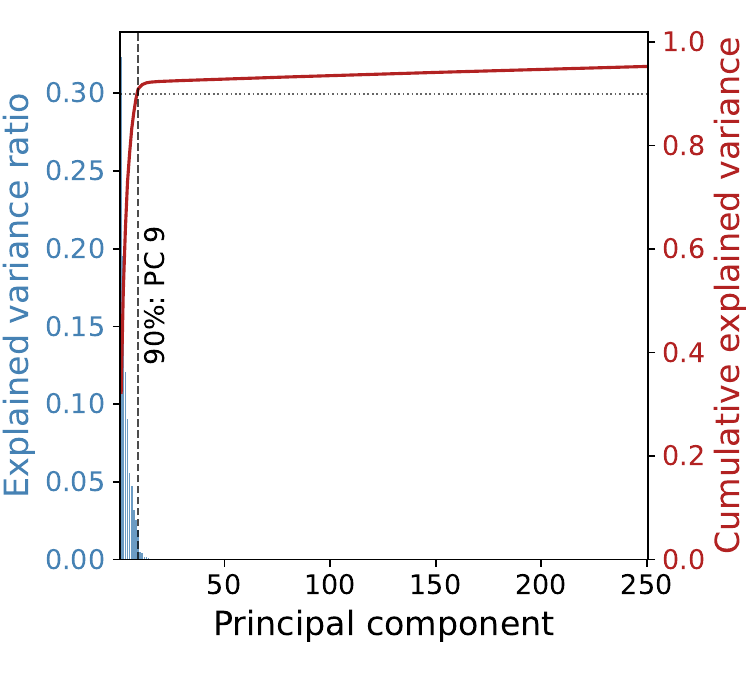}
        \caption{\centering\textsc{RePlaid} (s.c.) with $d_e=768$\\(\textbf{NELBO OWT PPL}: $26.1$ at 500K)}
        \label{fig:pca-scree-768-replaid}
    \end{subfigure}
    \begin{subfigure}{0.45\linewidth}
        \centering
        \includegraphics[width=\linewidth]{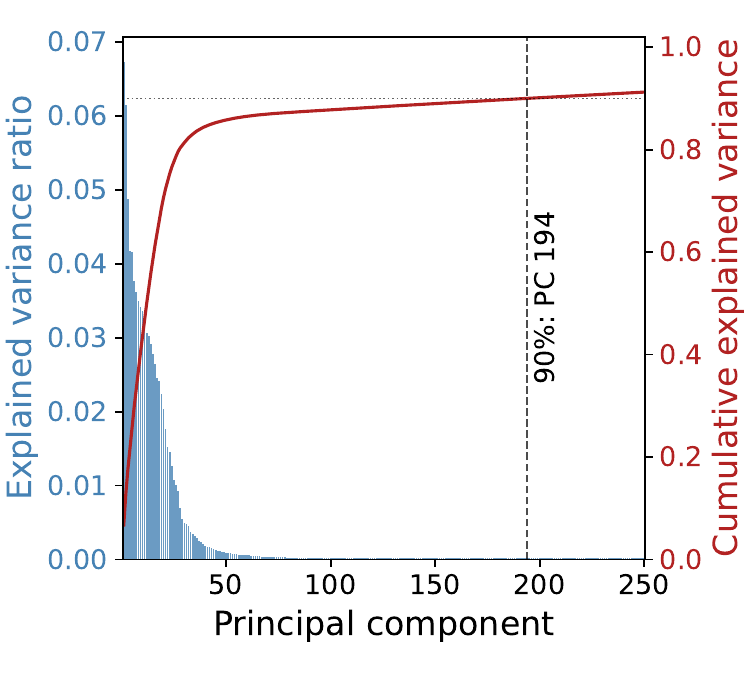}
        \caption{\centering LangFlow ($d_e=768$ by default)\\(\textbf{NELBO OWT PPL}: $32.2$ at 1M)}
        \label{fig:pca-scree-768-langflow}
    \end{subfigure}
    \caption{\small Comparing PCA scree plots for \textsc{RePlaid} (s.c.) and LangFlow, both using self-conditioning, $d_e=768$, and length-normalized embeddings. \textsc{RePlaid} (s.c.) yields a lower-rank embedding geometry while achieving a better PPL bound.}
    \label{fig:pca-scree-768}
    \vspace{-0.5em}
\end{figure*}

\begin{figure*}[h]
    \centering
    \includegraphics[width=0.45\linewidth]{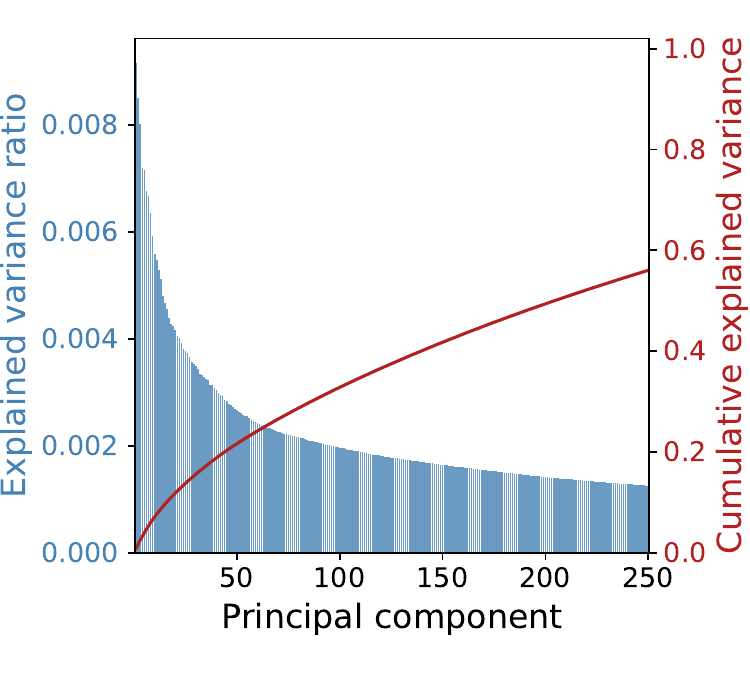}
    \caption{\small PCA scree plot for MDLM (low var.) ($d_e=768$ by default) trained on OWT for 1M steps.}
    \label{fig:pca-scree-768-mdlm}
    \vspace{-0.5em}
\end{figure*}

\clearpage

\section{Near-Linear Per-Timestep Cross-Entropy Loss for Diverse Embeddings}\label{sec:linear-ce-diverse}

[Back to~\cref{thm:linear_ce} of the main paper]

To faciliate a meaningful comparison with pretrained embeddings, we change the default tokenizer for OWT, GPT-2, and use \texttt{bert-base-cased} instead. Below we report the validation PPLs (250K steps) reached by RePlaid (s.c.) using different embeddings (learnable v.s. fixed, small v.s. large $d_e$): 

\begin{itemize}
    \item $d_e=16$ (learn embeddings): $25.9$;
    \item $d_e=16$ (freeze embeddings to be randomly initialized): $46.5$;
    \item $d_e=768$ (learn embeddings): $28.4$;
    \item $d_e=768$ (freeze embeddings to be randomly initialized): $69.8$;
    \item $d_e=768$ (use frozen BERT embeddings): $52.7$;
    \item $d_e=768$ (use frozen BERT embeddings -- normalized by Euclidean length): $50.2$;
    \item $d_e=768$ (learn embeddings, initialized from normalized BERT embeddings): $28.7$;
    \item $d_e=V$ (use frozen one-hot embeddings -- inspired by~\citep{lee2026flow, potaptchik2026discrete, pynadath2025candi}): $64.9$.
\end{itemize}


We see that learning the noise schedule leads to a near-linear per-timestep cross-entropy loss regardless of the embedding geometry (\cref{fig:ce-vs-t-ablation}).

\begin{figure}[H]
  \centering
  \includegraphics[width=0.5\linewidth]{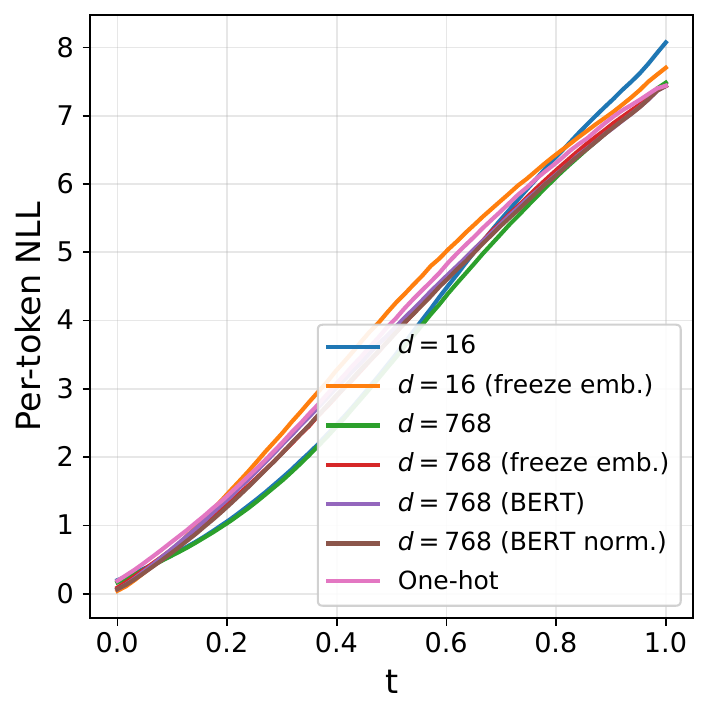}
  \caption{\small Learning the noise schedule leads to a near-linear per-timestep cross-entropy loss regardless of the embedding geometry.}
  \label{fig:ce-vs-t-ablation}
\end{figure}

\clearpage

\section{Do RePlaid's Components Transfer to Discrete DLMs?}
\label{app:component-transfer}

[Back to~\cref{subsec:replaid-align} in the main paper]

RePlaid combines several ingredients (learnable low-dimensional embeddings, a learnable noise schedule, self-conditioning, an output prior), and it is natural to ask which of them are intrinsic to continuous diffusion and which could be ported to masked or uniform diffusion. 

\paragraph{Self-conditioning.} This transfers, and was never exclusive to continuous DLMs: it predates RePlaid as a component of the original Plaid, and LDDM~\citep{jo2026loopholing} demonstrates that it also benefits masked and uniform diffusion at a comparable 0.1B setup. The relevant point for our claims is that RePlaid stays competitive \textit{in both directions}: without self-conditioning, RePlaid reaches $\replaidnscppl$ OWT PPL versus MDLM's $23.1$ and UDLM's $27.4$ (\cref{tab:val-ppl}); with self-conditioning, RePlaid reaches $\replaidscppl$, again sitting between LDDM's masked and uniform variants. Our conclusions therefore do not rest on self-conditioning being unavailable to discrete models.

\paragraph{Learnable noise schedule.} This splits into two pieces (\cref{sec:background}). The learnable \textit{endpoints} $\gamma_0 < \gamma_1$ are specific to continuous DLMs: they exist because the model has an encoder--decoder-like structure bridging a continuous embedding space and a discrete token space, and there is no analogue in a purely discrete corruption process. The learnable \textit{interior shape} $\tilde\gamma(t)$ is not specific in principle. Discrete DLMs also care about this quantity. For example, MDLM selected its default log-linear schedule from several candidates precisely to minimize the variance of its ELBO~\citep[Tab.~9]{sahoo2024simple}. But, to our knowledge, no discrete DLM \emph{learns} the schedule shape for that purpose, though nothing rules it out. We also note that for RePlaid, learning the schedule shape improves PPL but is not what makes it competitive: removing it costs $2.3$ PPL and still leaves RePlaid ahead of Duo (\cref{tab:replaid-ablations}).


\paragraph{Samplers.} Finally, we note that the continuous samplers used in this work (DDPM, DDIM, DPM-Solver++(2M), Heun) are standard numerical solvers for continuous diffusion and do not exploit entropy- or confidence-based token selection. 


\clearpage


\section{Additional Embedding Ablations}
\label{app:embed-ablations}

[Back to~\cref{subsec:geometric_reg} in the main paper]

\subsection{Embedding Dimension Sweep}
\label{app:de-sweep}

RePlaid inherits $d_e = 16$ from Plaid, which is far smaller than the $d_e = 768$ typical of discrete DLMs. To test whether this choice is merely inherited or actually favorable, we train RePlaid with $d_e \in \{4, 8, 32, 64\}$ for $250$K steps on OWT under the default GPT-2 tokenizer and compare against $d_e = 16$ and $d_e = 768$.

\begin{center}
\small
\begin{tabular}{cc}
\toprule
Embedding dimension $d_e$ & OWT PPL ($\downarrow$; 250K steps) \\
\midrule
$4$ & $30.4$ (worst) \\
$8$ & $25.8$ \\
$16$ & $24.9$ \\
$32$ & $24.8$ (best) \\
$64$ & $25.9$ \\
$768$ & $28.0$ \\
\bottomrule
\end{tabular}
\end{center}

A lower $d_e$ is \emph{not} monotonically better: performance degrades in both directions away from a flat basin at $d_e \in \{16, 32\}$, whose two entries differ by $0.1$ PPL ($0.005$ nats) and are therefore indistinguishable at this scale --- we report $d_e = 16$ throughout the paper and do not read the sweep as favoring $32$ over it. Outside that basin the degradation is clear and roughly symmetric in $\log d_e$: halving to $8$ costs $0.9$ PPL and doubling to $64$ costs $1.0$, while the extremes cost $5.5$ ($d_e = 4$) and $3.1$ ($d_e = 768$). This is consistent with the learned geometry reported in~\cref{subsec:geometric_reg}: for RePlaid at $d_e = 16$, $7$ principal components explain $90\%$ of the embedding variance and $12$ explain $99\%$ (\cref{fig:pca}), so the embeddings occupy a genuinely low-dimensional subspace. Below that intrinsic dimension ($d_e = 4, 8$) the embedding table lacks the capacity to represent it; well above it ($d_e = 64, 768$) the extra directions add parameters without adding usable structure. Notably $d_e = 768$ --- the default for discrete DLMs --- is worse than $d_e = 8$, so the gain is not a capacity effect. The optimal $d_e$ should therefore be read as a property of the learned geometry rather than a universal constant.

\subsection{Auxiliary Cross-Entropy Weight Sweep}
\label{app:ce-sweep}

\cref{subsec:geometric_reg} reports a single auxiliary CE weight ($w = 1$). To check that the conclusion is not an artifact of that choice, we sweep $w$ and train each ablation from scratch for $250$K steps on OWT, recording both the PPL bound and the rank of the resulting embedding geometry (number of principal components needed to explain $90\%$ and $99\%$ of the variance of $\Emb$).

Ranks are obtained by PCA on the row-normalized $\Emb$. The $w = 0$ ranks here ($8$ and $14$) exceed the $7$ and $12$ of~\cref{fig:pca}, which shows the same configuration at $1$M steps: the embedding continues to consolidate onto a lower-dimensional subspace over training.

\begin{center}
\small
\begin{tabular}{cccc}
\toprule
CE weight $w$ & OWT PPL ($\downarrow$; 250K steps) & PCs for $90\%$ variance & PCs for $99\%$ variance \\
\midrule
$0$ (NELBO only) & $24.9$ (best) & $8$ & $14$ \\
$0.1$ & $25.6$ & $11$ & $15$ \\
$1$ & $27.0$ & $13$ & $16$ \\
$5$ & $30.6$ & $14$ & $16$ \\
$10$ & $31.9$ (worst) & $14$ & $16$ \\
\bottomrule
\end{tabular}
\end{center}

The conclusion of~\cref{subsec:geometric_reg} holds across the sweep and is monotone in $w$: as the CE weight increases, the PPL bound degrades monotonically and the embedding geometry flattens monotonically, with the number of components needed for $90\%$ variance rising from $8$ to $14$. The two move together, which is why we read low-rankness as a characteristic of the likelihood-optimized solution rather than as an independent cause of the likelihood gain.


\clearpage

\section{Text Samples Generated by RePlaid Trained on OWT}\label{app:samples}

[Back to~\cref{subsec:sampling} in the main paper]

We present un-cherry-picked, unconditional samples of length $L = 1024$ generated by RePlaid (s.c.) trained for 1M steps on OpenWebText.
Document boundaries (\textcolor{red}{\uline{<|endoftext|>}}) are highlighted; entities and concepts that the model tracks coherently across multiple sentences and paragraphs are highlighted in \textcolor{blue}{\textbf{blue bold}}. The samples demonstrate RePlaid's strong long-range topical consistency. Note that newline characters have been replaced by an empty line for readability.

\paragraph{Warning} Please be advised that the following un-cherry-picked text samples contain unmoderated content, including fictionalized descriptions of violence and sensitive political references.

\begin{samplebox}[title={\textsc{RePlaid} (s.c.), ancestral sampler, steps $T = 64$ (GenPPL: $48.5$, Entropy: $5.50$, MAUVE: $0.88$)}]
\vspace{1em}

\textcolor{red}{\uline{<|endoftext|>}} The plot focused on moving over pirate territory, but this was not for social support. The real objective was creating a PCCC. At this point, Kris mentions that players are not interested in pirates because they haven't actually seen this dangerous thing.

\vspace{1em}
Speaking of 5th character: Game Lead, I sadly did not get a chance to play like the princess. The art is available here.\textcolor{red}{\uline{<|endoftext|>}}Godcess Theft Auto 4 owners can gear up for the 7th version of the \textcolor{blue}{\textbf{PlayStation Matchbox}} tutorial series, which premiered again this year at the Official Site!

\vspace{1em}
For those of us veterans looking to get started, take a look at the unique \textcolor{blue}{\textbf{PlayStation Basic Packs}} that fall on you every time you get the holiday season underway. Find your Credit as well as bring your \textcolor{blue}{\textbf{PlayStation 4}} character to life and move into the attic.

\vspace{1em}
Spending the first 20 minutes on the two \textcolor{blue}{\textbf{Basic Packs}} makes get in adulthood a lot of good. And shudder at the thought of having one in to your sidebox this holiday season --- you can check out some of the amazing minimalistic \textcolor{blue}{\textbf{PlayStation Matchbox}} tutorials below:

\vspace{1em}
-Xbox Matchbox Repair for Super Boost Hunt \& Sparkle Strike

\vspace{1em}
-Theme Arcade Game: Everything You'll Need to Know to Get Started

\vspace{1em}
-References and Gags from the Holiday Past

\vspace{1em}
-When Mother Tanks Were Mumbles Into Enhancements and Super Stars

\vspace{1em}
-There's a Random \textcolor{blue}{\textbf{PlayStation}} Skinny Hunt\textcolor{red}{\uline{<|endoftext|>}}\textcolor{blue}{\textbf{Noogame-UPS}}: New Collaborations and Standouts to Share Preditional Center

\vspace{1em}
After announcing April 28, May 30 and June 20 finals at the low point of the academic year, this is the first time a high school graduate from the University of Michigan and faculty is welcome to participate in \textcolor{blue}{\textbf{Noogame}} on campus --- and person.

\vspace{1em}
``This collaboration recently reached a new rotation on the University of Michigan board,''* Scott praised Doug, Vice President of Engineering at \textcolor{blue}{\textbf{Noogame}}. ``So I'm excited for the unique opportunity to get such a young student to be part of a collaboration and then see the group focused on opportunities including advancing their scholarly careers.

\vspace{1em}
I'm envy to them their long homes, large family, tremendous work ethic, strong academic talents and all of the journeys they went through. I am looking forward to hearing from them and seeing their generation crystal clear. Together we all become lifelong friends and dedicated GOET Brothers --- collaborating along every stage of their GOET-related careers.''

\vspace{1em}
To support this collaboration, \textcolor{blue}{\textbf{Noogame-UPS}} Grants are available to be purchased online. Customer service and lodging are available year round.

\vspace{1em}
The \textcolor{blue}{\textbf{Noogame-UPS}} collaboration is the evolution of the Continuous Enterprise Initiative (SPA). \textcolor{blue}{\textbf{Noogame}} was formed in summer 2008 to recruit/create 24,000 faculty positions at the University of Michigan. GH is an organization established to encourage, practice, study and practice innovative designs and creative techniques. This three stage organization, based in Ann Arbor, Michigan, receives Federal Center for Design and Service Technology (CKeSAP) attention (Fall New Tier) and James' Alfred and Gordon Hastleditch Book Award.

\vspace{1em}
``While the \textcolor{blue}{\textbf{Noogame-UPS}} collaboration has garnered enormous support across MSU, other forms of collaborative action have taken years to gain momentum,'' said Mike Cosstead, GOET Associate Staff Pierester for Strategy \& GOET News. ``\textcolor{blue}{\textbf{Noogame}}'s successful collaborative design action attracts exceptional engineering talent and motivates FTendite's work, but right now, this is compromising student morale and creating a personal impact among members.''

\vspace{1em}
Members of the \textcolor{blue}{\textbf{Noogame-UPS}} Promotion are required to register by Feb. 30. For real-time work on the internship programs, and advanced analysis, participants will be expected to spend 15 minutes per meeting.

\vspace{1em}
What is \textcolor{blue}{\textbf{Noogame-UPS}}?

\vspace{1em}
This new GOET-sponsored initiative wants readers to know that the new \textcolor{blue}{\textbf{Noogame}} Zone 7 (translated as ``FTendite-go'') is August 7 in Ohio. The campuses of Zone 7 will cover Cincinnati and live in three cities; Boston, New Orleans and New York. Five campus communities in Ohio are also currently in construction --- Columbus, Cincinnati, Lansing, Baton Rouge, Grand Rapids, New Orleans and San Antonio.

\vspace{1em}
Two engineering projects, CUOC (Ohio Institute of Technology) and \textcolor{blue}{\textbf{Noogame Workshop}} are at the forefront of the planning of Zone 7. The CUOC Beta project is seeking to increase the usability and accessibility of complex physical projects and allow students to construct underutilized environments to maximize resiliency and safety. The project, which started with America's best-known Institute of Engineering, is working to replicate the aesthetics, design, feeling and experience of compact environments for high engineering research.

\vspace{1em}
Our chair FTendite Director Raul Steele and two Faculty Adolescors are currently writing a highly anticipated book for our high school programs ``Multiscreen Disciples: New America'' will argue for ways to push the boundaries of\textcolor{red}{\uline{<|endoftext|>}}
\end{samplebox}

\begin{samplebox}[title={\textsc{RePlaid} (s.c.), ancestral sampler, $T = 1024$ (GenPPL: $18.3$, Entropy: $5.21$, MAUVE: $0.85$)}]
\vspace{1em}

\textcolor{red}{\uline{<|endoftext|>}} not identified the location where the suspect went out with friends.

\vspace{1em}
Police did say the shooting happened on East Mickleman Avenue and Port Washington Boulevard.

\vspace{1em}
An image from the \textcolor{blue}{\textbf{Target store}} does not show the exact location where the victim was shot.

\vspace{1em}
The owner of the sale, Coach Properties, brought authorities to the \textcolor{blue}{\textbf{store}} because they are so close to shoppers.

\vspace{1em}
Police said they didn't have a security office to investigate at the time of the incident, but the nearest \textcolor{blue}{\textbf{store}} released a statement saying Saturday they will still open their doors.

\vspace{1em}
``We still express our condolences to the individual that lost life and are decamping for his recovery, with news breaking since the beginning of this week that his \textcolor{blue}{\textbf{store}} will indeed open their doors,'' Sgt. Juan Pecanelli said.

\vspace{1em}
Ariel Dudind, one of the company co-founders, said \textcolor{blue}{\textbf{Target stores}} have been growing fast since the shooting in the neighborhood.

\vspace{1em}
``Our \textcolor{blue}{\textbf{Target store}} has generated large margins that have caused significant inconvenience to our employees and customers,'' Dudind said.

\vspace{1em}
Witnesses leaving the scene told ProPublica the gunman may have been aiming at another person, and did turn out to be a 28-year-old man from Brooklyn, about two-blocks from the \textcolor{blue}{\textbf{store}}. Police say Aaron Gill, a 24-year old male, was also thought to be in the \textcolor{blue}{\textbf{store}}, and the gunman is still on the loose.

\vspace{1em}
Aaron Gill, an 18-year-old fourth-grade student at Wood Woods Middle School, arrived at the Port Washington \textcolor{blue}{\textbf{store}} at 8:30 a.m. Friday with friends, then left the room on the 3rd floor. Gill was gunned down on West Smith Place after he told police a friend.

\vspace{1em}
Saturday, a police helicopter arrived and a chase ensued. Police recovered what was described as a U.S. Unambortunity sweatshirt, with a jockine symbol on the shirt. The backpack was opened and the casing was searched and used in heists.

\vspace{1em}
Some of those not who called police:

\vspace{1em}
Officer Don Parker, a 25-year-old local resident who came over to the \textcolor{blue}{\textbf{store}} seeking deliveries, told police he was concerned.

\vspace{1em}
Officer Mark Clements, a 45-year-old neighbor who came to the area with friends and family, told police several men inside the building holding a weapon. Clements then saw a man still armed.

\vspace{1em}
``The associates knew what you were going to see, including myself,'' Pecanelli said. He said he saw three men with automatic rifles.

\vspace{1em}
A spokesman for Sgt. Juan Pecanelli declined to comment, saying officers participated in ``skeptist celebrations'' in the neighborhood.

\vspace{1em}
District Attorney David Souter said in a statement that Southridge-Euranto was ``deeply informed'' of the incident.

\vspace{1em}
``This incident is an incredible concern to law enforcement officers, and it's a chilling experience for everyone.'' Souter said.\textcolor{red}{\uline{<|endoftext|>}}Bob Corker Richard (Bob) Wayne CorkerThe U.S. Senate says Ohio can step up shutdown threat over blocked abortion referrals Trump campaign to push back with `just say yes' to 2020 bid MORE (R-Tenn.) on Monday announced changes to the first-come-compete \textcolor{blue}{\textbf{ObamaCare}} law that would provide members of Congress with up to \$1.1 billion to buy \textcolor{blue}{\textbf{health insurance}}.

\vspace{1em}
Americans covered in the new \textcolor{blue}{\textbf{plan}} would receive an average \$3.3 million in subsidies for each prescription covered per year. The subsidies make coverage far more affordable than \textcolor{blue}{\textbf{ObamaCare}}'s employer-sponsored \textcolor{blue}{\textbf{insurance}} expansion, which requires employers to charge only over the average of a family \$1,800 per year.

\vspace{1em}
Supporters of Corker's proposal swiftly angered Democrats by objecting to the health \textcolor{blue}{\textbf{plan}} as being purely cost-\textcolor{blue}{\textbf{insurance}}, where the government indirectly retains revenues from premiums.

\vspace{1em}
``This is a single window approach,'' said John Snook, an analyst with Prudential Advisor, a Virginia-based brokerage. ``That's exactly what the Republican Congress wants.''

\vspace{1em}
ADVERTISEMENT

\vspace{1em}
The news comes just days after Congress repeals President Trump's revised version of \textcolor{blue}{\textbf{ObamaCare}}.

\vspace{1em}
Views have traditionally tilt toward revamping \textcolor{blue}{\textbf{ObamaCare}}, but in light of Corker's announcement, the political line has shifted significantly.

\vspace{1em}
Both Democrats and Republicans are divided over where the \textcolor{blue}{\textbf{plan}} depurs on the precedent-setting new federal health cost-\textcolor{blue}{\textbf{insurance}}, analysts said.

\vspace{1em}
``This is still viewed as the tiniest way to do it,'' said Snook, an economist at MIT's Tableau Institute, which tracks \textcolor{blue}{\textbf{health insurance}} options.

\vspace{1em}
Experts have estimated the Trump-McGrigley arrangement, which allows people to cover only medical costs per year, is 10 billion to trillion dollars (\$38 billion) less than employer-sponsored \textcolor{blue}{\textbf{insurance}}\textcolor{red}{\uline{<|endoftext|>}}
\end{samplebox}

\clearpage

\section*{NeurIPS Paper Checklist}

\begin{enumerate}

\item {\bf Claims}
    \item[] Question: Do the main claims made in the abstract and introduction accurately reflect the paper's contributions and scope?
    \item[] Answer: \answerYes{} 
    \item[] Justification: We believe the main claims made in the abstract and introduction accurately reflect the paper's contributions and scope. 

\item {\bf Limitations}
    \item[] Question: Does the paper discuss the limitations of the work performed by the authors?
    \item[] Answer: \answerYes{} 
    \item[] Justification: We discuss the limitations of this work in \cref{app:discussion}. 

\item {\bf Theory assumptions and proofs}
    \item[] Question: For each theoretical result, does the paper provide the full set of assumptions and a complete (and correct) proof?
    \item[] Answer: \answerYes{} 
    \item[] Justification: We provide full sets of assumptions and complete correct proofs for each theoretical result. We have checked all proofs and confirmed their correctness, to the best of our knowledge. 

    \item {\bf Experimental result reproducibility}
    \item[] Question: Does the paper fully disclose all the information needed to reproduce the main experimental results of the paper to the extent that it affects the main claims and/or conclusions of the paper (regardless of whether the code and data are provided or not)?
    \item[] Answer: \answerYes{} 
    \item[] Justification: We have made detailed guidelines on how we obtained the main experimental results in several of the appendices, such as \cref{supp:sec:align-table,supp:sec:opt,supp:sec:transformer}. 

\item {\bf Open access to data and code}
    \item[] Question: Does the paper provide open access to the data and code, with sufficient instructions to faithfully reproduce the main experimental results, as described in supplemental material?
    \item[] Answer: \answerNo{} 
    \item[] Justification: We plan to release the code after the acceptance of the work. 

\item {\bf Experimental setting/details}
    \item[] Question: Does the paper specify all the training and test details (e.g., data splits, hyperparameters, how they were chosen, type of optimizer) necessary to understand the results?
    \item[] Answer: \answerYes{} 
    \item[] Justification: All training and test details can be found in several of the appendices, such as \cref{supp:sec:align-table,supp:sec:opt,supp:sec:transformer}. 

\item {\bf Experiment statistical significance}
    \item[] Question: Does the paper report error bars suitably and correctly defined or other appropriate information about the statistical significance of the experiments?
    \item[] Answer: \answerNo{} 
    \item[] Justification: Due to the huge computational cost of training, we do not provide error bars for most of the experimental results. 

\item {\bf Experiments compute resources}
    \item[] Question: For each experiment, does the paper provide sufficient information on the computer resources (type of compute workers, memory, time of execution) needed to reproduce the experiments?
    \item[] Answer: \answerYes{} 
    \item[] Justification: We indicate the type of compute workers and the duration of individual compute jobs. Total compute can be derived from individual job compute.
    
\item {\bf Code of ethics}
    \item[] Question: Does the research conducted in the paper conform, in every respect, with the NeurIPS Code of Ethics \url{https://neurips.cc/public/EthicsGuidelines}?
    \item[] Answer: \answerYes{} 
    \item[] Justification: We confirm that the research conducted in the paper conform, in every respect, with the NeurIPS Code of Ethics. 

\item {\bf Broader impacts}
    \item[] Question: Does the paper discuss both potential positive societal impacts and negative societal impacts of the work performed?
    \item[] Answer: \answerYes{} 
    \item[] Justification: The potential negative societal impacts of this work, such as misuse for disinformation, are inherited from text generation broadly and are not qualitatively new at this scale. 
    
\item {\bf Safeguards}
    \item[] Question: Does the paper describe safeguards that have been put in place for responsible release of data or models that have a high risk for misuse (e.g., pre-trained language models, image generators, or scraped datasets)?
    \item[] Answer: \answerNA{} 
    \item[] Justification: Our model does not introduce capabilities or risks beyond those of existing open small-scale language models. No additional safeguards are warranted. 

\item {\bf Licenses for existing assets}
    \item[] Question: Are the creators or original owners of assets (e.g., code, data, models), used in the paper, properly credited and are the license and terms of use explicitly mentioned and properly respected?
    \item[] Answer: \answerYes{} 
    \item[] Justification: All datasets (SlimPajama [Apache 2.0], OpenWebText [CC0 1.0 Universal], LM1B [Unavailable]), the GPT-2 tokenizer [MIT], ModernBERT-Large [Apache 2.0] for MAUVE [GPLv3], spaCy [MIT] for POS tagging, and the external diffusion-LM checkpoints (MDLM [Apache 2.0], Duo [Apache 2.0], FLM [Apache 2.0], LangFlow [Unavailable], ELF [MIT]) used in our experiments are publicly released, cited in the relevant sections of the paper, and used within their original licenses' permitted research use. 

\item {\bf New assets}
    \item[] Question: Are new assets introduced in the paper well documented and is the documentation provided alongside the assets?
    \item[] Answer: \answerYes{} 
    \item[] Justification: We will release the code and checkpoints upon acceptance. 

\item {\bf Crowdsourcing and research with human subjects}
    \item[] Question: For crowdsourcing experiments and research with human subjects, does the paper include the full text of instructions given to participants and screenshots, if applicable, as well as details about compensation (if any)? 
    \item[] Answer: \answerNA{} 
    \item[] Justification: This paper does not involve crowdsourcing nor research with human subjects. 

\item {\bf Institutional review board (IRB) approvals or equivalent for research with human subjects}
    \item[] Question: Does the paper describe potential risks incurred by study participants, whether such risks were disclosed to the subjects, and whether Institutional Review Board (IRB) approvals (or an equivalent approval/review based on the requirements of your country or institution) were obtained?
    \item[] Answer: \answerNA{} 
    \item[] Justification: This paper does not involve crowdsourcing nor research with human subjects. 

\item {\bf Declaration of LLM usage}
    \item[] Question: Does the paper describe the usage of LLMs if it is an important, original, or non-standard component of the core methods in this research? Note that if the LLM is used only for writing, editing, or formatting purposes and does \emph{not} impact the core methodology, scientific rigor, or originality of the research, declaration is not required.
    \item[] Answer: \answerNA{} 
    \item[] Justification: The core method development in this research does not involve LLMs as any important, original, or non-standard components. LLM assistance was used for writing, editing and formatting purposes. 

\end{enumerate}

\end{document}